\documentclass{article}

\usepackage[preprint]{neurips_2022}
\usepackage[utf8]{inputenc} % allow utf-8 input
\usepackage[T1]{fontenc}    % use 8-bit T1 fonts
\usepackage{hyperref}       % hyperlinks
\usepackage{url}            % simple URL typesetting
\usepackage{booktabs}       % professional-quality tables
\usepackage{amsfonts}       % blackboard math symbols
\usepackage{nicefrac}       % compact symbols for 1/2, etc.
\usepackage{microtype}      % microtypography
\usepackage{xcolor}         % colors

\usepackage{graphicx}
\usepackage{caption}
\usepackage{subcaption}
\usepackage{amsmath}
\usepackage{amsmath}
\usepackage{amssymb}
\usepackage{mathtools}
\usepackage{amsthm}
\usepackage{multirow}
\usepackage{tikz-cd}

\theoremstyle{plain}
\newtheorem{theorem}{Theorem}[section]
\newtheorem{proposition}[theorem]{Proposition}

\theoremstyle{definition}

\theoremstyle{remark}

\usepackage{algorithm}
\usepackage{algpseudocode}

\title{Speeding up PCA with priming}

\author{%
  Bálint Máté \\
  Department of Computer Science\\
  University of Geneva\\
  1227 Carouge, Switzerland \\
  \texttt{balint.mate@unige.ch} \\
  \And
  François Fleuret \\
  Department of Computer Science\\
  University of Geneva\\
  1227 Carouge, Switzerland \\
  \texttt{francois.fleuret@unige.ch} \\
}
\DeclareMathOperator{\Var}{Var}
\DeclareMathOperator{\spn}{span}
\begin{document}

\maketitle

\begin{abstract}
We introduce primed-PCA (pPCA), a two-step algorithm for speeding up the approximation of principal components. 
This algorithm first runs any approximate-PCA method to get an initial estimate of the principal components (priming), and then applies an exact PCA in the subspace they span. Since this subspace is of small dimension in any practical use, the second step is extremely cheap computationally. Nonetheless, it improves accuracy significantly for a given computational budget across datasets. In this setup, the purpose of the priming is to narrow down the search space, and prepare the data for the second step, an exact calculation.
We show formally that pPCA improves upon the priming algorithm under very mild conditions, and we provide experimental validation on both synthetic and real large-scale datasets showing that it systematically translates to improved performance. 
In our experiments we prime pPCA by several approximate algorithms and report an average speedup by a factor of 7.2 over Oja's rule, and a factor of 10.5 over EigenGame.
\end{abstract}
\section{Introduction}
Principal Component Analysis is a widely used tool both within and  outside of computer science. Introduced more than a century ago by \citet{PCA}, it has been used for compression and feature extraction, and has led to many important works and variants \citep{EigenFace,kernel-PCA}.

Let $X \in \mathbb{R}^{n\times d}$ be a centered dataset of cardinality $n$ and dimension $d$. 
In terms of linear algebra, the principal directions are the eigenvectors of the covariance matrix $X^TX$.
Unfortunately, if $d$ and $n$ are as high as in modern datasets, then traditional approaches, like computing the full-SVD of the covariance matrix gets computationally challenging. 
This has led to several attempts trying to circumvent this problem by using approximate/heuristic alternatives to full-SVD for finding the first few principal components of datasets of larger scale. Recently, \citet{gemp2020eigengame} introduced the EigenGame-algorithm and managed to perform approximate PCA on datasets of dimension $>10M$.

The contributions of this paper are:

In Section \ref{sec:ppca}, we introduce a family of approximate-PCA algorithms that performs a one-time, cheap full-PCA step on the output of the priming algorithm. 

In Section \ref{sec:theory}, we analyse this full-PCA step using elementary linear algebra and derive the mild theoretical conditions under which it improves the performance of the priming algorithm. 

In Section \ref{sec:experiments}, we study it empirically on several datasets by priming it by the power rule \citep{Rutishauser1970SimultaneousIM}, Oja's algorithm \citep{oja-simplified-neuron-model-1982} and EigenGame \citep{gemp2020eigengame} and demonstrate the improvement in accuracy and convergence speed. Whenever the size of the dataset allows we adopt the metric of ``Longest Correct Eigenvector Streak'' of \citet{gemp2020eigengame} that accounts for proper estimation of individual eigenvectors.

\begin{figure}[h]
\centering

\begin{minipage}{.54\textwidth}
\centering
 \includegraphics[width=\textwidth,trim= 20 10 40 10]{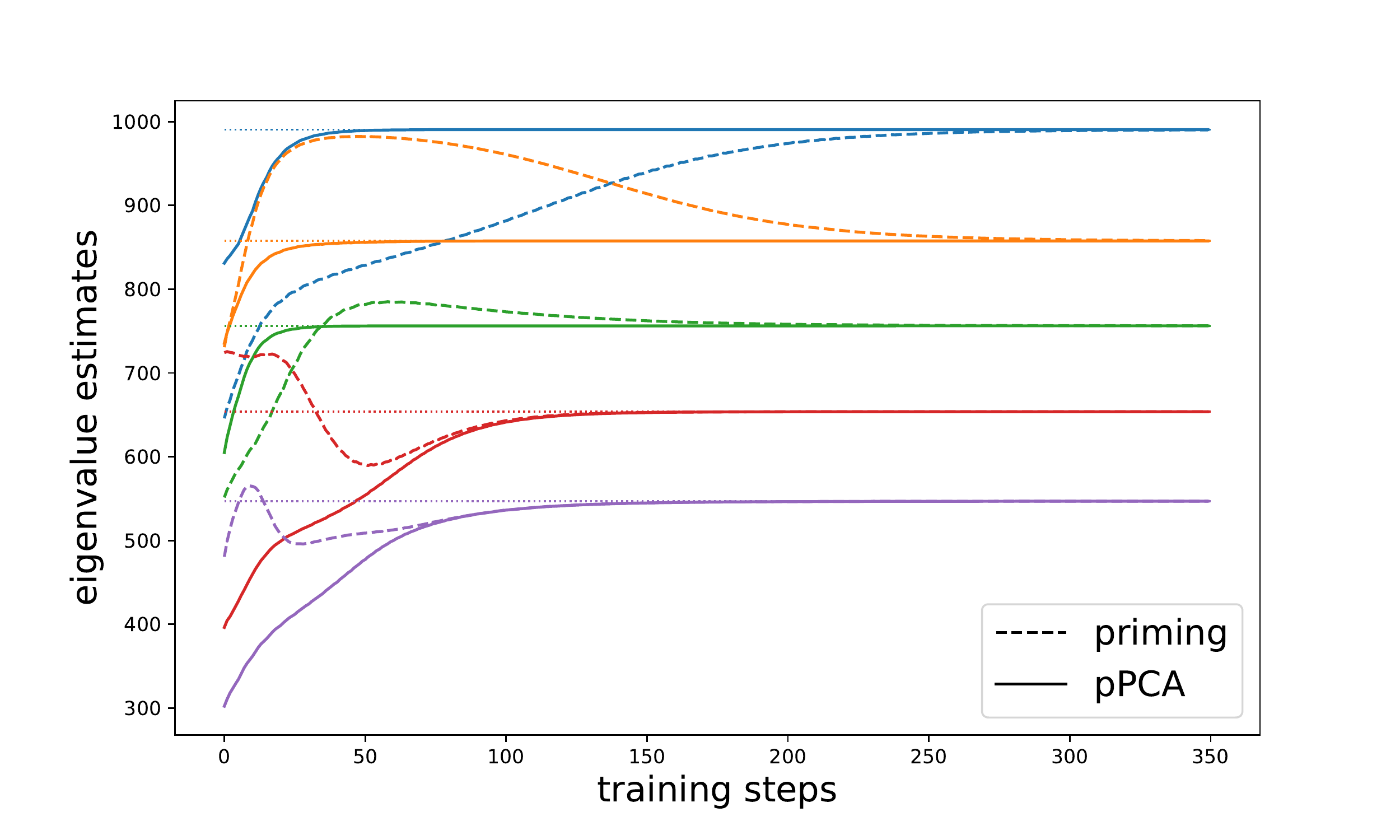}
\end{minipage}
\begin{minipage}{.45\textwidth}
\centering
\vspace{-2mm}
\begin{algorithm}[H]
      \label{alg:EigPCA}
      \textbf{Input}: Dataset of dimension $d$, $\{\textbf{x}_i\}_{i=1}^n$\\
      \textbf{Parameter}: $k>0$, $l>0$, \textsc{priming}  \\
      \textbf{Output}:  First $k$ principal components
      \begin{algorithmic}[1] %[1] enables line numbers
        \State $v_1,...,v_{k+l} \gets \mathrm{\textsc{priming}}(\{\textbf{x}_i\}_{i=1}^n,k+l)$
        \State $V \gets \texttt{span}(v_1,...,v_{k+l})$
        \State $\tilde{\textbf{x}}_i \gets \operatorname{proj}_V (\textbf{x}_i)$
        \State $e_1,...,e_k \gets \textbf{full-PCA}(\{\tilde{\textbf{x}}_i\}_{i=1}^n,k)$
        \State \textbf{return} $e_1,...,e_k$
        \end{algorithmic}
        \caption{Primed PCA ($\text{pPCA}_l$)}
      \end{algorithm}
\end{minipage}
  \caption{(Left) The results of an example run of the proposed algorithm on the synthetic data with exponentially decaying spectrum. Thin, dotted lines denote the eigenvalues of the covariance matrix. Dashed lines denote the eigenvalue estimates of the priming algorithm, while solid lines denote eigenvalue estimates of pPCA. (Right) Pseudocode for primed-PCA.}
  \end{figure}

\section{Related Work}
\subsection{Power method}
The power method, introduced by \citet{Rutishauser1970SimultaneousIM}, initialises a random vector $x_0 \in \mathbb{R}^d$ of unit norm
and iteratively computes $x_{i+1}=\tfrac{Mx_i}{\|Mx_i\|}$ until the first eigendirection of $M=X^TX$ dominates. The algorithm terminates when $||x_i - x_{i+1}|| < \epsilon$ for some small $\epsilon$.
To capture multiple principal components, \citet{SharmaPCA} propose to simply repeat the above algorithm and orthogonalize after every training step to stay in the complement of the span of the higher order directions.

\subsection{Oja's rule}
Oja's learning rule \citep{oja-simplified-neuron-model-1982} considers the output of a single neuron $y=\textbf{x}^T\textbf{w}=\textbf{w}^T\textbf{x}$
with update rule $\Delta \textbf{w}=\alpha (\textbf{x}y-y^2\textbf{w})=\alpha (\textbf{x}\textbf{x}^T\textbf{w}-(\textbf{w}^T\textbf{x})(\textbf{x}^T\textbf{w})\textbf{w})$.
Averaging over all points $C=\tfrac{1}{n}\sum_i\mathbf{x}_i\mathbf{x}_i^T$,  and plugging in the fixed-point condition ($\Delta \textbf{w}=0$) one gets
$C \textbf{w} -  (\textbf{w}^T  C\textbf{w}) \textbf{w}=0 $, an eigenvector-eigenvalue equation for the covariance matrix $C$, i.e. the weights following this dynamics end up as the coordinates of principal directions.
To find multiple principal components, the algorithm initializes $w_1,...,w_m$ weight vectors and uses the update rule
$\Delta \textbf{w}_m=\alpha y_m\big(\textbf{x}-\sum_{l\leq m} y_l\textbf{w}_l\big)$.
This generalized version of Oja's rule is sometimes called Sanger's rule \citep{Oja1992PrincipalCM,Sanger}.

\subsection{EigenGame}\label{EigenGame}
EigenGame \citep{gemp2020eigengame} interprets PCA in a game-theoretical framework, where vectors on the unit sphere $\mathcal S^{d-1} \subset \mathbb{R}^d$ correspond to strategies of players, playing the following multiplayer game. Let $v_1,...v_k \in \mathbb{R}^d (k<d)$ denote the players. The utility function of the first player is
\begin{equation}
    U_1(v_1)=v_1^T(X^TX)v_1
\end{equation}
i.e. player 1 is trying to tune the vector $v_1 \in \mathcal S^{d-1}$\footnote{$S^{d-1}$ denotes the (d-1) dimensional sphere, i.e. the $d$-dimensional vectors of unit norm.} to capture the maximum variance of the data. By definition, $v_1$'s goal is to find the first principal component of $X$. 
To make $v_2$  find the second principal component, the authors set the utility function of $v_2$ to
\begin{equation}
U_2(v_2)=v_2^T(X^TX)v_2-\frac{(v_2^T(X^TX)v_1)^2}{v_1^T(X^TX)v_1}
\end{equation}
where the first term rewards $v_2$ if it finds a direction of high variance but the second term cancels the reward from its component parallel to $v_1$.
Similarly, the utility function of $v_j(1<j\leq k)$ is
\begin{equation}
    U_j(v_j)=v_j^T(X^TX)v_j-\sum_{1\leq i<j}\frac{(v_j^T(X^TX)v_i)^2}{v_i^T(X^TX)v_i}
\end{equation}
Finding the Nash-equilibrium of the game defined by the utility functions $U_1,...,U_k$ is equivalent to finding the principal components of $X$.

This Gram-Schmidt-like setup defines a hierarchy between the players. Intuitively, $v_1$ aims to maximise it's variance without having to care about the other players, while all other players are also trying to maximise their variance, but have the additional constraint to stay orthogonal to the players with lower indices.

\section{Priming PCA}
\label{sec:ppca}
\subsection{\textsc{priming}}
Our algorithm makes use of an already existing approximate PCA algorithm (priming), not necessarily one of those we mentioned above. To avoid the cumbersome "any approximate PCA algorithm" in the rest of the paper, let \textsc{priming} be a placeholder for the priming algorithm to which any PCA algorithm can be assigned.

During the training process of \textsc{priming}, it can happen that the exact principal directions $e_1,...e_k$ are not yet properly captured by the predicted directions $v_1,...,v_k$, but they already lie (up to some error term) in $\texttt{span}( v_1,...,v_k )$.

\subsection{PCA as post-processing} 
\textbf{Question.}
Are $v_1,...,v_k$ the best approximations of $e_1,...e_k$ in $\texttt{span}( v_1,...,v_k )$?

\noindent Whenever the answer to this question is not affirmative, it makes sense to further optimize $v_1,...v_k$ within $\texttt{span}(v_1,...,v_k)$. Since usually $k \ll d$, full-PCA is feasible after projecting  the data to this $k$-dimensional subspace. We term the resulting algorithm, the combination of \textsc{priming} and full-PCA, primed-PCA (pPCA), and the algorithm for the initial approximation will be referred to as the priming algorithm or \textsc{priming}.

We expect that the full-PCA will ease numerical issues and inaccuracies related to stochastic gradient descent and speed up convergence.
In this paper we investigate how full-PCA improves the accuracy and speeds up \textsc{priming}. 

\subsection{Extra components}\label{ExtraComponents}
Knowing that after running \textsc{priming}, we will post-process the output and not accept it as the final prediction of the principal components, allows us to modify \textsc{priming} in a way that makes the job of full-PCA step easier.

For instance, if we are interested in the first $k$ principal components of $X$, we could run \textsc{priming} to search for $k+l$ principal directions, project onto the $(k+l)$-dimensional $\texttt{span}(v_1,...,v_k,...v_{k+l})$ and do full-PCA to extract the first $k$ components. Since the computational cost of running \textsc{priming} (usually) scales quadratically in the number of directions, $l$ has to be small. The question, of course, is whether or not the  performance gained from a few additional components compensates  for increased computational costs. In the rest of the paper, we refer to pPCA with $l$ additional directions as $\text{pPCA}_l$.

  \begin{figure*}
    \centering
    \begin{subfigure}[t]{0.3\textwidth}
      \includegraphics[width=\textwidth,trim=50 0 30 0]{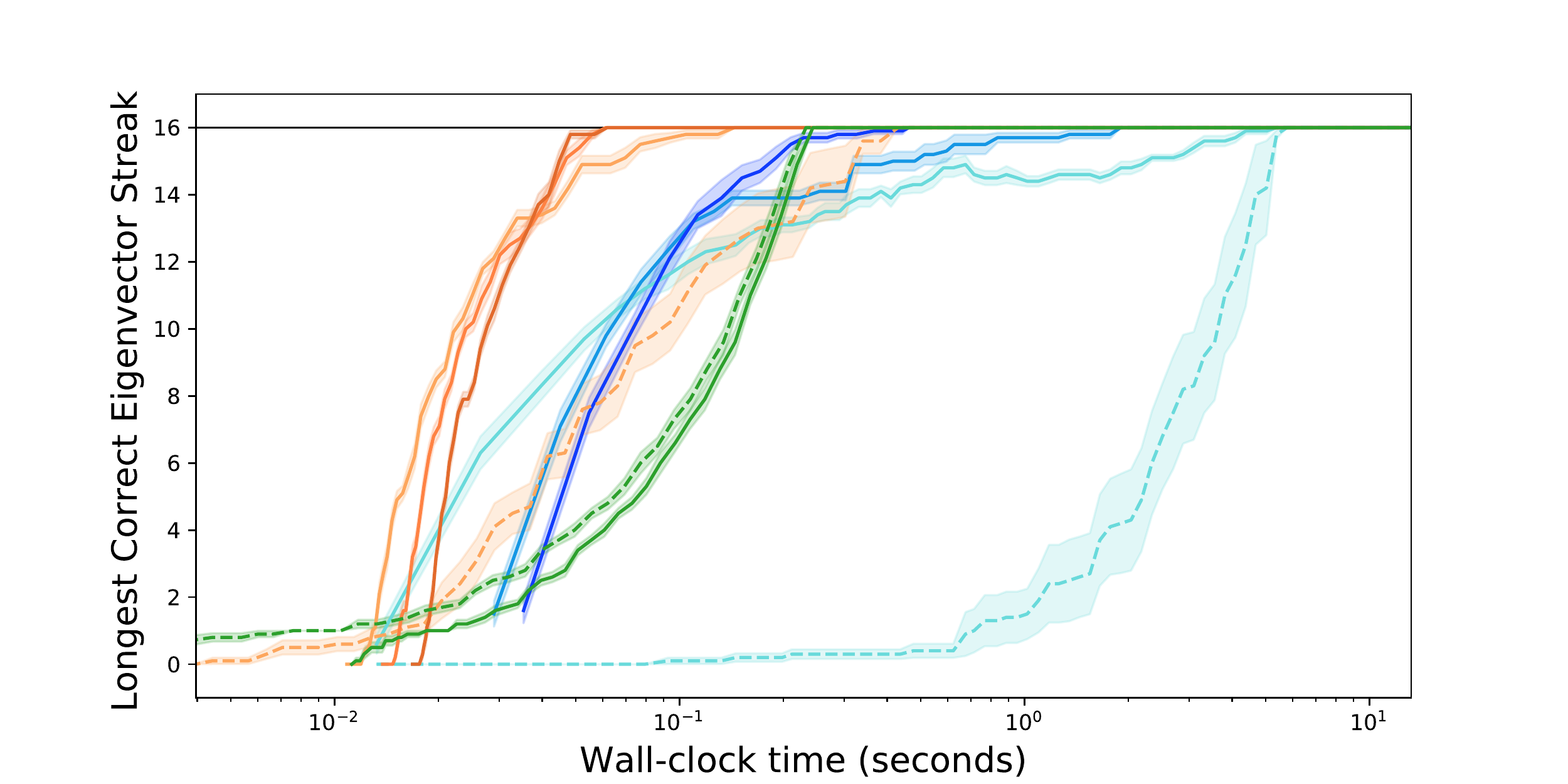}
      \caption{Exponential Synthetic data}
    \end{subfigure}
    \begin{subfigure}[t]{0.3\textwidth}
      \includegraphics[width=\textwidth,trim=50 0 30 0]{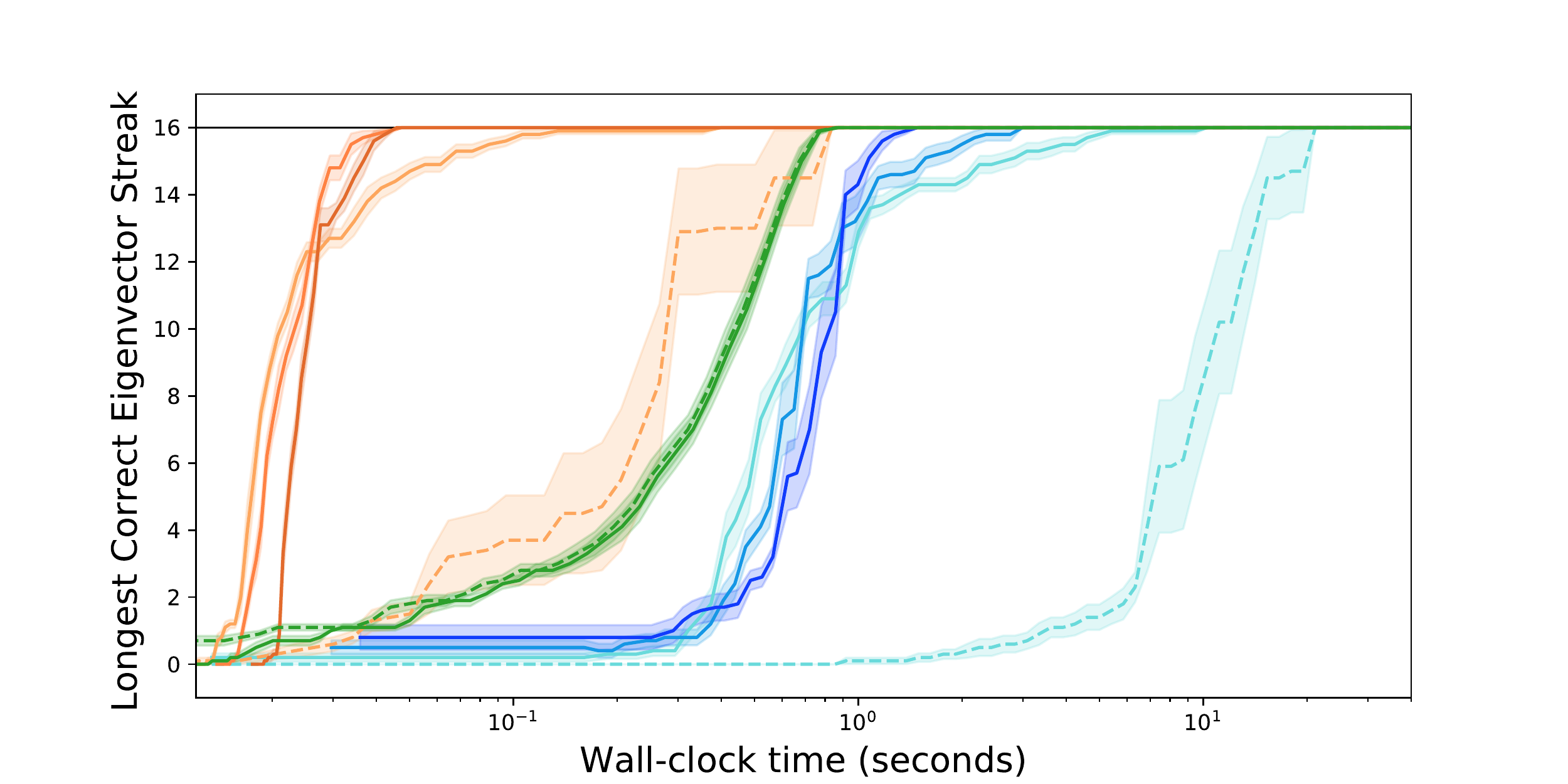}
      \caption{Linear Synthetic data}
    \end{subfigure}
    \begin{subfigure}[t]{0.3\textwidth}
      \centering
      \includegraphics[width=\textwidth,trim=50 0 30 0]{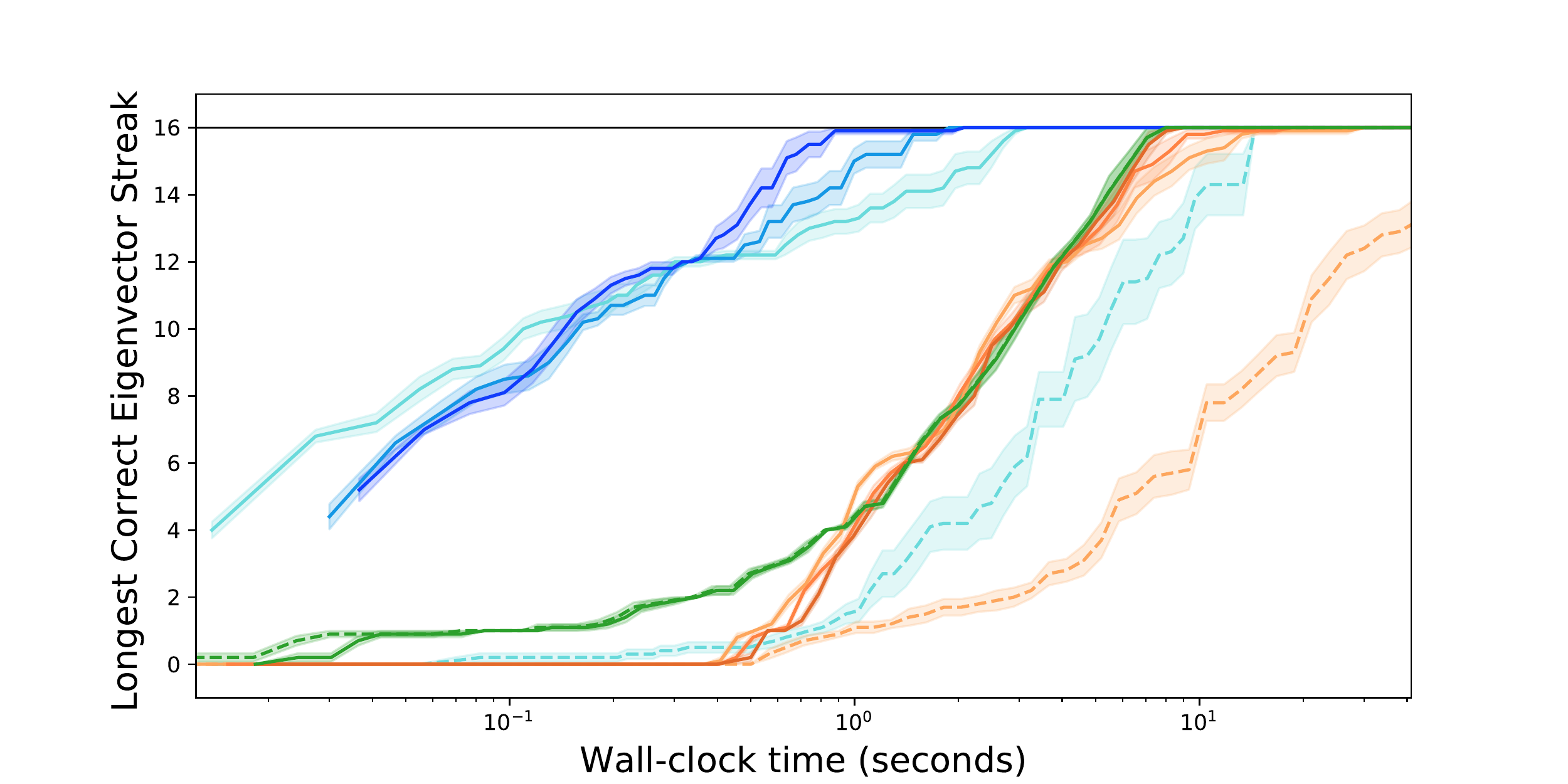}
      \caption{MNIST}
    \end{subfigure}
    \begin{subfigure}[t]{0.3\textwidth}
      \includegraphics[width=\textwidth,trim=50 0 30 0]{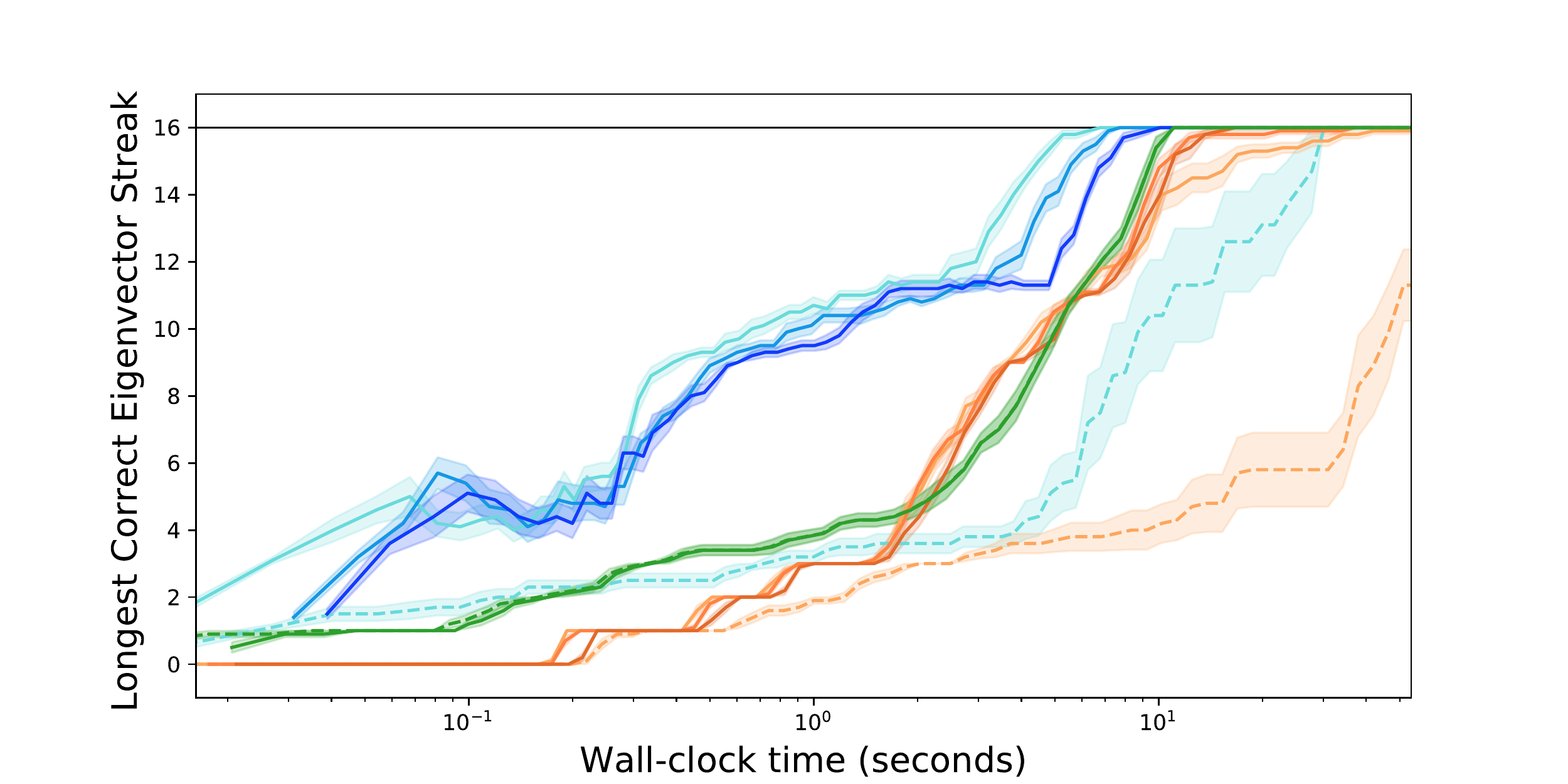}
      \caption{CIFAR10}
    \end{subfigure}
    \begin{subfigure}[t]{0.3\textwidth}
      \includegraphics[width=\textwidth,trim=50 0 30 0]{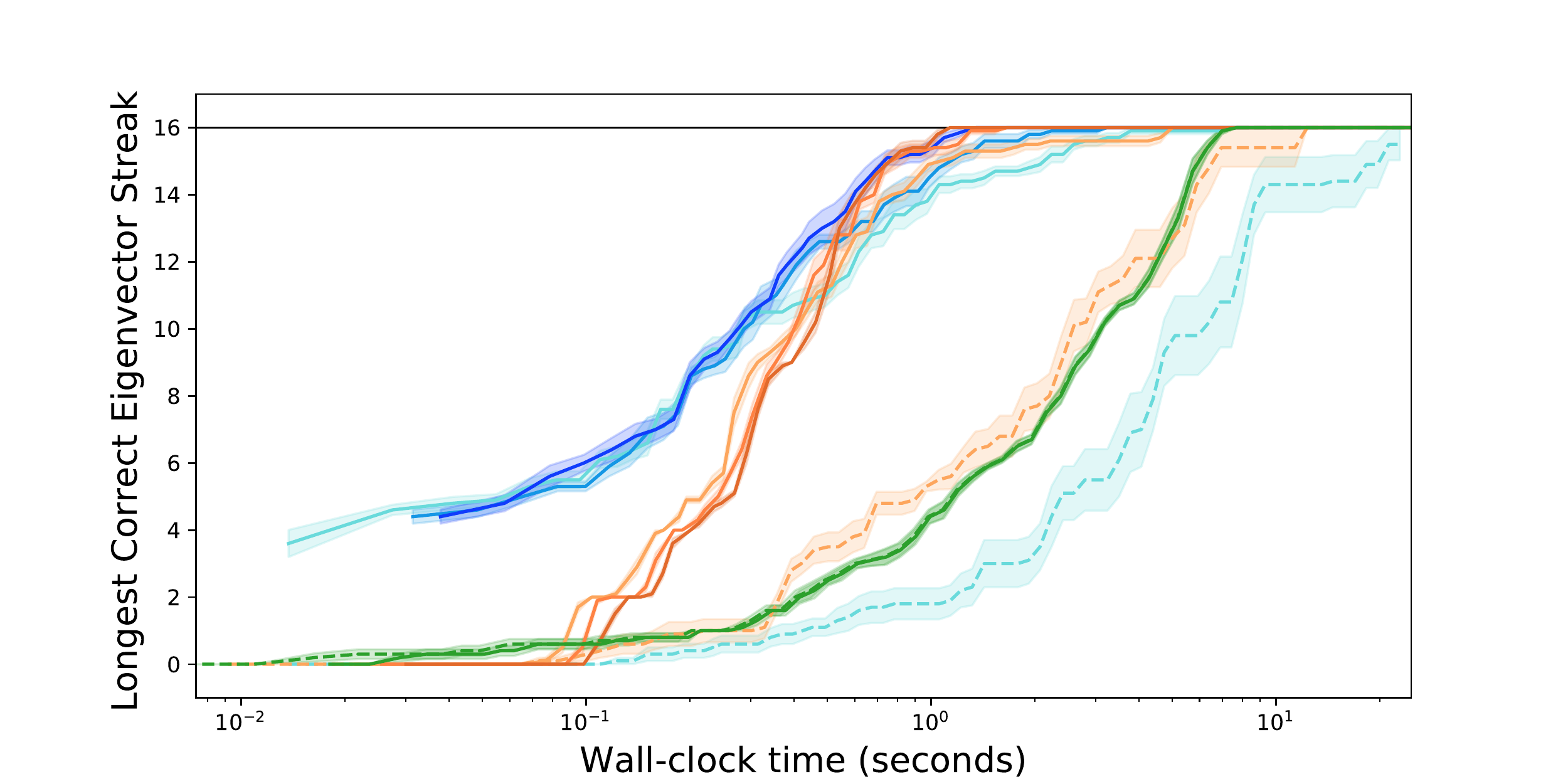}
      \caption{NIPS bag of words}
    \end{subfigure}
    \begin{subfigure}[t]{0.3\textwidth}
      \centering
      \includegraphics[width=3.6cm]{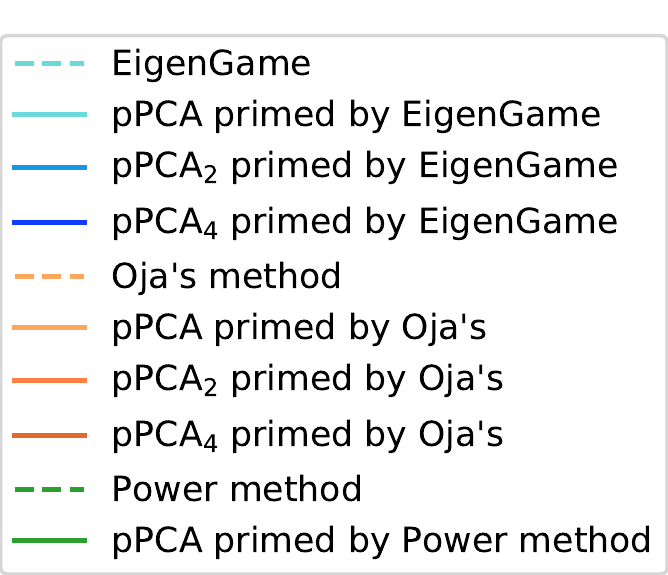}
    \end{subfigure}
    \caption{Plots showing the effect of the full-PCA step and the extra components on the small-scale experiments for $V=\pi/8$. 
    Solid lines are pPCA runs, dashed lines denote the 3 baselines (Power Method, Oja's, EigenGame).
    The pPCA curves have been primed the baseline of the same color.
     Shaded regions denote $\pm$ standard error of the mean. Plots for $V\in \{ \pi/16, ... ,\pi/1024\}$ are available in the appendix.
     }
    \label{plotFigure1}
  \end{figure*}
    \begin{figure*}
      \centering
      \begin{subfigure}[b]{.4\textwidth}
      \includegraphics[width=\textwidth,trim=50 0 30 0]{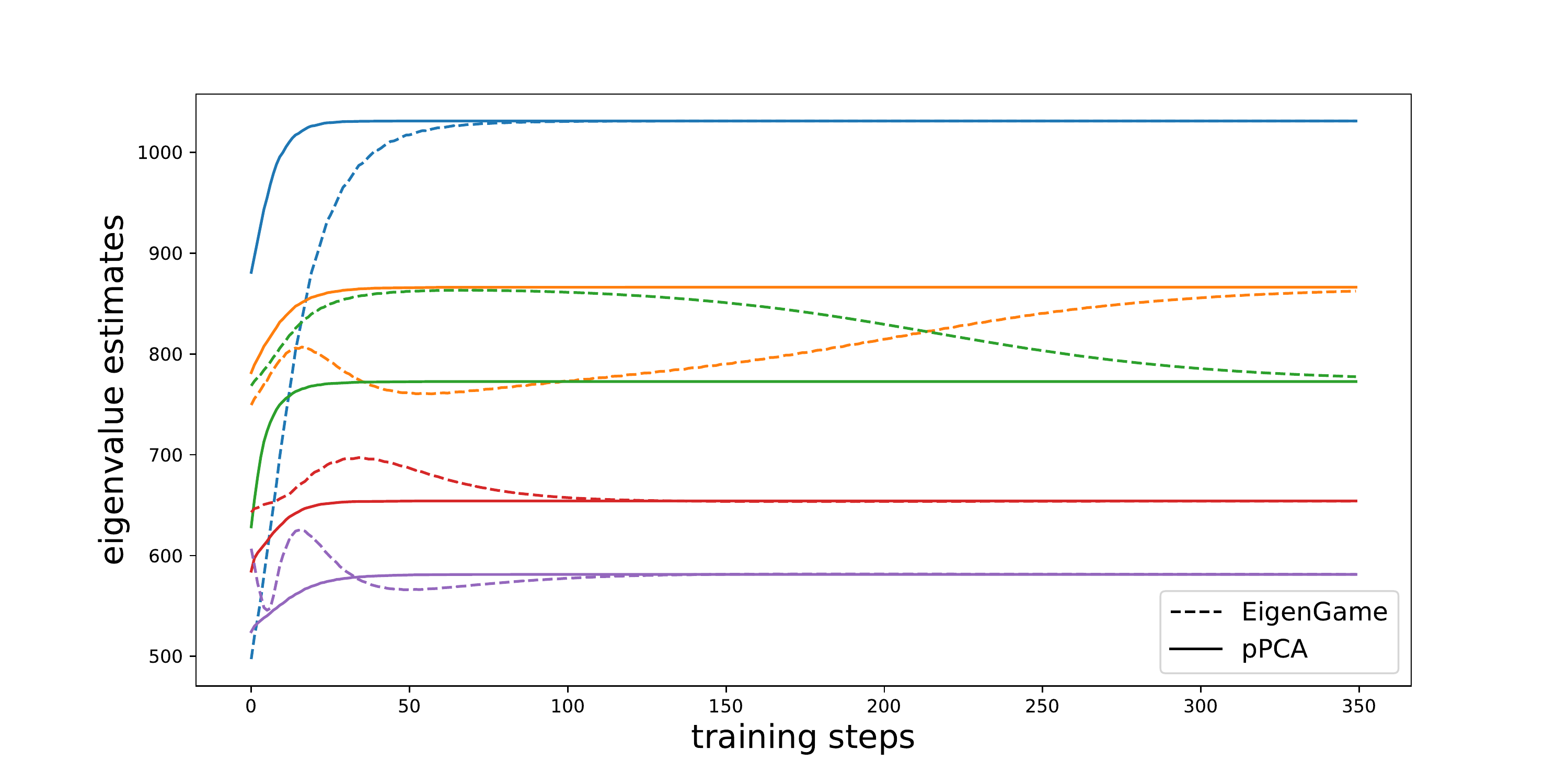}
      \caption{Exponential Synthetic data}
      \label{exp}
      \end{subfigure}
      \begin{subfigure}[b]{.4\textwidth}
        \includegraphics[width=\textwidth,trim=50 0 30 0]{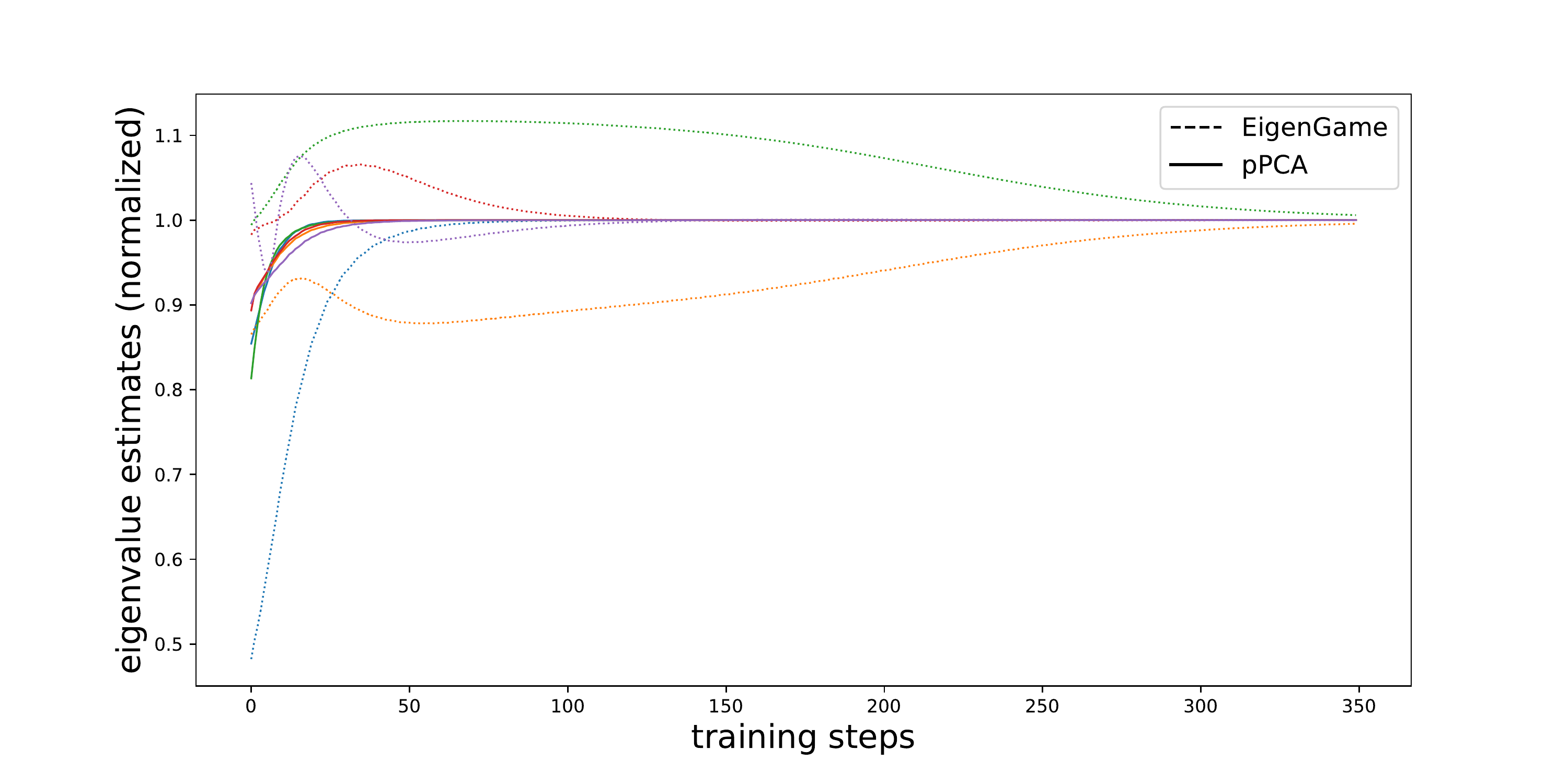}
        \caption{Exponential Synthetic data, normalized}
        \label{exp_norm}
        \end{subfigure}
      \caption{Figure (\subref{exp}) shows the estimates of the eigenvalues by pPCA and EigenGame on the Exponential Synthetic dataset with $k=5$. Figure (\subref{exp_norm}) is obtained by normalizing the curves of (\subref{exp})  by the average value of their respective final 5 datapoints.}
      \label{plotFigure3}
    \end{figure*}

  \begin{figure*}
    \centering
  \begin{subfigure}{.4\textwidth}
   \includegraphics[width=\textwidth,trim=50 0 30 0]{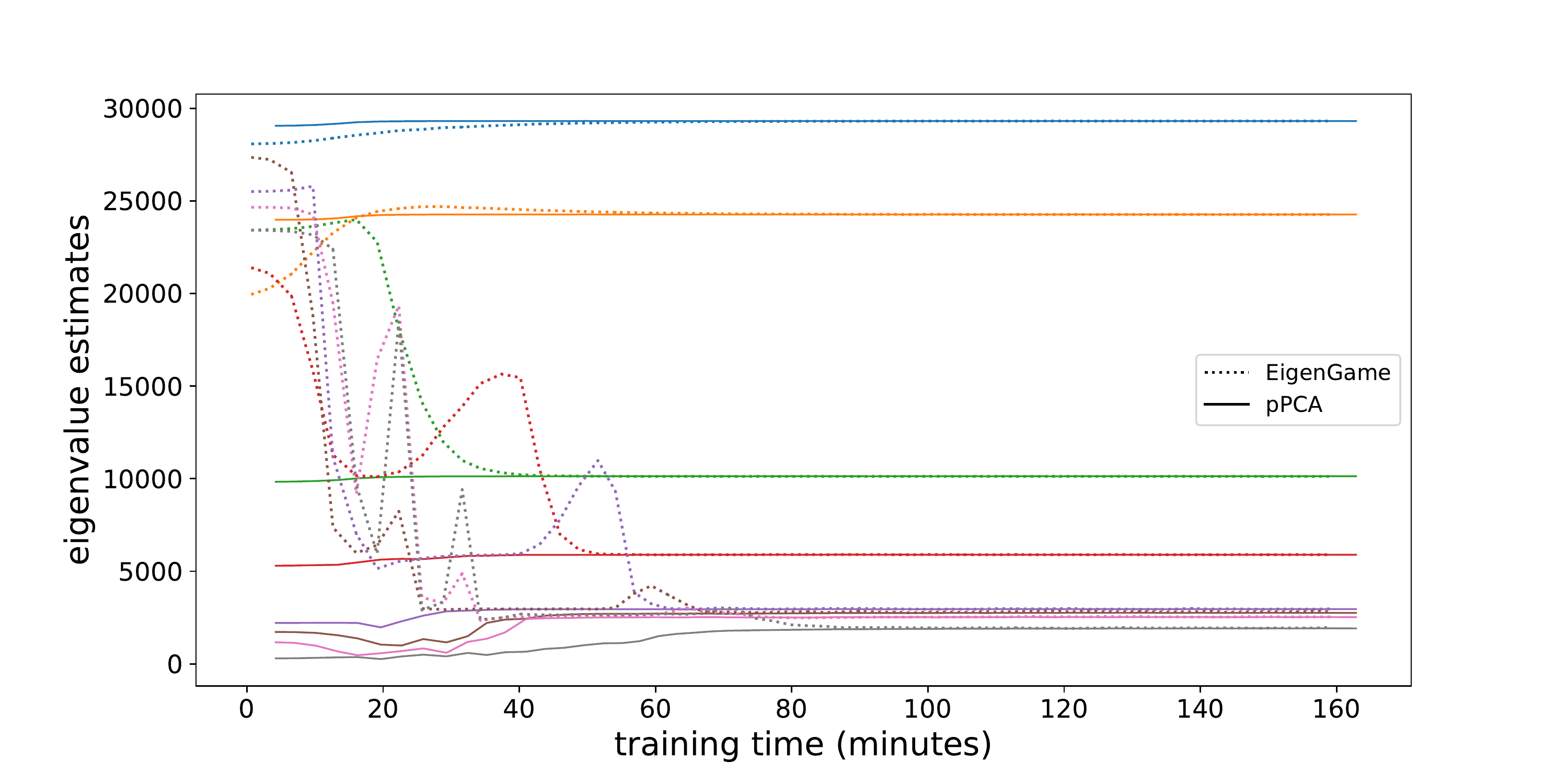}
  \caption{Large Lobster Image Dataset}
  \label{lobster}
  \end{subfigure}
  \begin{subfigure}{.4\textwidth}
     \includegraphics[width=\textwidth,trim=50 0 30 0]{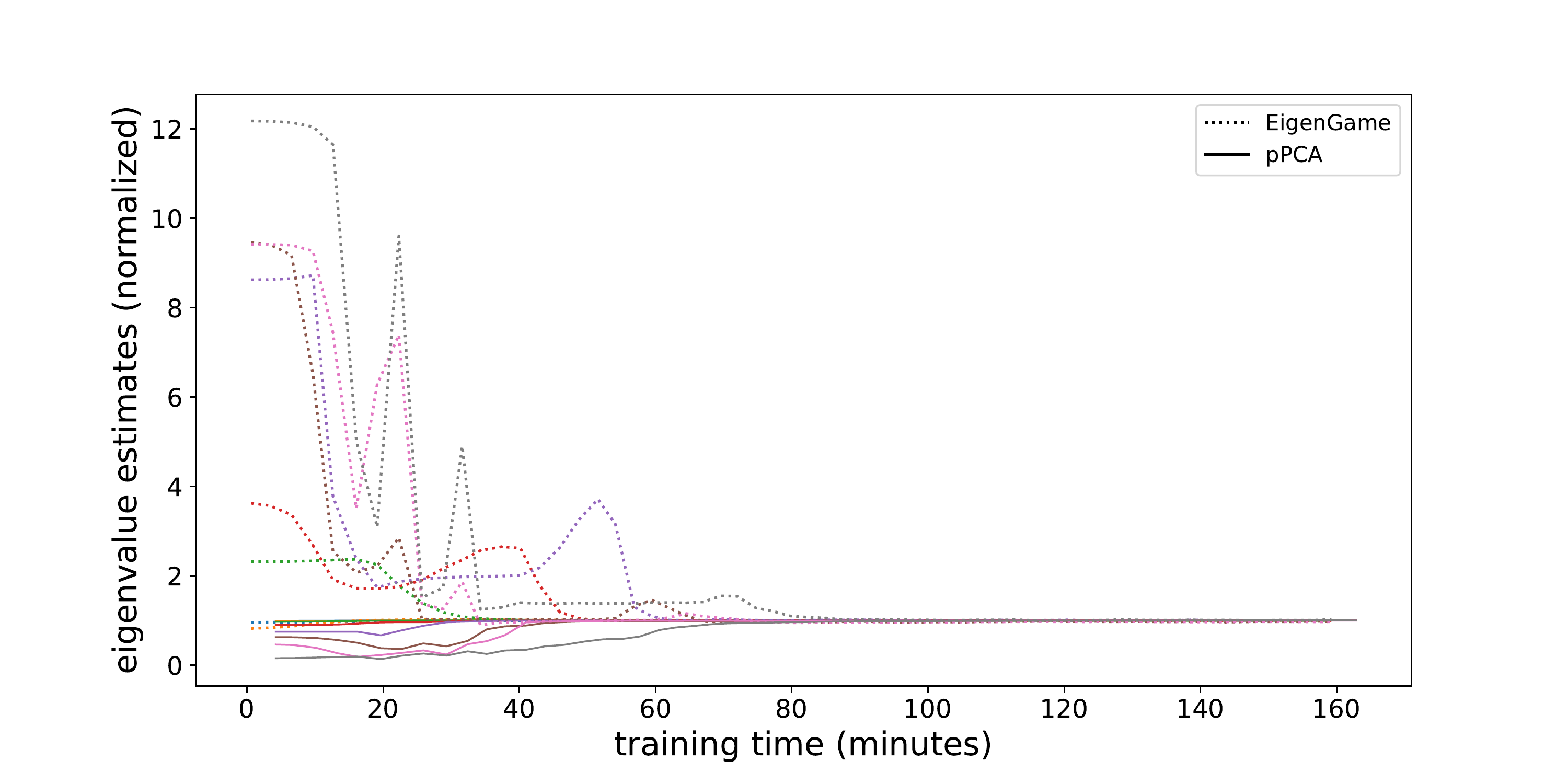}
    \caption{Large Lobster Image Dataset, normalized}
    \label{lobster_normalized}
    \end{subfigure}
    \begin{subfigure}{.4\textwidth}
       \includegraphics[width=\textwidth,trim=50 0 30 0]{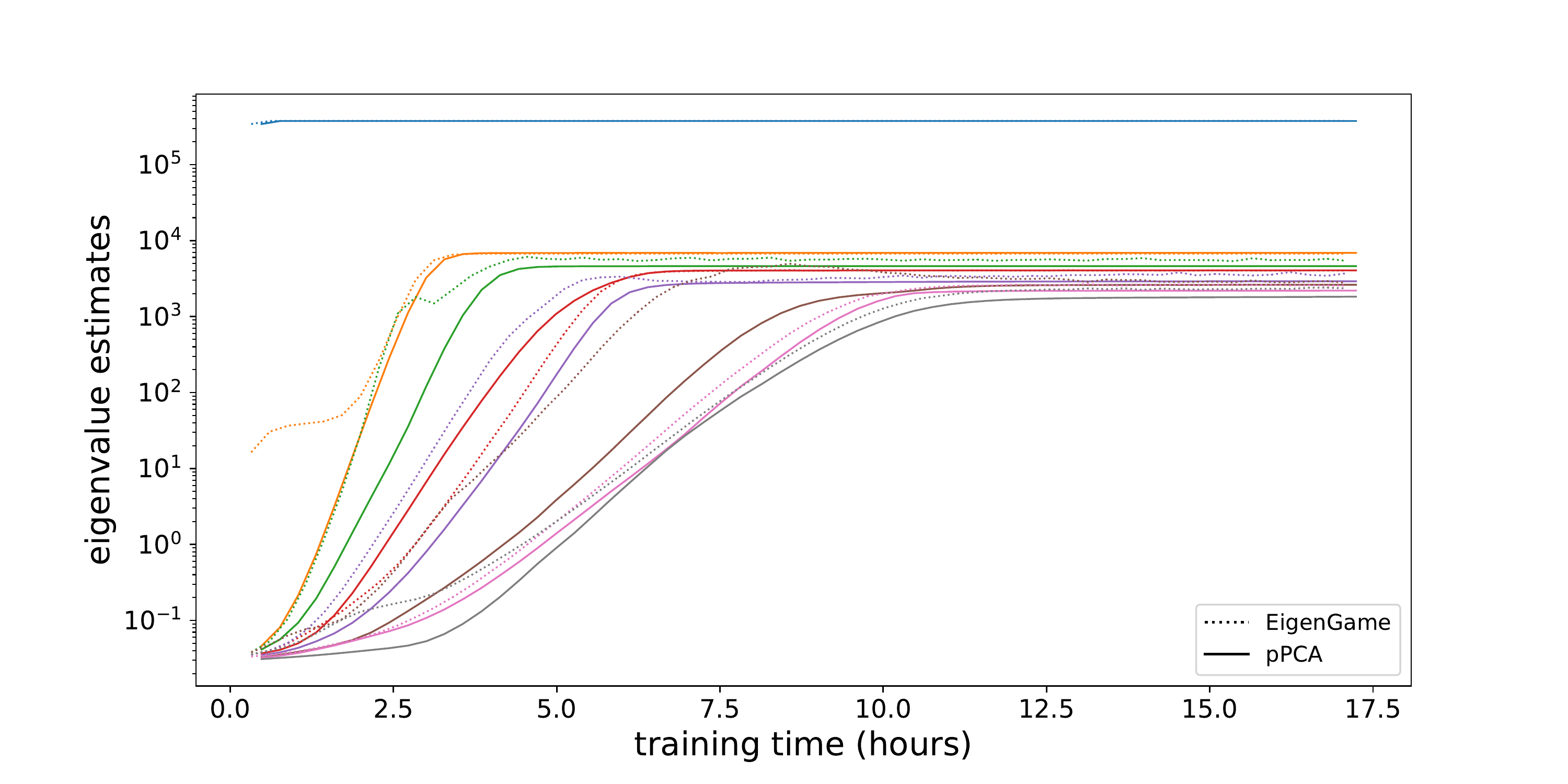}
      \caption{ResNet activations}
      \label{resnet}
    \end{subfigure}
    \begin{subfigure}{.4\textwidth}
     \includegraphics[width=\textwidth,trim=50 0 30 0]{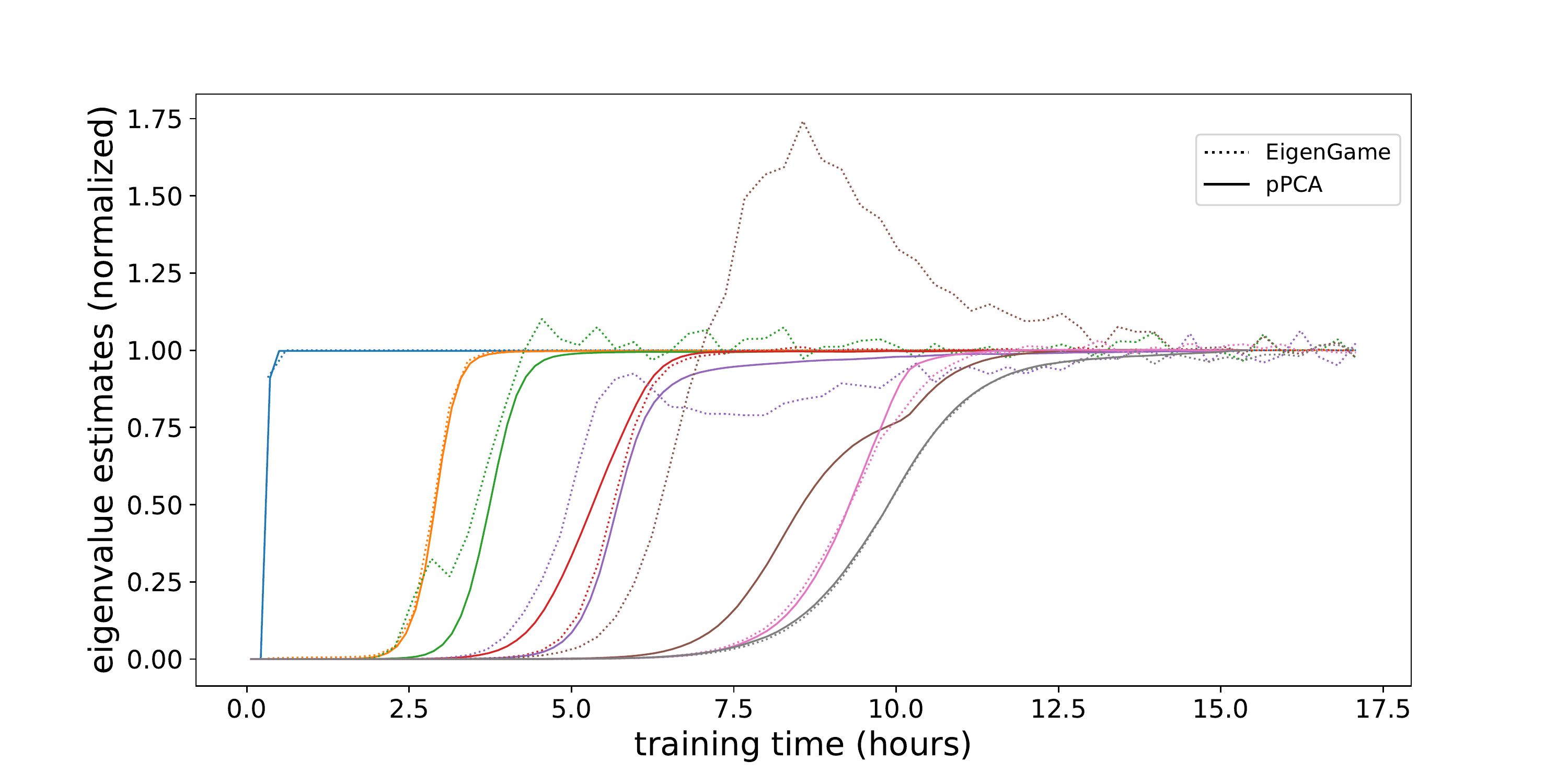}
    \caption{ResNet activations, normalized}
    \label{resnet_normalized}
    \end{subfigure}
  \caption{Plots showing the effect of the full-PCA on the large-scale experiments. Figures (\subref{lobster}) and (\subref{resnet}) show the estimates of the eigenvalues by pPCA and EigenGame. Figures (\subref{lobster_normalized}) and (\subref{resnet_normalized}) are obtained by normalizing the curves of (\subref{lobster}) and (\subref{resnet}).}
  \label{plotFigure2}
  \end{figure*}

\section{Theoretical analysis}\label{sec:theory}
It is intuitive to expect that the extra full-PCA step on the projected data should improve (or at least not hurt) the accuracy of the predicted principal directions. In this section, we analyse under what conditions  this intuitive expectation holds up.
\subsection{A counterexample} \label{counterexp}
Consider the $3$-dimensional, centered dataset of 6 elements, $X=\{\pm 3 (1,0,0), \pm 2(0,1,0), \pm(0,0,1)\}$. Suppose we are interested in finding the first 2 principal components using the process described above.
The covariance matrix is given by
\begin{equation}
    \tfrac{1}{6}X^TX=\tfrac{1}{3}\begin{pmatrix}
9 & 0 & 0\\
0 & 4 & 0 \\
0 & 0 & 1
\end{pmatrix}
\end{equation}
and the first two principal components are $e_1=(1,0,0), e_2=(0,1,0)$.
Now let us suppose that \textsc{priming} results in the vectors
$e_1^{\mathrm{\textsc{priming}}}= (\epsilon,0,\sqrt{1-\epsilon^2}), e_2^{\mathrm{\textsc{priming}}}=(0,1,0)$.
for some small $\epsilon \ll 1$. That is, the second principal direction is perfectly recovered. On the other hand, the first principal component is estimated to be a linear combination of the first and third principal directions almost having no contribution from the first one. 

Now, projecting the data onto $\texttt{span}(e_1^{\mathrm{\textsc{priming}}},e_2^{\mathrm{\textsc{priming}}})$  gives us
\begin{equation}
     \tilde X=\{\pm 3\epsilon (\epsilon,0,\sqrt{1-\epsilon^2}), \pm 2(0,1,0),\pm\sqrt{1-\epsilon^2}(\epsilon,0,\sqrt{1-\epsilon^2})\}
\end{equation}
Or, in the basis of $\{e_1^{\mathrm{\textsc{priming}}},e_2^{\mathrm{\textsc{priming}}}\}$,
$\tilde X=\{\pm 3\epsilon (1,0), \pm 2(0,1), \pm\sqrt{1-\epsilon^2}(1,0)\}$.
The covariance matrix of the projected data is then
\begin{equation}
    \tfrac{1}{6}\tilde X^T\tilde X=\tfrac{1}{3}\begin{pmatrix}
9 \epsilon^2 +2(1-\epsilon^2) & 0 \\
0 & 4  \\
\end{pmatrix} 
\end{equation}
If $\epsilon$ is small enough, the full-PCA after the projection will predict the principal directions to be
$e_1^{pPCA}=e_2^{\mathrm{\textsc{priming}}}$ and  $e_2^{pPCA}=e_1^{\mathrm{\textsc{priming}}}$.
i.e it changes the ordering of the principal directions.
Unfortunately,
$\langle e_1^{\mathrm{\textsc{priming}}},e_1\rangle^2=\langle(\epsilon,0,\sqrt{1-\epsilon^2}),(1,0,0)\rangle^2=\epsilon^2$ and
$\langle e_1^{pPCA},e_1\rangle^2=\langle (0,1,0),(1,0,0)\rangle^2=0$.
In other words, we had a better approximation of the first principal direction before doing the full-PCA step.

It is important to note the absurdity of this example. The starting assumption that
$e_1^{\mathrm{\textsc{priming}}}=(\epsilon,0,\sqrt{1-\epsilon^2})$.
would mean that \textsc{priming} almost completely ignored the first two  principal components and aligned itself with the third, smallest one. This, of course, can happen but it is not the expected behavior from an algorithm that is designed to find the principal components. The typical scenario is that \textsc{priming} finds $e_1$ up to some small error, and the remaining variance is picked up by the other players with lower indices. This is the case where the final full-PCA step will be useful.

Nonetheless, the above example demonstrates that we cannot just state
that full-PCA step cannot decrease the accuracy of predicted principal components.

\noindent The reason why the above example fails is that the subspace spanned by the vectors $e_i^{\mathrm{\textsc{priming}}}$ is almost orthogonal to the first principal direction. Projecting the data onto this subspace ``forgets" that the original data has had high variance along $e_1$ and $e_1$ becomes impossible to recover because the lower principal components dominate the projected data.

\subsection{When does the full-PCA step help?}
In this section we derive the theoretical conditions under which the full-PCA step is useful. The main result of the section is Theorem \ref{thm:main} which roughly states if the priming algorithm converges then situations like the one in the Section \ref{counterexp} can only happen in the early phase of the priming algorithm where the eigenvector estimates are not yet "aligned enough" with the actual eigenvectors.

We begin the discussion with fixing the notation,
\begin{center}
\renewcommand{\arraystretch}{1.3}
\begin{tabular}{ c|l } 
notation & meaning \\

\hline
 $X,  X^TX$ & the dataset and its covariance matrix \\ 
 $\{e_1,...,e_k\}$ & the first $k$ principal directions of $X$ (eigenvectors of $X^TX$) \\
$e_1^{\mathrm{\textsc{priming}}},...,e_k^{\mathrm{\textsc{priming}}}$ & \textsc{priming}'s estimates of the principal directions \\
$S, \pi_S$ & $\spn (e_1^{\mathrm{\textsc{priming}}},...,e_k^{\mathrm{\textsc{priming}}})$, projection operator onto $S$ \\
$\tilde X, \{\tilde e_1,...,\tilde e_k\}$ & the projection of $X$ and $\{e_1,...,e_k\}$ onto $S$ \\
$\mathbb{S}^{d-1}$ & the $(d-1)$ dimensional sphere, vectors of unit norm
\end{tabular}
\end{center}

The variance of the data $X$ in a  direction $v \in \mathbb S^{d-1}$ is
$\Var_v(X)=v^T(X^TX)v$.
Write now $\tilde X$ as the matrix product $\tilde X=X P_S$, where $P_S$ is the projection matrix from $\mathbb{R}^d$ onto $S$.
The variance of $\tilde X$ in a given direction $v \in S \cap \mathbb S^{d-1}$ is then
\begin{align*}
    \Var_v(\tilde X)&=v^T(\tilde X^T\tilde X)v 
                        =v^T(P^T_SX^T X P_S)v
                        =(P_S v)^T(X^T X)(P_Sv)
                       =v^T(X^T X)v
\end{align*}
i.e. the quadratic form of the variance of $\tilde X$ is just that of $X$ after restricting its domain to $S$,
$\Var_{(\_)}(\tilde X)=\Var_{(\_)}(X)\big |_S$.
%\subsubsection{The first principal component}
\begin{proposition}
  \label{pr1}
If $\pi_S(e_1)$ maximises $\Var_{(\_)}(X)$ on $S$, then the full-PCA step on $S$ can not decrease the accuracy of the predicted first principal component.
\end{proposition}
\begin{proof}
If $\pi_S(e_1)$ maximises ${\Var}_{(\_)}(X)$ on $S$, then the full-PCA step on $S$ returns $\pi_S(e_1)$ as the first principal component. If \textsc{priming} already output $\pi_S(e_1)$ as the first principal component, then this has no effect on the accuracy, in all other cases, accuracy is improved.
\end{proof}
In essence, the only thing that can go wrong for the first principal component is that $S$ is such that the variance of $\tilde X$ in the direction of $\pi_S(e_1)$ is smaller than in the direction of, say, $\pi_S(e_2)$.

\begin{proposition}
  \label{pr2}
  If the assumption of Proposition \ref{pr1} are satisfied , and
  $\pi_S(e_2)$ maximises ${\Var}_{(\_)}(X)$ on $\pi_S(e_1)^\perp$,
    then the full-PCA step on $S$ can not decrease the accuracy of the predicted first and second principal components.
\end{proposition}
\begin{proof}
The proof is just a repeated use of Proposition \ref{pr1}. As the assumption of Proposition \ref{pr1} is satisfied, full-PCA returns $\pi_S(e_1)$ as the first principal component, not decreasing the accuracy of the prediction. Let us now consider the second principal direction. Since the eigenvectors of a covariance matrix are orthogonal, and  $\pi_S(e_2)$ maximises ${\Var}_{(\_)}(X)$ on $\pi_S(e_1)^\perp$ by assumption, the repeated use of Proposition \ref{pr1} on $\pi_S(e_1)^\perp$ implies that the full-PCA step outputs $\pi_S(e_2)$ as the second eigenvector, not decreasing the accuracy of the prediction.
\end{proof}

Continuing this line of reasoning, always requiring the projection of the next principal component to maximise the variance on the orthogonal complement of the previous eigenvectors, we arrive at following sequence of conditions
  
\begin{proposition}
  \label{pr3}
  If $V$ is such that
  \begin{itemize}\setlength\itemsep{1pt}\setlength\topsep{0pt}
    \item $\pi_S(e_1)$ maximises ${\Var}_{(\_)}(X)$ on $S$,
    \item $\pi_S(e_2)$ maximises ${\Var}_{(\_)}(X)$ on $\pi_S(e_1)^\perp$,
    \item[] \dots
    \item $\pi_S(e_i)$ maximises ${\Var}_{(\_)}(X)$ on the orthogonal complement of $\spn(\pi_S(e_1),...,\pi_S(e_i))$
  \end{itemize}
    then the full-PCA step on $S$ can not decrease the accuracy of the predicted $1^{st},2^{nd},....,i^{th}$ principal components.
\end{proposition}
\begin{proof}
This is again just a repeated application of Proposition \ref{pr1} as in the proof of Proposition \ref{pr2}.
\end{proof}
It is important to note that even though the conditions of Propositions \ref{pr3} look cumbersome, the outputs of every approximate PCA algorithm satisfies them if trained long enough.
\begin{theorem}
\label{thm:main}
Let $e_i^{\textsc{priming}}(\tau)$ denote the output of \textsc{priming} at timestep $\tau$. If \textsc{priming} converges, i.e. if 
\begin{equation}
    \label{eq:cnd}\lim_{\tau \to \infty} e_i^{\textsc{priming}}(\tau) \rightarrow e_i \qquad \textnormal{for all } 0<i\leq k
\end{equation}
then there exists a $\tau^*$ such that if $\tau>\tau^*$ then the conditions of Proposition \ref{pr3} are satisfied and applying the full-PCA step on $\{e_1^{\mathrm{\textsc{priming}}}(\tau),...,e_k^{\mathrm{\textsc{priming}}}(\tau)\}$ improves performance with respect to the metric of the longest correct eigenvector streak for any $V$.
\end{theorem}
%\begin{proof}
% Let $e_1,..., e_k$ denote the first $k$ principal directions with corresponding eigenvalues $\lambda_1 > \lambda_2 > ... > \lambda_k$. Since $e_1^{\textsc{priming}}(\tau) \rightarrow e_1$, there exists $\tau_1$ such that $e_1^{\textsc{priming}}(\tau)$ for $\tau>\tau_1$ is close enough to $e_1$ to satisfy 
% \begin{equation}
%     \label{eq:1}
%     Var_{e_1^{\textsc{priming}}(\tau)}(X)>\max_{v \in [e_1^{\textsc{priming}}(\tau)]^\perp} Var_v(X)
% \end{equation}
% Note that for $\tau>\tau_1$ then the condition of Proposition \ref{pr1} is satisfied. Using the same argument for $e_2$, there exists a $\tau_2$ such that if $\tau>\tau_2$, then both (\ref{eq:1}) and
% \begin{equation}
%     \label{eq:2}
%     Var_{e_2^{\textsc{priming}}(\tau)}(X)>\max_{v \in V^\perp} Var_v(X) \quad V=span(e_1^{\textsc{priming}}(\tau),e_2^{\textsc{priming}}(\tau))
% \end{equation}
% are satisfied, which then implies the implications of Proposition \ref{pr2}. Continuing this line of reasoning, we  construct a sequence of $\tau_1<\tau_2<...<\tau_k$, and for $\tau^*=\tau_k$ (\ref{eq:cnd}) is satisfied.
% \end{proof}

\begin{proof}%[Alternative Proof]
If $e_i^{\textsc{priming}}(\tau) = e_i$ for all $i$, i.e. if \textsc{priming} already found the eigenvectors then the implication of the theorem obviously holds. As the conditions of \ref{pr3} are an open condition (only involving inequalities), they are satisfied on an open set. This means that there exists an open neighborhood $U$ of $\{e_1,...,e_k\}$ on which these conditions  of \ref{pr3} are satisfied. Since 
$$\lim_{\tau \to \infty} e_i^{\textsc{priming}}(\tau) \rightarrow e_i \qquad \textnormal{for all } 0<i\leq k$$
The vectors $\{e_1^{\textsc{priming}}(\tau),...,e_k^{\textsc{priming}}(\tau)\}$ enter and stay in $U$ if $\tau$ is large enough.
\end{proof}
\section{Experiments and results} \label{sec:experiments}

We test the proposed algorithm on several dataset of varying size. Here we present them in order of increasing complexity. Doing full-PCA in the original data space is doable for the small-scale datasets (synthetic, MNIST, CIFAR10 and  ``NIPS bag of words"), but infeasible for the Large Lobster Image Dataset and ResNet activations. 

We train with SGD using Nesterov momentum with a factor of 0.9 \citep{Nesterov1983AMF}. Everything is implemented in Pytorch, the  experiments on the small-scale datasets are executed on a NVIDIA RTX 3090 while the experiments on the large-scale datasets are executed on an NVIDIA A100.

\begin{table*}[ht!]
    \centering
    \vspace{10mm}
    \begin{tabular}{l| c c c c c }
      \multicolumn{1}{c|}{ } &{\begin{tabular}{c}Exponential \\Synthetic\end{tabular}}& {\begin{tabular}{c}Linear \\Synthetic \end{tabular}} & {MNIST} & {CIFAR10} & {\begin{tabular}{c}NIPS \\bag of words\end{tabular}}\\ \noalign{\hrule height1pt}
      \multicolumn{1}{c|}{dimensions} & {50} & {50} & {768} & {3092} & {11 463}\\
      \multicolumn{1}{c|}{points} & {5000} & {5000} & {60 000} & {50 000} & {5812}\\
      \cmidrule(lr){1-6}
       \multicolumn{6}{c}{$V=\pi/8$} \\
       \cmidrule(lr){1-6}
        EigenGame             & $4.55 \pm 0.67$         &  $12.24\pm4.14$        &  $9.33\pm3.64$         &  $15.79\pm8.33$        &  $7.54\pm2.19$                    \\
        + pPCA (Ours)         & $3.48 \pm 0.87$         &  $4.04\pm2.30$         &  $\mathbf{2.21\pm0.59}$         &  $\mathbf{5.02\pm0.73}$         &  $\mathbf{3.81\pm1.37}$            \\
         Oja's method         & $0.23 \pm 0.09$         &  $0.32\pm0.18$         &  n.a.                  &  n.a.                  &  $5.67\pm2.15$            \\
        + pPCA (Ours)         & $\mathbf{0.08 \pm 0.03}$         &  $\mathbf{0.11\pm0.09}$         &  $11.53\pm5.74$        &  n.a.                  &  $2.68\pm1.71$            \\
        Power method          & $0.21 \pm 0.01$         &  $0.67\pm0.07$         &  $6.32\pm0.80$         &  $9.63\pm1.13$         &  $4.47\pm0.38$            \\
        + pPCA (Ours)         & $0.23 \pm 0.01$         &  $0.68\pm0.07$         &  $6.33\pm0.80$         &  $9.63\pm1.13$         &  $4.48\pm0.38$            \\
       \cmidrule(lr){1-6}
       \multicolumn{6}{c}{$V=\pi/32$} \\
       \cmidrule(lr){1-6}
       EigenGame                      & $5.81 \pm 0.83$         &  n.a.                  &  n.a.                   &  $28.37\pm10.09$       &  n.a.            \\
       + pPCA (Ours)            & $4.58 \pm 0.95$         &  $7.45\pm3.00$         &  $\mathbf{3.44\pm0.57}$          &  $\mathbf{6.92\pm0.79}$         &  ${6.57\pm2.31}$            \\
       Oja's method                   & $0.34 \pm 0.09$         &  $0.58\pm0.24$         &  n.a.                   &  n.a.                  &  $10.03\pm2.98$            \\
       + pPCA (Ours)            & $\mathbf{0.15 \pm 0.05}$         &  $\mathbf{0.19\pm0.10}$         &  $19.74\pm7.31$         &  n.a.                  &  $4.54\pm1.98$            \\
       Power method                   & $0.22 \pm 0.01$         &  n.a.                  &  $6.38\pm0.80$          &  $9.73\pm1.13$         &  $\mathbf{4.52\pm0.38}$            \\
       + pPCA (Ours)            & $0.23 \pm 0.01$         &  $0.69\pm0.07$         &  $6.38\pm0.80$          &  $9.74\pm1.13$         &  $4.54\pm0.38$            \\
      \cmidrule(lr){1-6}
    \end{tabular}
    \caption{Time to reach an Eigenvalues streak of length 16 with different threshold values $V \in \{\pi/8,\pi/32 \}$ for the algorithms on the small-scale datasets. All experiments were repeated 10 times, mean and variance values are reported. All values have units of seconds, if one the 10 runs did not reach a streak of length 16 during training, then  ``n.a.'' is reported. A more detailed table including pPCA$_2$ and pPCA$_4$ can be found in the appendix.} 
    \label{results_small}
  \end{table*}
\subsection{Small-scale Datasets}
We run EigenGame for each of the datasets 10 times with learning rates $\{1,10^{-1},10^{-2},10^{-3},10^{-4},10^{-5},10^{-6}\}$. We then choose the learning rate with the smallest angular error when summing over runs, timesteps and principal components. The pPCA methods primed by EigenGame are then executed also 10 times with the learning rate that best fitted EigenGame. 
Similarly, we run Oja's algorithm for each of the datasets 10 times with learning rates $\{10^{-3},10^{-4},10^{-5},10^{-6}\}$. We then choose the learning rate with the smallest angular error when summing over runs, timesteps and principal components. The pPCA methods primed by Oja's method are then executed with this learning rate.
We run the power method with termination conditions $\epsilon \in \{10^{-3},10^{-4},10^{-5},10^{-6},10^{-7}\}$, choose the best $\epsilon$ and run pPCA on top of it.
Due to the sequential nature of the power method, we don't expect much improvement from the full-PCA step in this setup. In all these experiments, we are interested in finding the first 16 principal components.

\paragraph{Synthetic data}
We generate synthetic datasets along the lines of the synthetic experiments of the EigenGame paper \citep{gemp2020eigengame}. The data consists of 5000 points in 50 dimensions with a spectrum that decays exponentially (resp. linearly) over 3 orders of magnitudes, from 1000 to 1. We train with a batch size of 1000.
\paragraph{MNIST and CIFAR10.}
We flatten the training images of MNIST \citep{lecun-mnisthandwrittendigit-2010} and CIFAR10 \citep{Krizhevsky2009MastersThesis} that  results in a dataset of cardinality 60 000 (resp. 50 000) and dimensionality 784 (resp. 3092). We use a batch size of 1000.

\paragraph{NIPS bag of words}
This dataset contains the frequency of  11 463 words in 5812 NIPS papers published between 1987 and 2015 \citep{perrone2016poisson}. Each entry in the 11 463-by-5812 matrix holds the number of occurrences of a given word in the corresponding paper. We train with a batch size of 1000.

\subsection{Large-scale Datasets}
For the large-scale experiments we only train EigenGame and run pPCA on top of it (without additional components).

\paragraph{ResNet activations}
Following the original EigenGame paper, we build a dataset from the activation patterns of a pretrained ResNet-152 \citep{he2015deep} evaluated on the validation set of ImageNet \citep{deng2009imagenet}. Each datapoint is constructed from the outputs of all residual blocks of the ResNet-152. This results in a dataset of dimension $d\sim 13.1M$ and cardinality $n = 50 000$. We use a batch size of 1024 and train for 600 training steps ($\sim 12$ epochs) with a learning rate of $10^{-6}$ to find the first 8 principal components.
The time-cost of the full-PCA step (in particular, of the projection) is comparable to the time cost of 1 training epoch. In this experiment, we only train for $12$ epochs, therefore the the cost of full-PCA is not negligible. Circumventing this issue, we only project 10\% of the data for computing full-PCA, reducing the overhead for full-PCA from $\sim95$ to $\sim9.5$ minutes.

\paragraph{Large Lobster Image Dataset}
The Large Lobster Image Dataset \citep{LobsterPaper} contains 6654 images of 238 southern rock lobsters taken at a lobster processor in Tasmania over the course of 6 days. The images have a resolution of $1944 \times 2592$ and 3 color channels. After flattening, we end up with $d\sim15.6M$  dimensional vectors. See Figure \ref{fig:lobsters} for examples. We use a batch size of 1024 and train for 300 training steps ($\sim 43$ epochs) with a learning rate of $10^{-4}$ to find the first 8 principal components.
\begin{figure}[ht]
  \centering
  \begin{subfigure}{0.24\textwidth}
    \centering
    \includegraphics[width=\textwidth,trim=80 30 80 0]{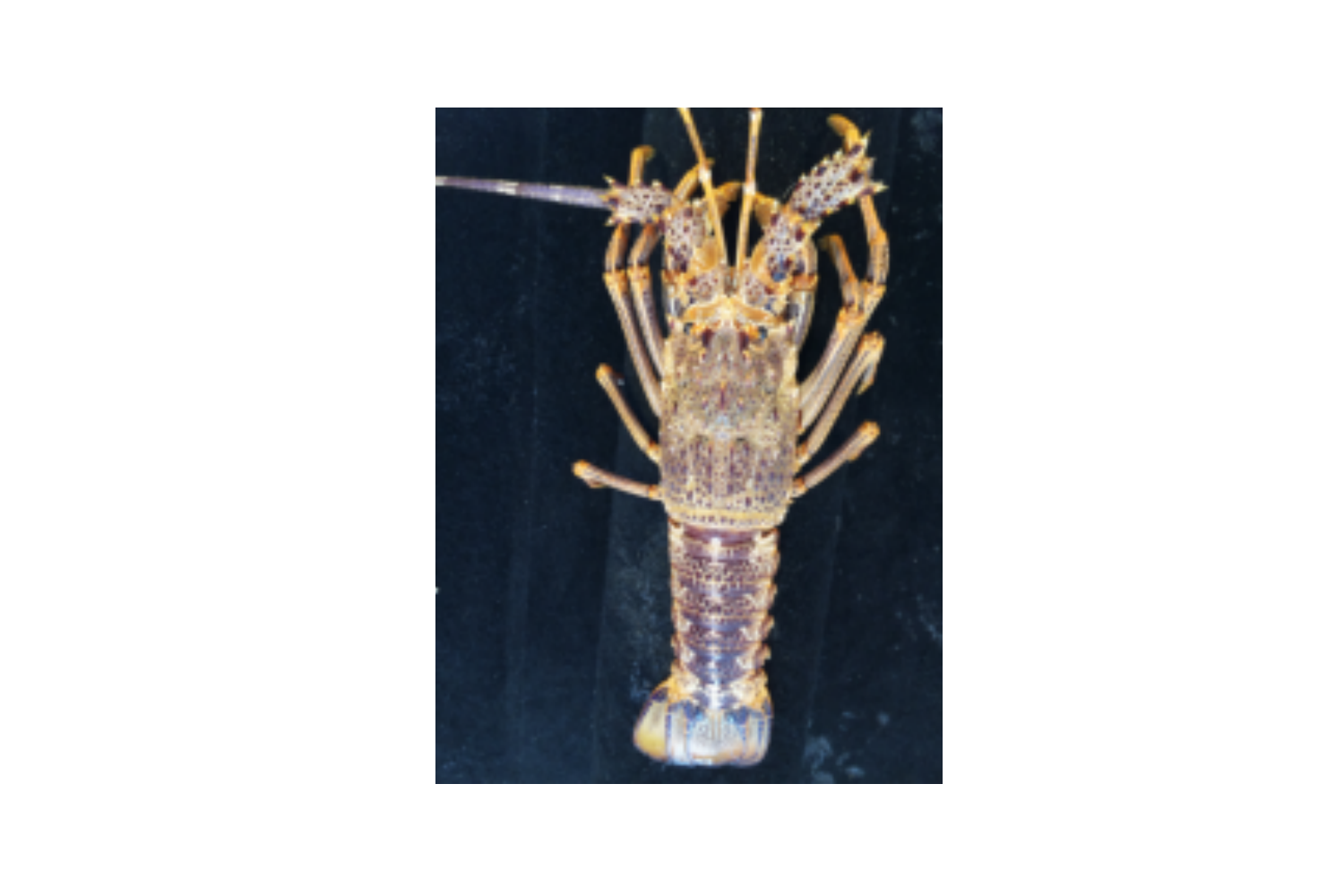}
    \caption{Lobster \#2 on day 2}
  \end{subfigure}
  \begin{subfigure}{0.24\textwidth}
    \centering
    \includegraphics[width=\textwidth,trim=80 30 80 0]{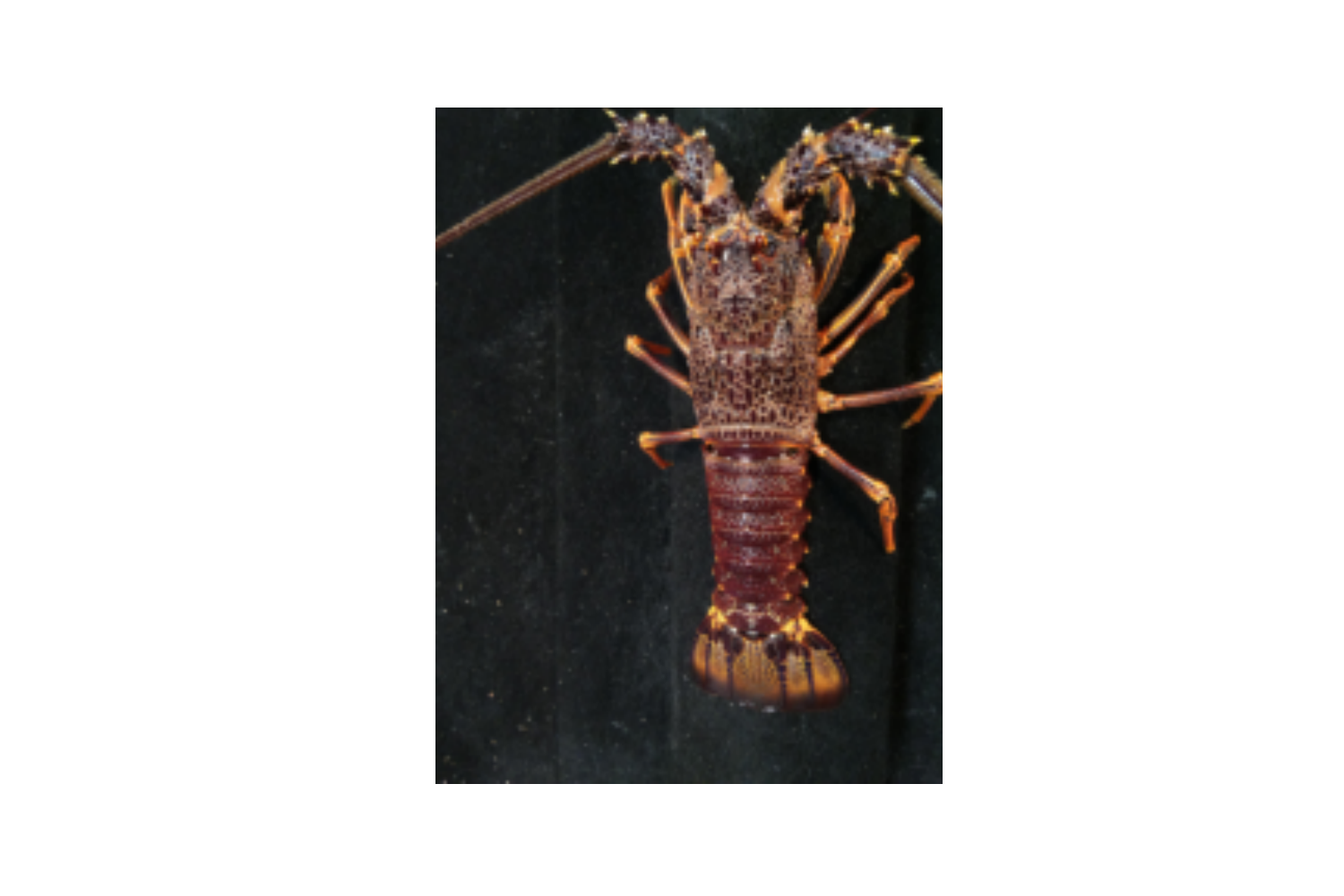}
    \caption{Lobster \#17 on day 5}
  \end{subfigure}
  \begin{subfigure}{0.24\textwidth}
    \centering
    \includegraphics[width=\textwidth,trim=80 30 80 0]{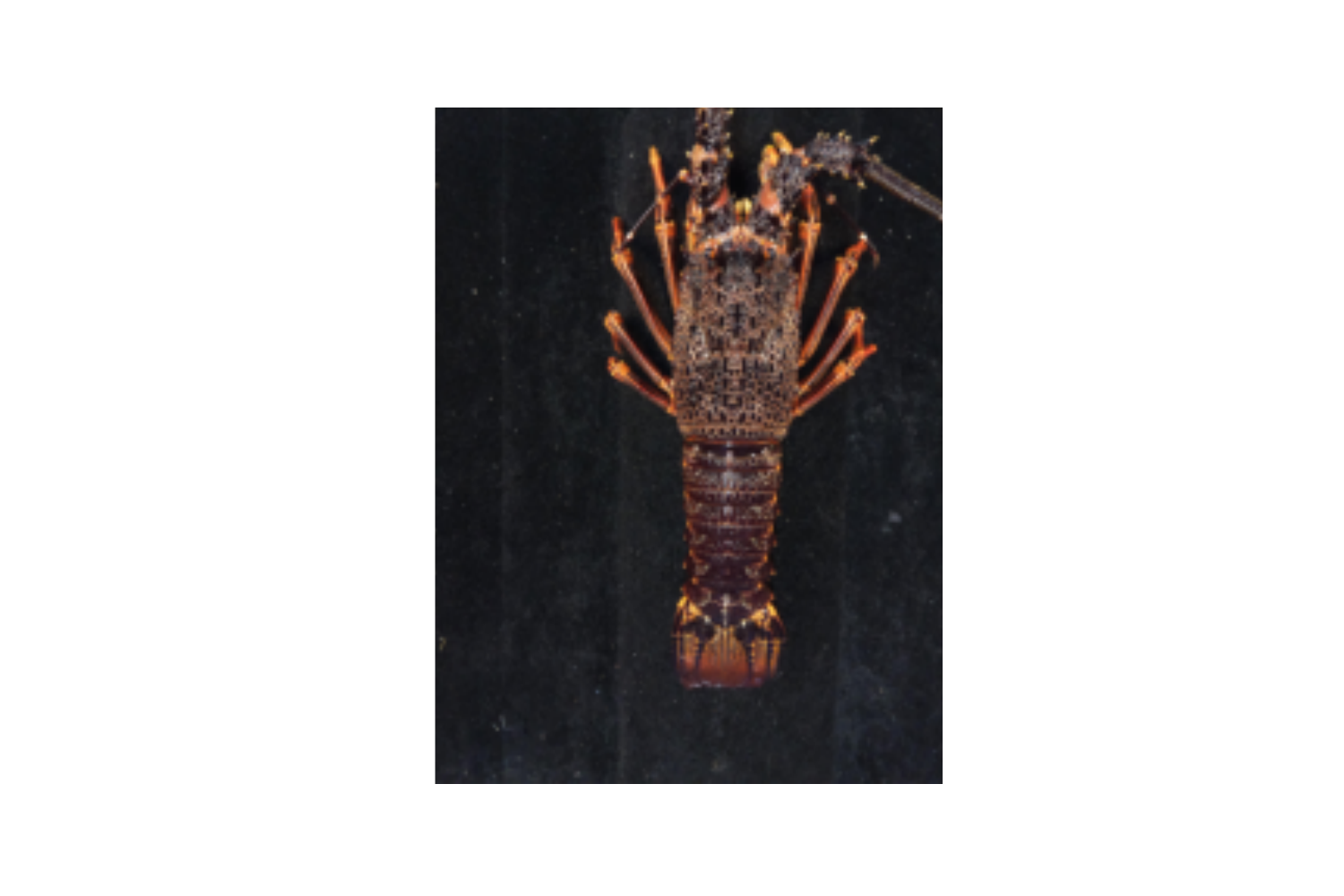}
    \caption{Lobster \#220 on day 6}
  \end{subfigure}
  \begin{subfigure}{0.24\textwidth}
    \centering
    \includegraphics[width=\textwidth,trim=80 30 80 0]{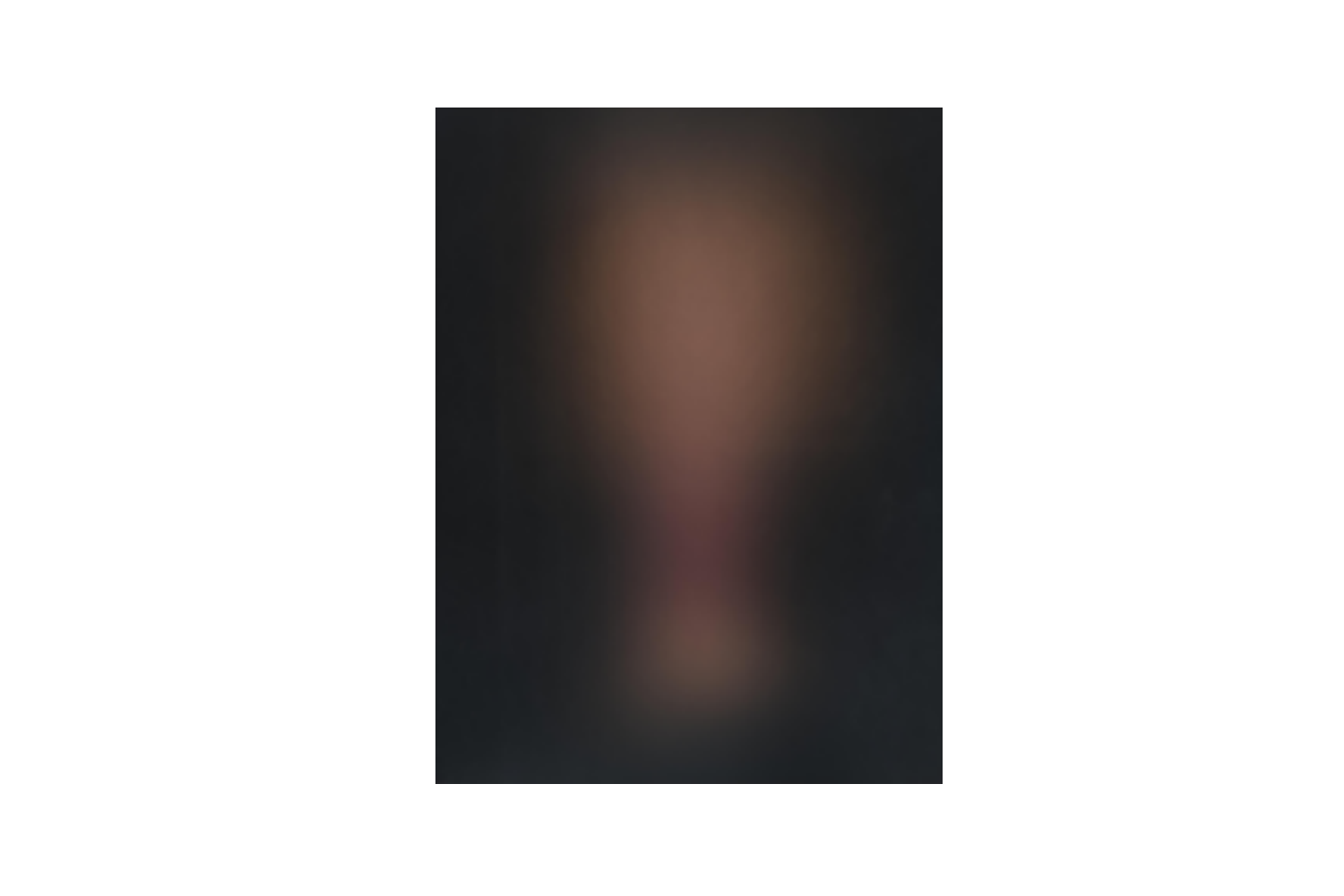}
    \caption{Average of the dataset}
  \end{subfigure}
  \caption{The Large Lobster Image Dataset}
  \label{fig:lobsters}
\end{figure}

\subsection{Measure of performance}

\paragraph{Small-scale}
To evaluate the small-scale experiments we adopt the demanding metric of ``Longest Correct Eigenvector Streak'' of \citet{gemp2020eigengame} which reflects the hierarchy of principal directions. Given the ground truth principal directions $\{e_1,...e_k\}$, the approximate directions coming from the learning algorithm $\{v_1,...,v_k\}$ and a threshold value $0<V<\pi/2$, we compute if the angle between $v_i$ and $e_i$ is smaller than $V$ for all $i \in \{1,...,k\}.$ Then, the number of consecutive pairs from index 1 that are within angle $V$ of each other is by definition the Longest Correct Eigenvector Streak. For instance, if for a given $V$ the angles are $[V/2,V/3,2V,V/2,...]$, then the Longest Correct Eigenvector Streak for this $V$ is 2.
In our experiments, we evaluate all runs with $V \in \{\pi/8,\pi/16,...,\pi/1024\}$.

\paragraph{Large-scale}
The remaining 2 experiments (Lobsters and ResNet activations) need to evaluated differently, since the real principal directions are not available. This means that we can only evaluate performance indirectly. We plot the variance\footnote{The variance captured by a vector $v$ is $v^T(X^TX)v$} captured by the result of EigenGame and pPCA. When the algorithm converges to the actual principal directions of the data, the variance will converge to the corresponding eigenvalue of the covariance matrix. Convergence of the variance is therefore a hint (but not evidence!) of convergence of the corresponding eigenvalues/principal directions.

\subsection{Results}

\paragraph{Small-scale}
Figure \ref{plotFigure1} shows the plots of the Longest Correct Eigenvector Streak for the small-scale experiments described above. We reuse the threshold  from the EigenGame paper \citep{gemp2020eigengame} and set it to $V=\pi/8$ for the plots of Figure \ref{plotFigure1}.
In the Appendix, plots are available for all threshold values $V \in \{\pi/16,...,\pi/1024\}$. 
A qualitative comparison is available in Table \ref{results_small}.

\paragraph{Large-scale}
Figure \ref{plotFigure2} displays the results for the large-scale experiments. Figure \ref{PrincipalLobsters} in the Appendix shows the eigenvectors found by both methods after 30 and 300 training steps of the Large Scale Lobster Experiment. After 300 training steps the methods converge to vectors that are indistinguishable by the human eye. Looking at the results after 30 training steps, we can also conclude the pPCA already found the first 4 components, while EigenGame only the first 2. Note that this is exactly what Figure \ref{plotFigure2} suggests.

\section{Conclusion}
We introduced primed-PCA, a method for computing the first $k$ principal directions by combining an already existing algorithm (called the priming) with a cheap, post-processing full-PCA step.  We have demonstrated on several datasets that pPCA  greatly improves upon the priming in terms of convergence speed and worked out the algebraic condition that guarantees that in most cases this is expected to happen. We have also demonstrated on small-scale datasets that using extra components improves performance further when we use Oja's algorithm or EigenGame for priming.

\subsection*{Future Work} 
As discussed in Subsection \ref{ExtraComponents}, the fact that after running the priming we will post-process the results allows for modifications of the priming algorithm. 
Since none of the existing approximate-PCA was not designed to be followed by a post-processing step, it is possible that their current form is not the optimal preprocessing step for the full-PCA.
It could also be investigated how the optimal value of additional players, $l$, changes as the parameters of the dataset, $d$ and $n$, or the number of principal components of interest, $k$, vary.

\section{Funding disclosure}
B{\'{a}}lint  M{\'{a}}t{\'{e}} was supported by the Swiss National Science Foundation under grant number FNS-193716 ``Robust Deep Density Models for High-Energy Particle Physics and Solar Flare Analysis (RODEM)".

\nocite{*}
\bibliographystyle{plainnat}
\bibliography{biblio}

\begin{thebibliography}{22}
\providecommand{\natexlab}[1]{#1}
\providecommand{\url}[1]{\texttt{#1}}
\expandafter\ifx\csname urlstyle\endcsname\relax
  \providecommand{\doi}[1]{doi: #1}\else
  \providecommand{\doi}{doi: \begingroup \urlstyle{rm}\Url}\fi

\bibitem[Deng et~al.(2009)Deng, Dong, Socher, Li, Li, and
  Fei-Fei]{deng2009imagenet}
Jia Deng, Wei Dong, Richard Socher, Li-Jia Li, Kai Li, and Li~Fei-Fei.
\newblock {Imagenet: A large-scale hierarchical image database}.
\newblock In \emph{{2009 IEEE conference on computer vision and pattern
  recognition}}, pages 248--255. Ieee, 2009.

\bibitem[Feldman et~al.(2018)Feldman, Schmidt, and Sohler]{feldman2018turning}
Dan Feldman, Melanie Schmidt, and Christian Sohler.
\newblock {Turning Big data into tiny data: Constant-size coresets for k-means,
  PCA and projective clustering}, 2018.

\bibitem[Gemp et~al.(2020)Gemp, McWilliams, Vernade, and
  Graepel]{gemp2020eigengame}
Ian Gemp, Brian McWilliams, Claire Vernade, and Thore Graepel.
\newblock {Eigengame: Pca as a nash equilibrium}.
\newblock \emph{arXiv preprint arXiv:2010.00554}, 2020.

\bibitem[Gemp et~al.(2021)Gemp, McWilliams, Vernade, and
  Graepel]{gemp2021eigengame}
Ian Gemp, Brian McWilliams, Claire Vernade, and Thore Graepel.
\newblock {EigenGame Unloaded: When playing games is better than optimizing},
  2021.

\bibitem[Ghashami et~al.(2015)Ghashami, Liberty, Phillips, and
  Woodruff]{DBLP:journals/corr/GhashamiLPW15}
Mina Ghashami, Edo Liberty, Jeff~M. Phillips, and David~P. Woodruff.
\newblock {Frequent Directions : Simple and Deterministic Matrix Sketching}.
\newblock \emph{CoRR}, abs/1501.01711, 2015.
\newblock URL \url{http://arxiv.org/abs/1501.01711}.

\bibitem[Harris et~al.(2020)Harris, Millman, van~der Walt, Gommers, Virtanen,
  Cournapeau, Wieser, Taylor, Berg, Smith, Kern, Picus, Hoyer, van Kerkwijk,
  Brett, Haldane, del R{\'{i}}o, Wiebe, Peterson, G{\'{e}}rard-Marchant,
  Sheppard, Reddy, Weckesser, Abbasi, Gohlke, and Oliphant]{harris2020array}
Charles~R. Harris, K.~Jarrod Millman, St{\'{e}}fan~J. van~der Walt, Ralf
  Gommers, Pauli Virtanen, David Cournapeau, Eric Wieser, Julian Taylor,
  Sebastian Berg, Nathaniel~J. Smith, Robert Kern, Matti Picus, Stephan Hoyer,
  Marten~H. van Kerkwijk, Matthew Brett, Allan Haldane, Jaime~Fern{\'{a}}ndez
  del R{\'{i}}o, Mark Wiebe, Pearu Peterson, Pierre G{\'{e}}rard-Marchant,
  Kevin Sheppard, Tyler Reddy, Warren Weckesser, Hameer Abbasi, Christoph
  Gohlke, and Travis~E. Oliphant.
\newblock {Array programming with {NumPy}}.
\newblock \emph{Nature}, 585\penalty0 (7825):\penalty0 357--362, September
  2020.
\newblock \doi{10.1038/s41586-020-2649-2}.
\newblock URL \url{https://doi.org/10.1038/s41586-020-2649-2}.

\bibitem[He et~al.(2015)He, Zhang, Ren, and Sun]{he2015deep}
Kaiming He, Xiangyu Zhang, Shaoqing Ren, and Jian Sun.
\newblock {Deep Residual Learning for Image Recognition}, 2015.

\bibitem[Karl~Pearson(1901)]{PCA}
F.R.S. Karl~Pearson.
\newblock {LIII. On lines and planes of closest fit to systems of points in
  space}.
\newblock \emph{The London, Edinburgh, and Dublin Philosophical Magazine and
  Journal of Science}, 2\penalty0 (11):\penalty0 559--572, 1901.
\newblock \doi{10.1080/14786440109462720}.

\bibitem[Krizhevsky(2009)]{Krizhevsky2009MastersThesis}
Alex Krizhevsky.
\newblock Learning multiple layers of features from tiny images.
\newblock Master's thesis, Department of Computer Science, University of
  Toronto, 2009.

\bibitem[LeCun and Cortes(2010)]{lecun-mnisthandwrittendigit-2010}
Yann LeCun and Corinna Cortes.
\newblock {{MNIST} handwritten digit database}.
\newblock 2010.
\newblock URL \url{http://yann.lecun.com/exdb/mnist/}.

\bibitem[Nesterov(1983)]{Nesterov1983AMF}
Y.~Nesterov.
\newblock {A method for unconstrained convex minimization problem with the rate
  of convergence o(1/$k^2$)}.
\newblock 1983.

\bibitem[Oja(1982)]{oja-simplified-neuron-model-1982}
Erkki Oja.
\newblock {Simplified neuron model as a principal component analyzer}.
\newblock \emph{Journal of Mathematical Biology}, 15\penalty0 (3):\penalty0
  267--273, November 1982.
\newblock ISSN 0303-6812.
\newblock \doi{10.1007/BF00275687}.
\newblock URL \url{http://dx.doi.org/10.1007/BF00275687}.

\bibitem[Oja(1992)]{Oja1992PrincipalCM}
Erkki Oja.
\newblock Principal components, minor components, and linear neural networks.
\newblock \emph{Neural Networks}, 5:\penalty0 927--935, 1992.

\bibitem[Paszke et~al.(2017)Paszke, Gross, Chintala, Chanan, Yang, DeVito, Lin,
  Desmaison, Antiga, and Lerer]{paszke2017automatic}
Adam Paszke, Sam Gross, Soumith Chintala, Gregory Chanan, Edward Yang, Zachary
  DeVito, Zeming Lin, Alban Desmaison, Luca Antiga, and Adam Lerer.
\newblock {Automatic differentiation in PyTorch}.
\newblock 2017.

\bibitem[Perrone et~al.(2016)Perrone, Jenkins, Spano, and
  Teh]{perrone2016poisson}
Valerio Perrone, Paul~A. Jenkins, Dario Spano, and Yee~Whye Teh.
\newblock {Poisson Random Fields for Dynamic Feature Models}, 2016.

\bibitem[Rutishauser(1970)]{Rutishauser1970SimultaneousIM}
H.~Rutishauser.
\newblock {Simultaneous iteration method for symmetric matrices}.
\newblock \emph{Numerische Mathematik}, 16:\penalty0 205--223, 1970.

\bibitem[Sanger(1989)]{Sanger}
Terence Sanger.
\newblock Optimal unsupervised learning in a single-layer linear feedforward
  neural network.
\newblock \emph{Neural Networks}, 2:\penalty0 459--473, 12 1989.
\newblock \doi{10.1016/0893-6080(89)90044-0}.

\bibitem[Sch\"{o}lkopf et~al.(1999)Sch\"{o}lkopf, Smola, and
  M\"{u}ller]{kernel-PCA}
Bernhard Sch\"{o}lkopf, Alexander~J. Smola, and Klaus-Robert M\"{u}ller.
\newblock \emph{{Kernel Principal Component Analysis}}, page 327–352.
\newblock MIT Press, Cambridge, MA, USA, 1999.
\newblock ISBN 0262194163.

\bibitem[Sharma and Paliwal(2007)]{SharmaPCA}
Alok Sharma and Kuldip Paliwal.
\newblock Fast principal component analysis using fixed-point algorithm.
\newblock \emph{Pattern Recognition Letters}, 28:\penalty0 1151--1155, 07 2007.
\newblock \doi{10.1016/j.patrec.2007.01.012}.

\bibitem[Tang(2019)]{DBLP:journals/corr/abs-1904-01750}
Cheng Tang.
\newblock {Exponentially convergent stochastic k-PCA without variance
  reduction}.
\newblock \emph{CoRR}, abs/1904.01750, 2019.
\newblock URL \url{http://arxiv.org/abs/1904.01750}.

\bibitem[Turk and Pentland(1991)]{EigenFace}
M.A. Turk and A.P. Pentland.
\newblock {Face recognition using eigenfaces}.
\newblock In \emph{{Proceedings. 1991 IEEE Computer Society Conference on
  Computer Vision and Pattern Recognition}}, pages 586--591, 1991.
\newblock \doi{10.1109/CVPR.1991.139758}.

\bibitem[Vo et~al.(2020)Vo, Scanlan, Turner, and Ollington]{LobsterPaper}
Son~Anh Vo, Joel Scanlan, Paul Turner, and Robert Ollington.
\newblock {Convolutional Neural Networks for individual identification in the
  Southern Rock Lobster supply chain}.
\newblock \emph{Food Control}, 118:\penalty0 107419, 2020.
\newblock ISSN 0956-7135.
\newblock \doi{https://doi.org/10.1016/j.foodcont.2020.107419}.
\newblock URL
  \url{https://www.sciencedirect.com/science/article/pii/S0956713520303352}.

\end{thebibliography}

\appendix
\newpage
\section{Additional numerical results}
\begin{table*}[ht!]
    \centering
    \vspace{10mm}
    \begin{tabular}{l| c c c c c }
      \multicolumn{1}{c|}{ } &{\begin{tabular}{c}Exponential \\Synthetic\end{tabular}}& {\begin{tabular}{c}Linear \\Synthetic \end{tabular}} & {MNIST} & {CIFAR10} & {\begin{tabular}{c}NIPS \\bag of words\end{tabular}}\\ \noalign{\hrule height1pt}
      \multicolumn{1}{c|}{dimensions} & {50} & {50} & {768} & {3092} & {11 463}\\
      \multicolumn{1}{c|}{points} & {5000} & {5000} & {60 000} & {50 000} & {5812}\\
      \cmidrule(lr){1-6}
       \multicolumn{6}{c}{$V=\pi/8$} \\
       \cmidrule(lr){1-6}
        EigenGame                      & $4.55 \pm 0.67$         &  $12.24\pm4.14$        &  $9.33\pm3.64$         &  $15.79\pm8.33$        &  $7.54\pm2.19$                    \\
        \quad + pPCA (Ours)        & $3.48 \pm 0.87$         &  $4.04\pm2.30$         &  $2.21\pm0.59$         &  $5.02\pm0.73$         &  $3.81\pm1.37$            \\
        \quad + $\text{pPCA}_2$ (Ours) & $0.89 \pm 0.58$         &  $1.64\pm0.62$         &  $1.15\pm0.37$         &  $3.17\pm1.54$         &  $1.66\pm0.54$            \\
        \quad + $\text{pPCA}_4$ (Ours) & $0.27 \pm 0.11$         &  $1.10\pm0.20$         &  $\mathbf{0.78\pm0.43}$&  $\mathbf{2.45\pm0.91}$&  $1.24\pm0.62$            \\
        Oja's method                   & $0.23 \pm 0.09$         &  $0.32\pm0.18$         &  n.a.                  &  n.a.                  &  $5.67\pm2.15$            \\
        \quad + pPCA (Ours)            & $0.08 \pm 0.03$         &  $0.11\pm0.09$         &  $11.53\pm5.74$        &  n.a.                  &  $2.68\pm1.71$            \\
        \quad + $\text{pPCA}_2$ (Ours) & $\mathbf{0.05\pm0.01}$  &  $\mathbf{0.03\pm0.01}$&  $8.03\pm3.32$         &  $14.38\pm7.24$        &  $1.07\pm0.30$            \\
        \quad + $\text{pPCA}_4$ (Ours) & $\mathbf{0.05\pm0.01}$  &  $0.04\pm0.00$         &  $6.83\pm0.86$         &  $11.58\pm2.02$        &  $\mathbf{0.91\pm0.14}$            \\
        Power method                   & $0.21 \pm 0.01$         &  $0.67\pm0.07$         &  $6.32\pm0.80$         &  $9.63\pm1.13$         &  $4.47\pm0.38$            \\
        \quad + pPCA (Ours)            & $0.23 \pm 0.01$         &  $0.68\pm0.07$         &  $6.33\pm0.80$         &  $9.63\pm1.13$         &  $4.48\pm0.38$            \\
       \cmidrule(lr){1-6}
       \multicolumn{6}{c}{$V=\pi/32$} \\
       \cmidrule(lr){1-6}
       EigenGame                      & $5.81 \pm 0.83$         &  n.a.                  &  n.a.                   &  $28.37\pm10.09$       &  n.a.            \\
       \quad + pPCA (Ours)            & $4.58 \pm 0.95$         &  $7.45\pm3.00$         &  $3.44\pm0.57$          &  $6.92\pm0.79$         &  $6.57\pm2.31$            \\
       \quad + $\text{pPCA}_2$ (Ours) & $2.35 \pm 0.61$         &  $2.92\pm0.93$         &  $2.18\pm0.57$          &  $5.09\pm1.42$         &  $3.18\pm0.38$            \\
       \quad + $\text{pPCA}_4$ (Ours) & $0.80 \pm 0.34$         &  $1.94\pm0.44$         &  $\mathbf{1.58\pm0.57}$ &  $\mathbf{4.61\pm1.10}$&  $2.59\pm0.59$            \\
       Oja's method                   & $0.34 \pm 0.09$         &  $0.58\pm0.24$         &  n.a.                   &  n.a.                  &  $10.03\pm2.98$            \\
       \quad + pPCA (Ours)            & $0.15 \pm 0.05$         &  $0.19\pm0.10$         &  $19.74\pm7.31$         &  n.a.                  &  $4.54\pm1.98$            \\
       \quad + $\text{pPCA}_2$ (Ours) & $0.09 \pm 0.02$         &  $\mathbf{0.05\pm0.01}$&  $13.07\pm4.45$         &  $26.32\pm9.60$        &  $1.90\pm0.43$            \\
       \quad + $\text{pPCA}_4$ (Ours) & $\mathbf{0.07\pm0.01}$  &  $\mathbf{0.05\pm0.01}$&  $11.96\pm1.44$         &  $18.97\pm3.80$        &  $\mathbf{1.61\pm0.21}$            \\
       Power method                   & $0.22 \pm 0.01$         &  n.a.                  &  $6.38\pm0.80$          &  $9.73\pm1.13$         &  $4.52\pm0.38$            \\
       \quad + pPCA (Ours)            & $0.23 \pm 0.01$         &  $0.69\pm0.07$         &  $6.38\pm0.80$          &  $9.74\pm1.13$         &  $4.54\pm0.38$            \\
      \cmidrule(lr){1-6}
       \multicolumn{6}{c}{$V=\pi/128$} \\
       \cmidrule(lr){1-6}
       EigenGame                      & n.a.                    &  n.a.                  &  n.a.                   &  n.a.                  &  n.a.            \\
       \quad + pPCA (Ours)            & n.a.                    &  n.a.                  &  n.a.                   &  n.a.                  &  n.a.            \\
       \quad + $\text{pPCA}_2$ (Ours) & $3.05\pm0.71$           &  $4.84\pm1.57$         &  n.a.                   &  n.a.                  &  n.a.            \\
       \quad + $\text{pPCA}_4$ (Ours) & $1.60\pm0.35$           &  $2.93\pm0.63$         &  n.a.                   &  n.a.                  &  n.a.            \\
       Oja's method                   & $0.46\pm0.10$           &  $0.92\pm0.29$         &  n.a.                   &  n.a.                  &  n.a.            \\
       \quad + pPCA (Ours)            & $0.23\pm0.06$           &  $0.29\pm0.11$         &  $29.88\pm8.00$         &  n.a.                  &  n.a.            \\
       \quad + $\text{pPCA}_2$ (Ours) & $0.14\pm0.03$           &  $0.07\pm0.02$         &  $19.45\pm4.69$         &  $40.01\pm11.90$       &  n.a.           \\
       \quad + $\text{pPCA}_4$ (Ours) & $\mathbf{0.09\pm0.01}$  &  $\mathbf{0.06\pm0.01}$&  $17.60\pm1.71$         &  $27.88\pm4.82$        &  n.a.            \\
       Power method                   & $0.22\pm 0.01$          &  n.a.                  &  n.a.                   &  n.a.                  &  n.a.            \\
       \quad + pPCA (Ours)            & $0.23\pm 0.01$          &  $0.69\pm0.07$         &  $\mathbf{6.46\pm0.80}$ &  $\mathbf{9.80\pm0.94}$&  $\mathbf{6.39\pm 0.58}$ \\           
      \cmidrule(lr){1-6 }  
    \end{tabular}
    \caption{Time to reach an Eigenvalues streak of length 16 with different threshold values $V \in \{\pi/8,\pi/32,\pi/128 \}$ for the algorithms on the small-scale datasets. All experiments were repeated 10 times, mean and variance values are reported. All values have units of seconds, if one the 10 runs did not reach a streak of length 16 during training, then  ``n.a.'' is reported.}
    \label{results_big}
  \end{table*}
\newpage
\section{Additional plots}
\begin{figure*}
  \begin{subfigure}[t]{\textwidth}
    \centering
    \includegraphics[width=.9\textwidth]{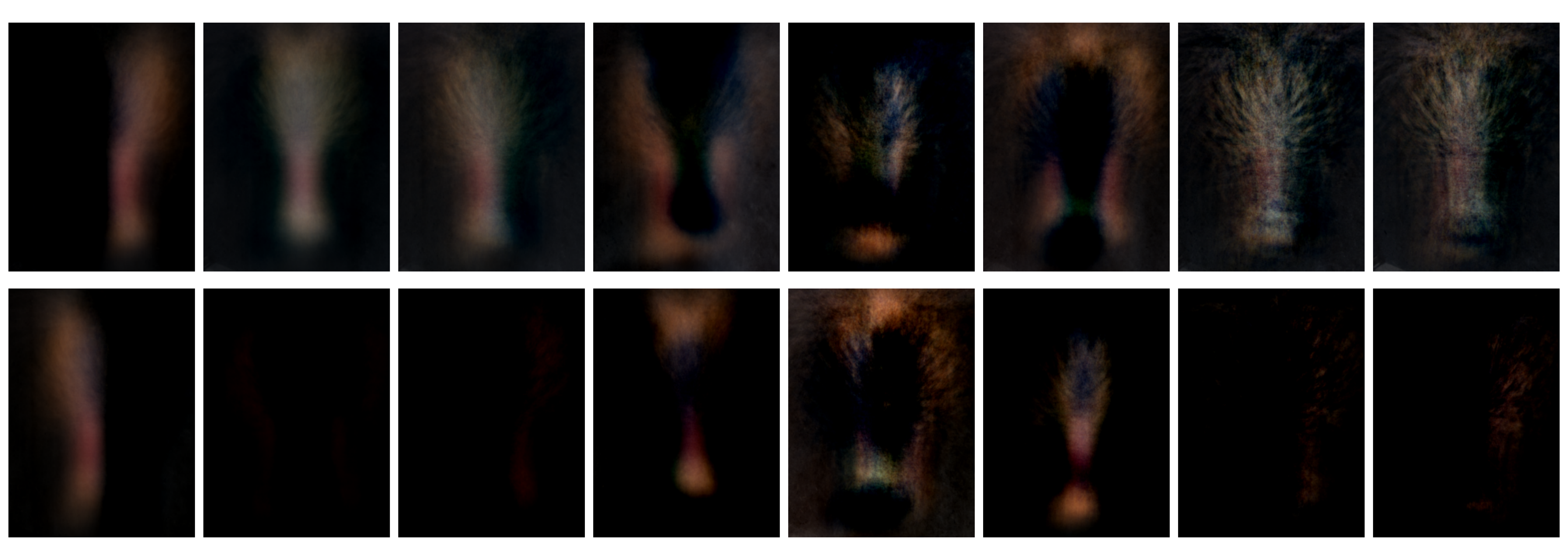}
   \caption{EigenGame players after 30 training steps}
  \end{subfigure}
  \begin{subfigure}[t]{\textwidth}
    \centering
    \includegraphics[width=.9\textwidth]{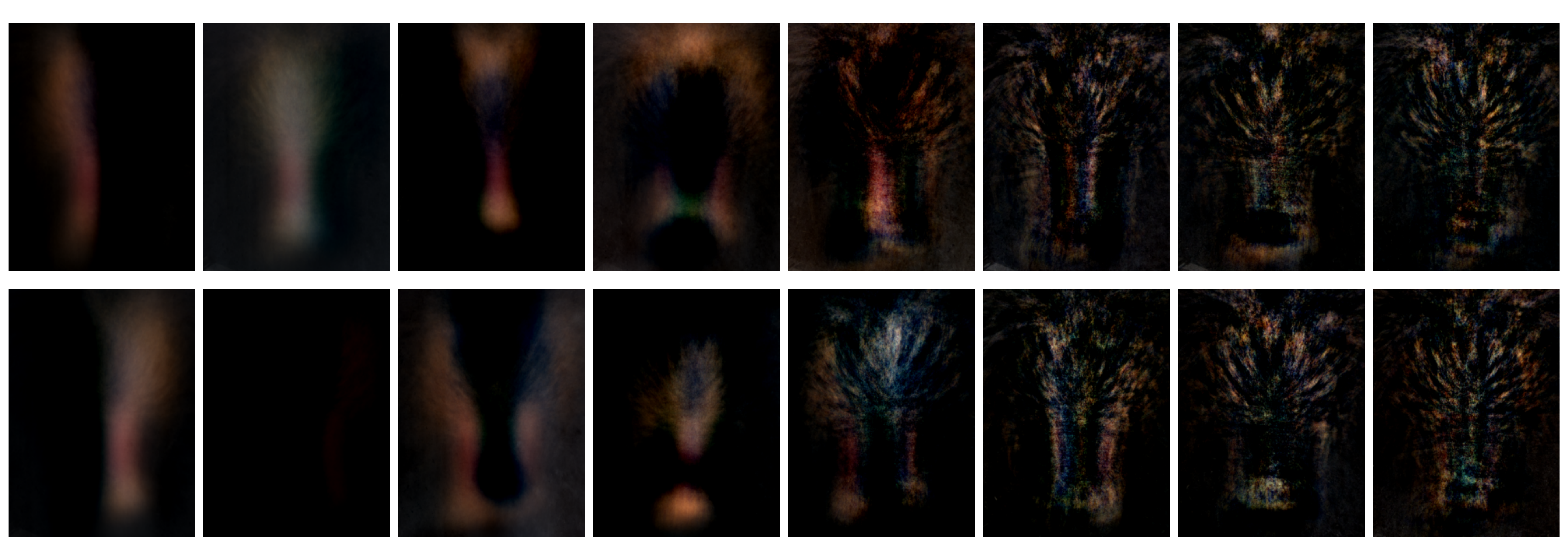}
   \caption{Principal directions predicted by pPCA after 30 training steps}
  \end{subfigure}
  \begin{subfigure}[t]{\textwidth}
    \centering
    \includegraphics[width=.9\textwidth]{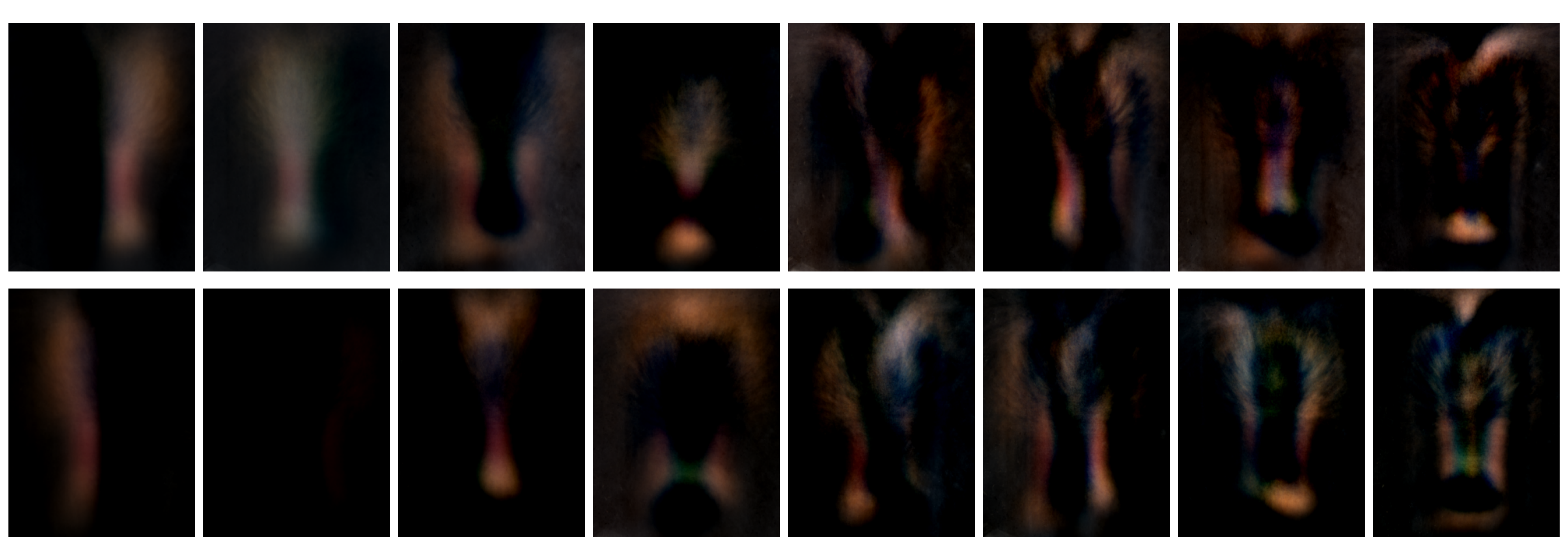}
   \caption{EigenGame players after 300 training steps}
  \end{subfigure}
  \begin{subfigure}[t]{\textwidth}
    \centering
    \includegraphics[width=.9\textwidth]{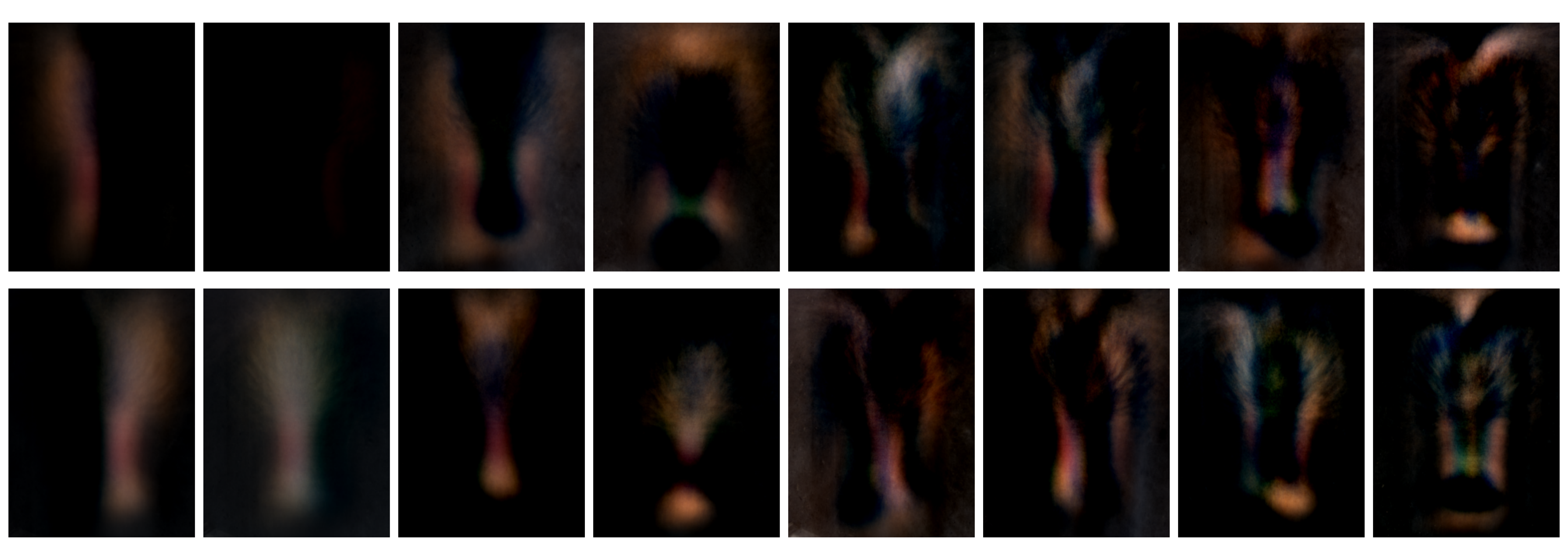}
   \caption{Principal directions predicted by pPCA after 300 training steps}
  \end{subfigure}
  \caption{Each of the subplots show the first 8 principal components of the Large Lobster Image Dataset organised into 8 columns found by the respective algorithms. The two lines show the positive and negative parts of the vectors. After 30 training steps pPCA already finds the first 4 vectors, while EigenGame only predicts  the first 2 correctly. After 300 training steps, the methods have the same output for all 8 principal directions.}
  \label{PrincipalLobsters}
\end{figure*}
\begin{figure*}
    \begin{subfigure}[t]{0.45\textwidth}
      \centering
      \includegraphics[width=4.5cm,trim=30 0 0 0]{./plots/legend.pdf}
    \end{subfigure}
    \begin{subfigure}[t]{0.5\textwidth}
      \includegraphics[width=\textwidth,trim=30 0 0 0]{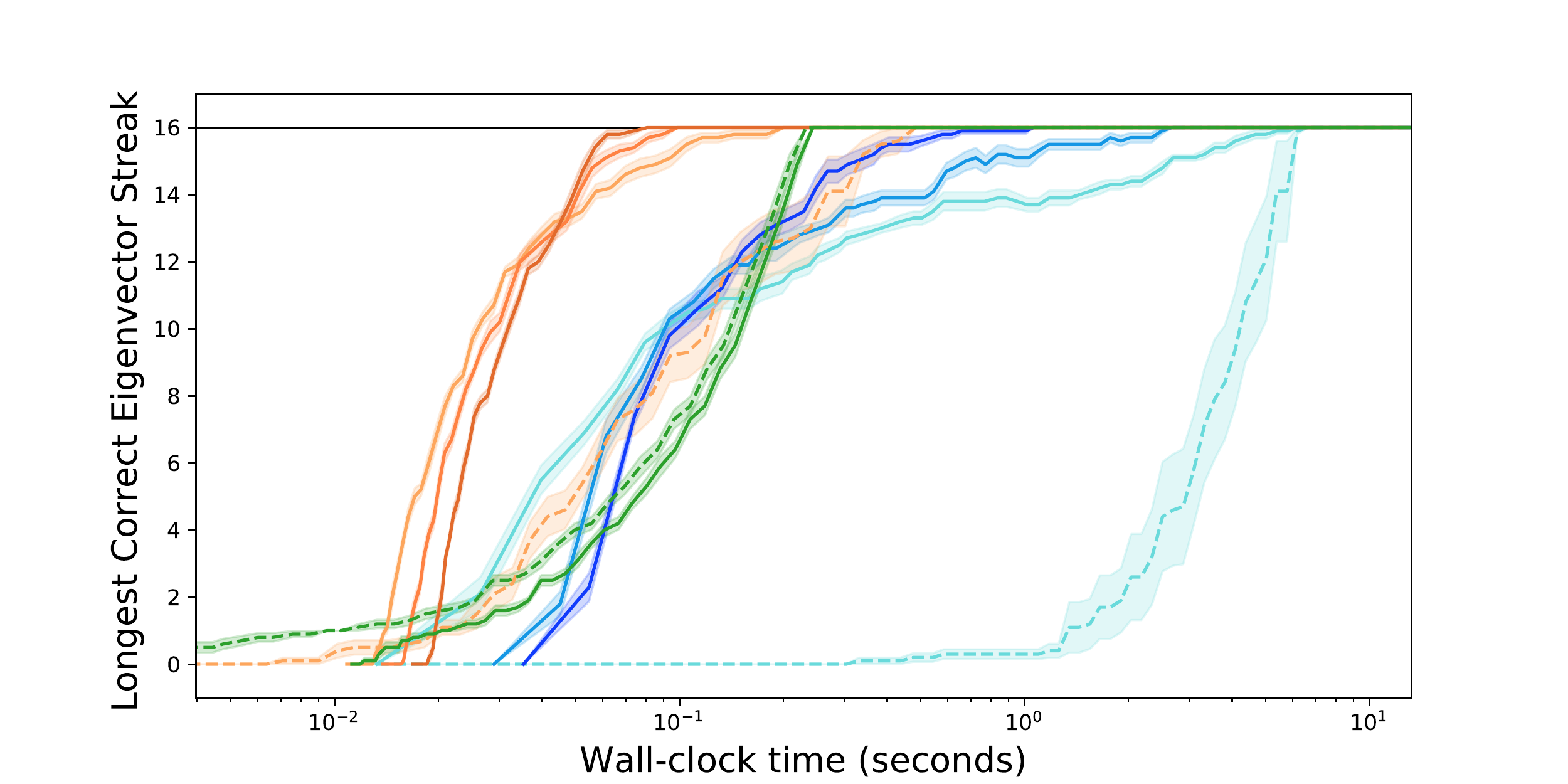}
      \caption{$V=\pi/16$}
    \end{subfigure}
    \begin{subfigure}[t]{0.5\textwidth}
      \includegraphics[width=\textwidth,trim=30 0 0 0]{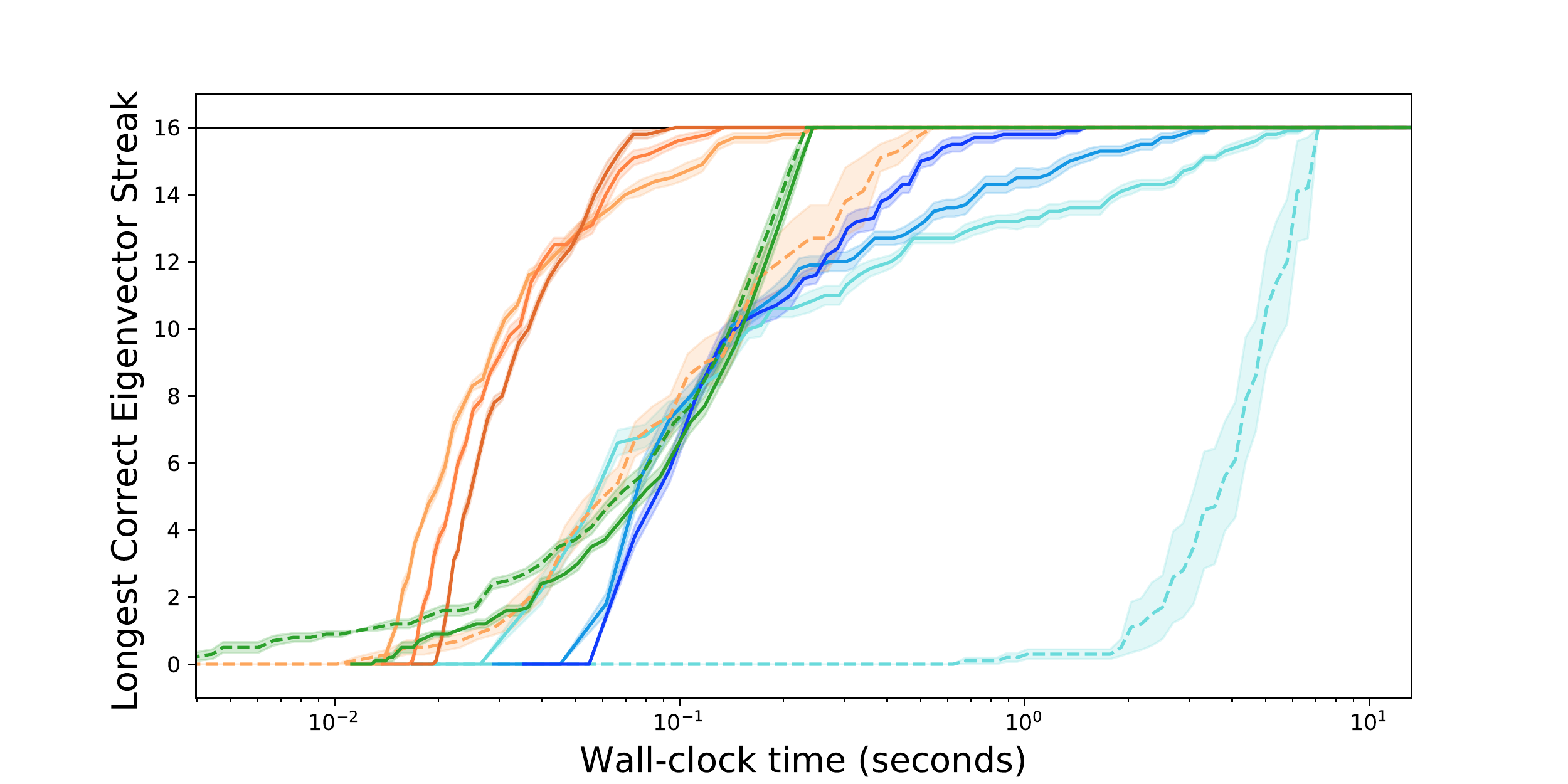}
      \caption{$V=\pi/32$}
    \end{subfigure}
    \begin{subfigure}[t]{0.5\textwidth}
      \includegraphics[width=\textwidth,trim=30 0 0 0]{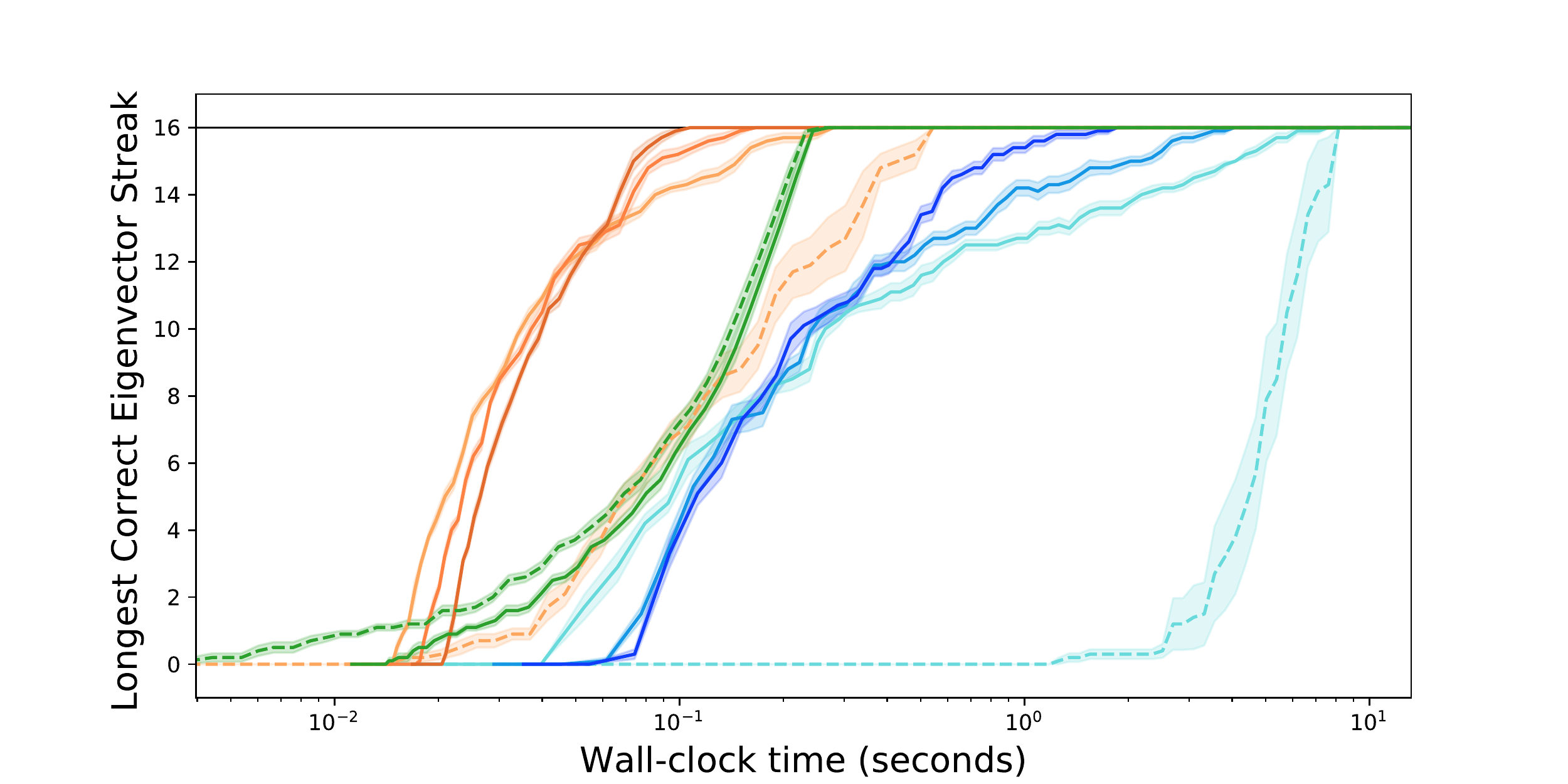}
      \caption{$V=\pi/64$}
    \end{subfigure}
    \begin{subfigure}[t]{0.5\textwidth}
      \includegraphics[width=\textwidth,trim=30 0 0 0]{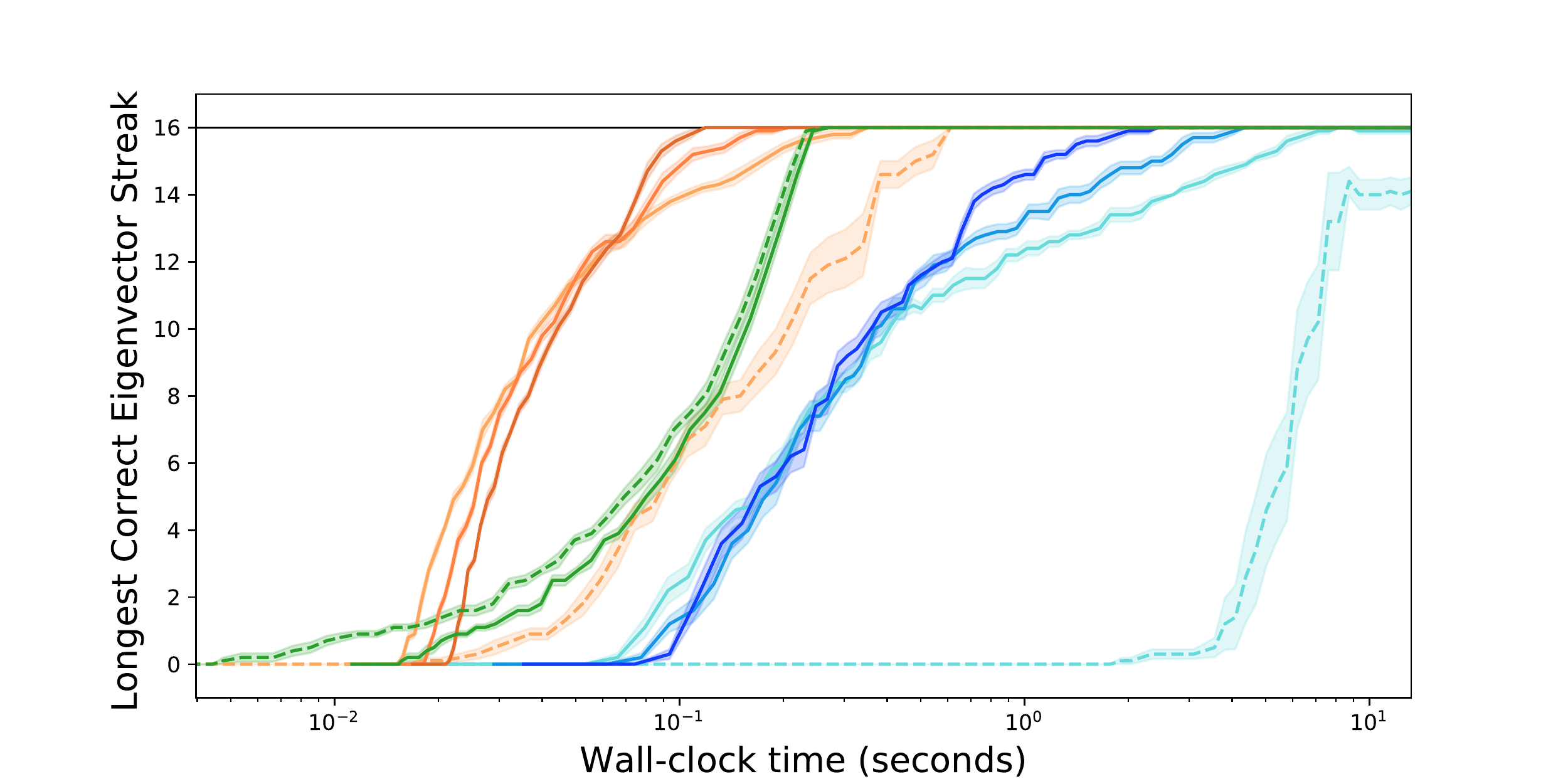}
      \caption{$V=\pi/128$}
    \end{subfigure}
    \begin{subfigure}[t]{0.5\textwidth}
      \includegraphics[width=\textwidth,trim=30 0 0 0]{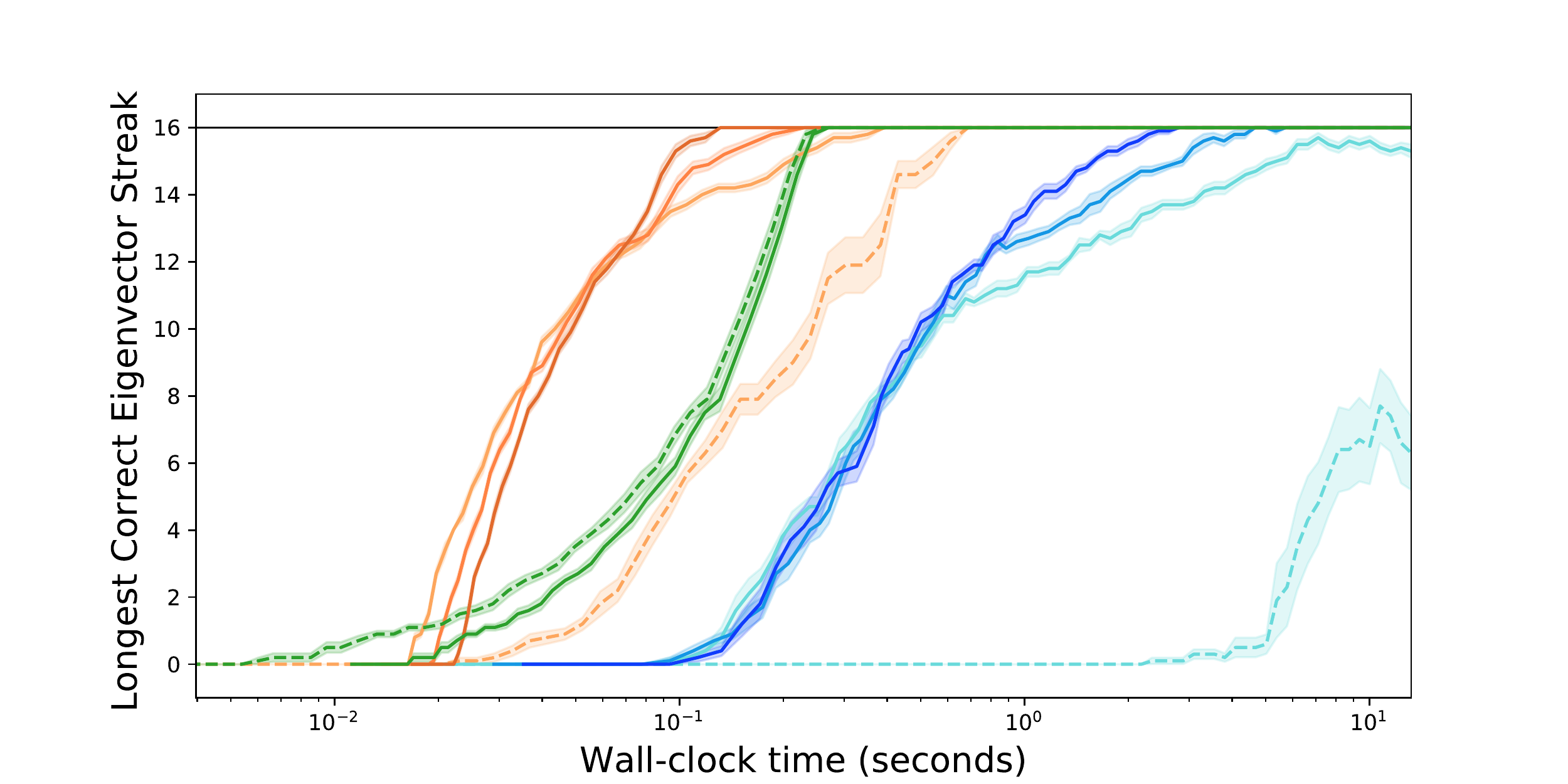}
      \caption{$V=\pi/256$}
    \end{subfigure}
    \begin{subfigure}[t]{0.5\textwidth}
      \includegraphics[width=\textwidth,trim=30 0 0 0]{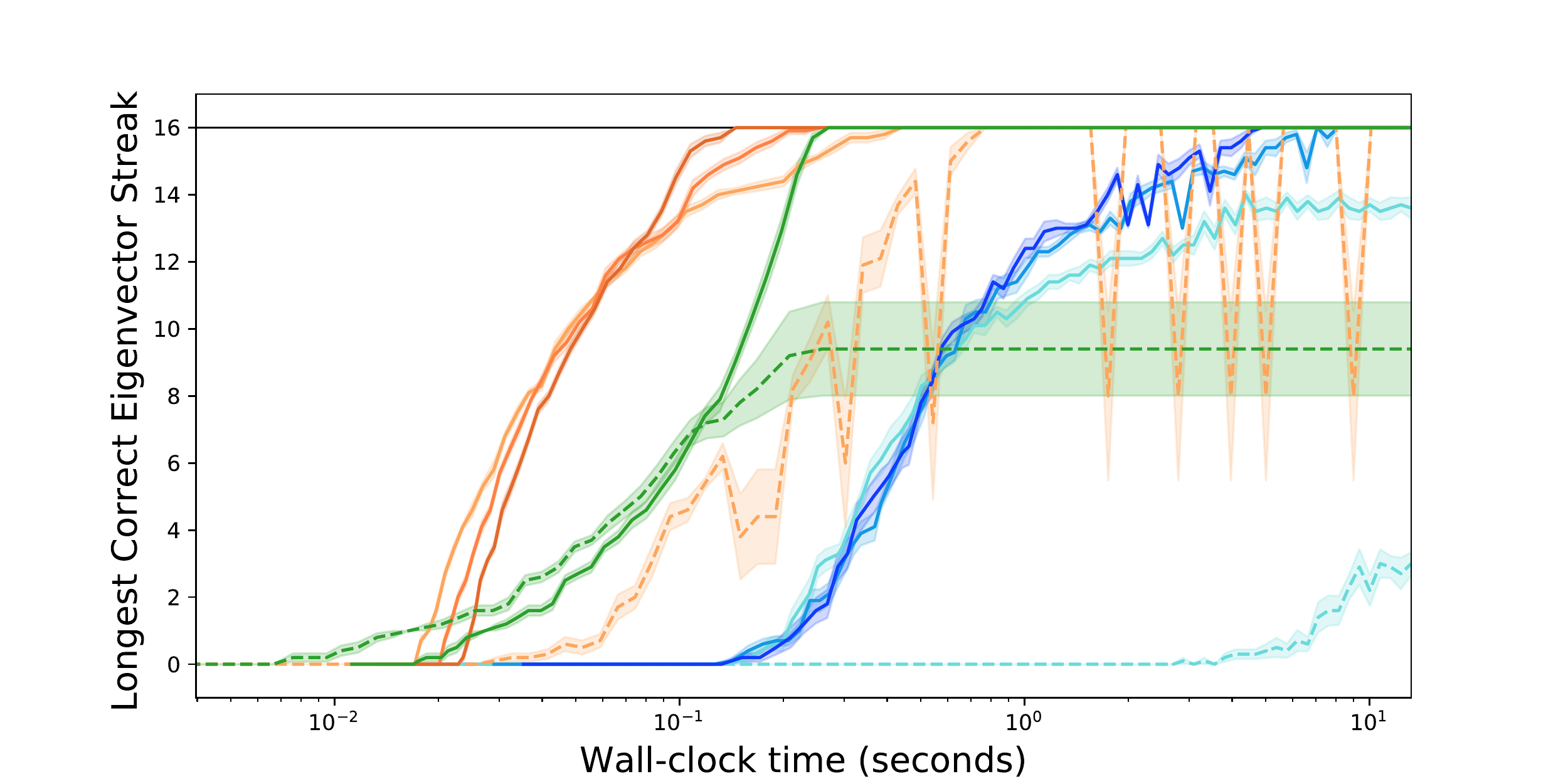}
      \caption{$V=\pi/512$}
    \end{subfigure}
    \begin{subfigure}[t]{0.5\textwidth}
      \includegraphics[width=\textwidth,trim=30 0 0 0]{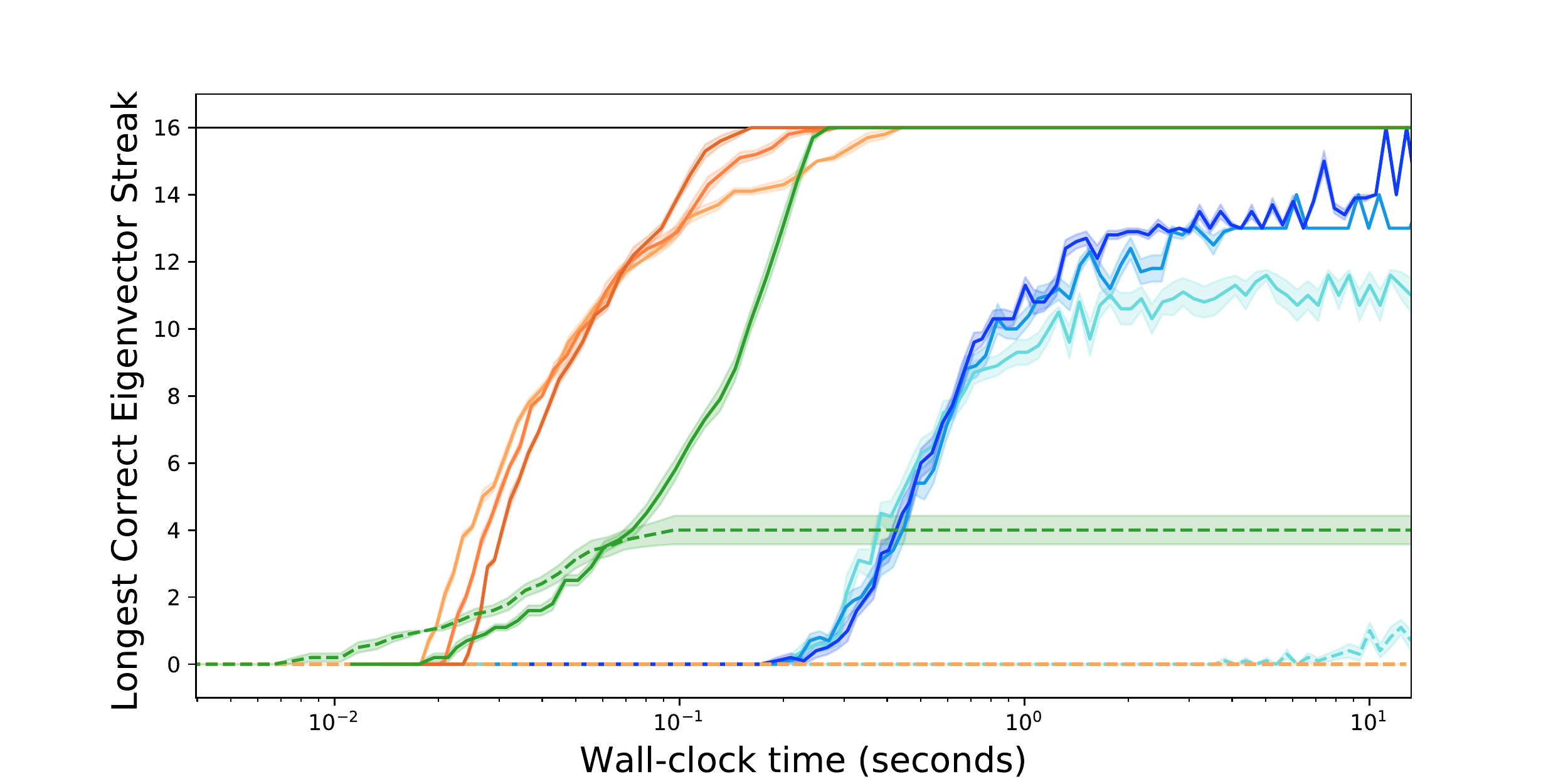}
      \caption{$V=\pi/1024$}
    \end{subfigure}
    \caption{Synthetic data with exponential spectrum}
\end{figure*}

\begin{figure*}
  \begin{subfigure}[t]{0.5\textwidth}
    \centering
    \includegraphics[width=4.5cm,trim=30 0 0 0]{./plots/legend.pdf}
  \end{subfigure}
    \begin{subfigure}[t]{0.5\textwidth}
      \includegraphics[width=\textwidth,trim=30 0 0 0]{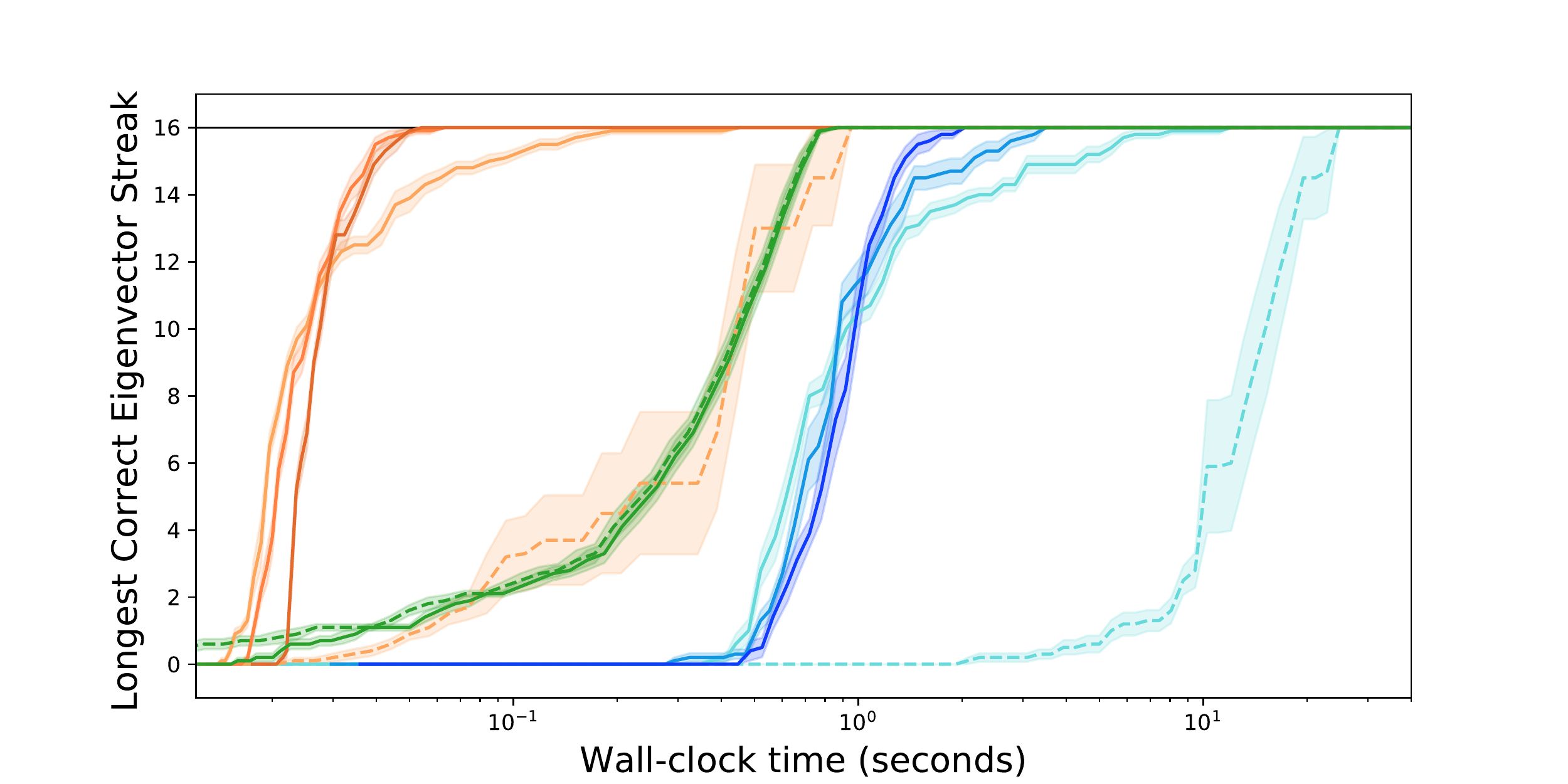}
      \caption{$V=\pi/16$}
    \end{subfigure}
    \begin{subfigure}[t]{0.5\textwidth}
      \includegraphics[width=\textwidth,trim=30 0 0 0]{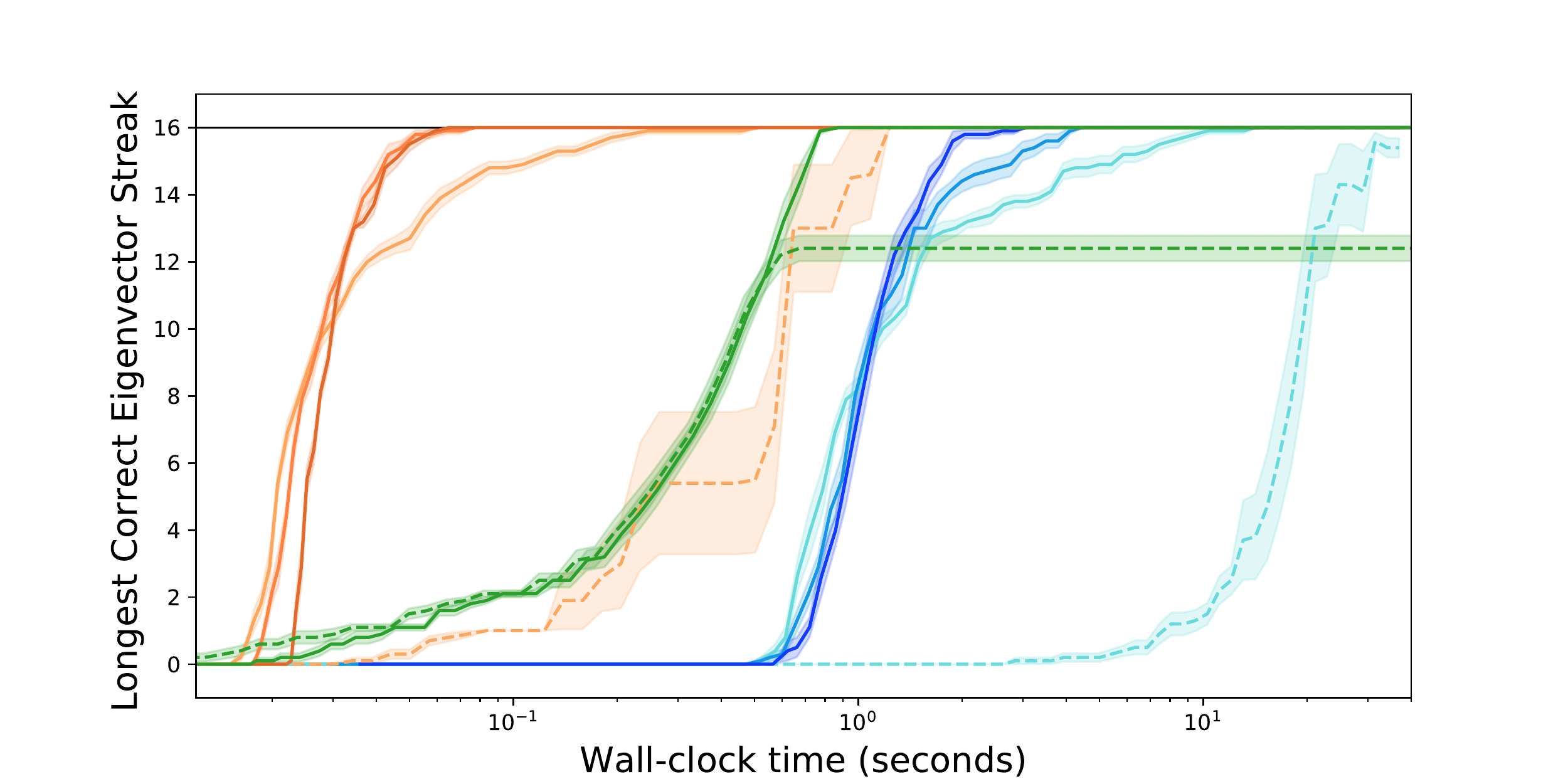}
      \caption{$V=\pi/32$}
    \end{subfigure}
    \begin{subfigure}[t]{0.5\textwidth}
      \includegraphics[width=\textwidth,trim=30 0 0 0]{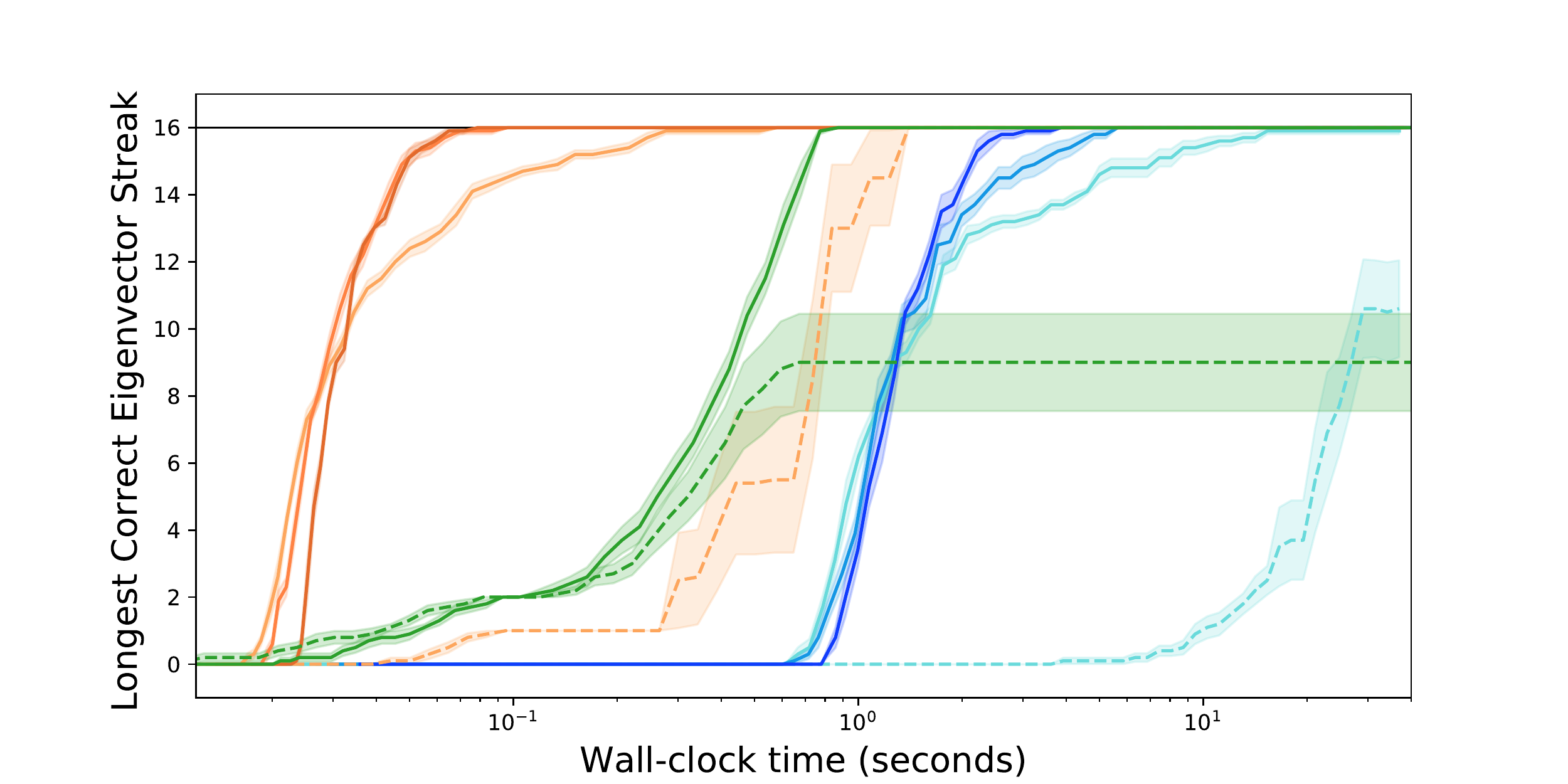}
      \caption{$V=\pi/64$}
    \end{subfigure}
    \begin{subfigure}[t]{0.5\textwidth}
      \includegraphics[width=\textwidth,trim=30 0 0 0]{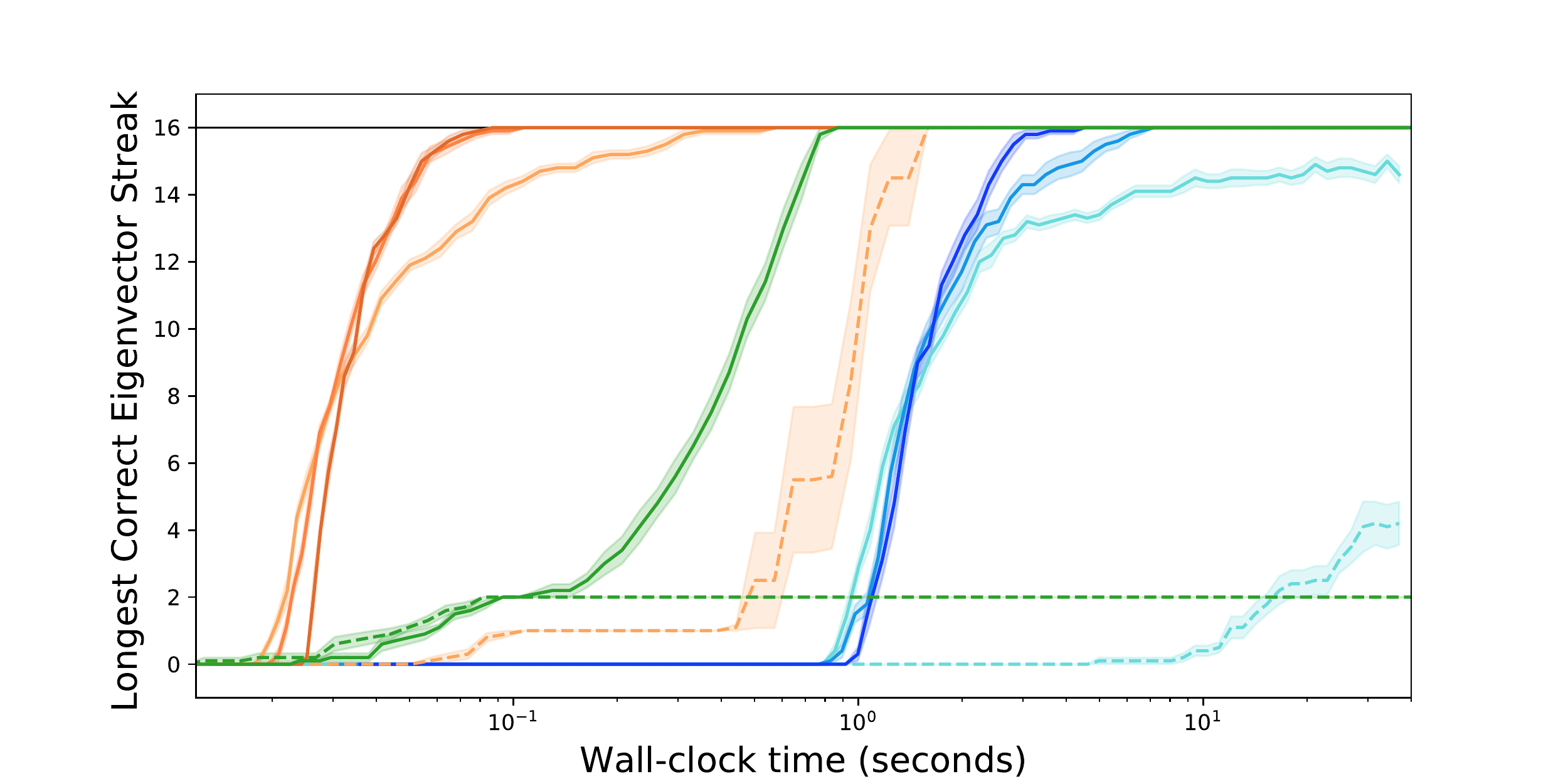}
      \caption{$V=\pi/128$}
    \end{subfigure}
    \begin{subfigure}[t]{0.5\textwidth}
      \includegraphics[width=\textwidth,trim=30 0 0 0]{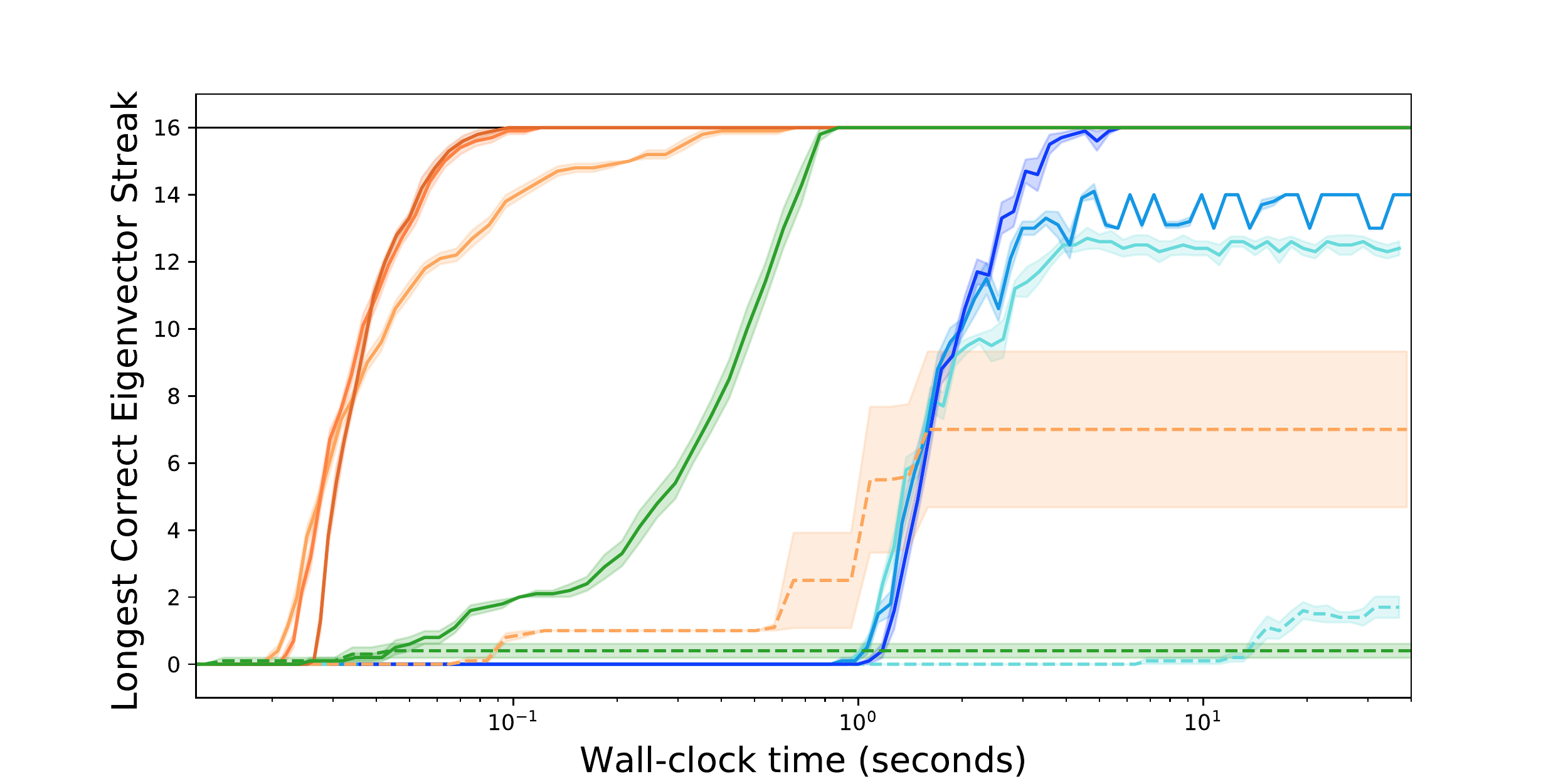}
     \caption{$V=\pi/256$}
    \end{subfigure}
    \begin{subfigure}[t]{0.5\textwidth}
      \includegraphics[width=\textwidth,trim=30 0 0 0]{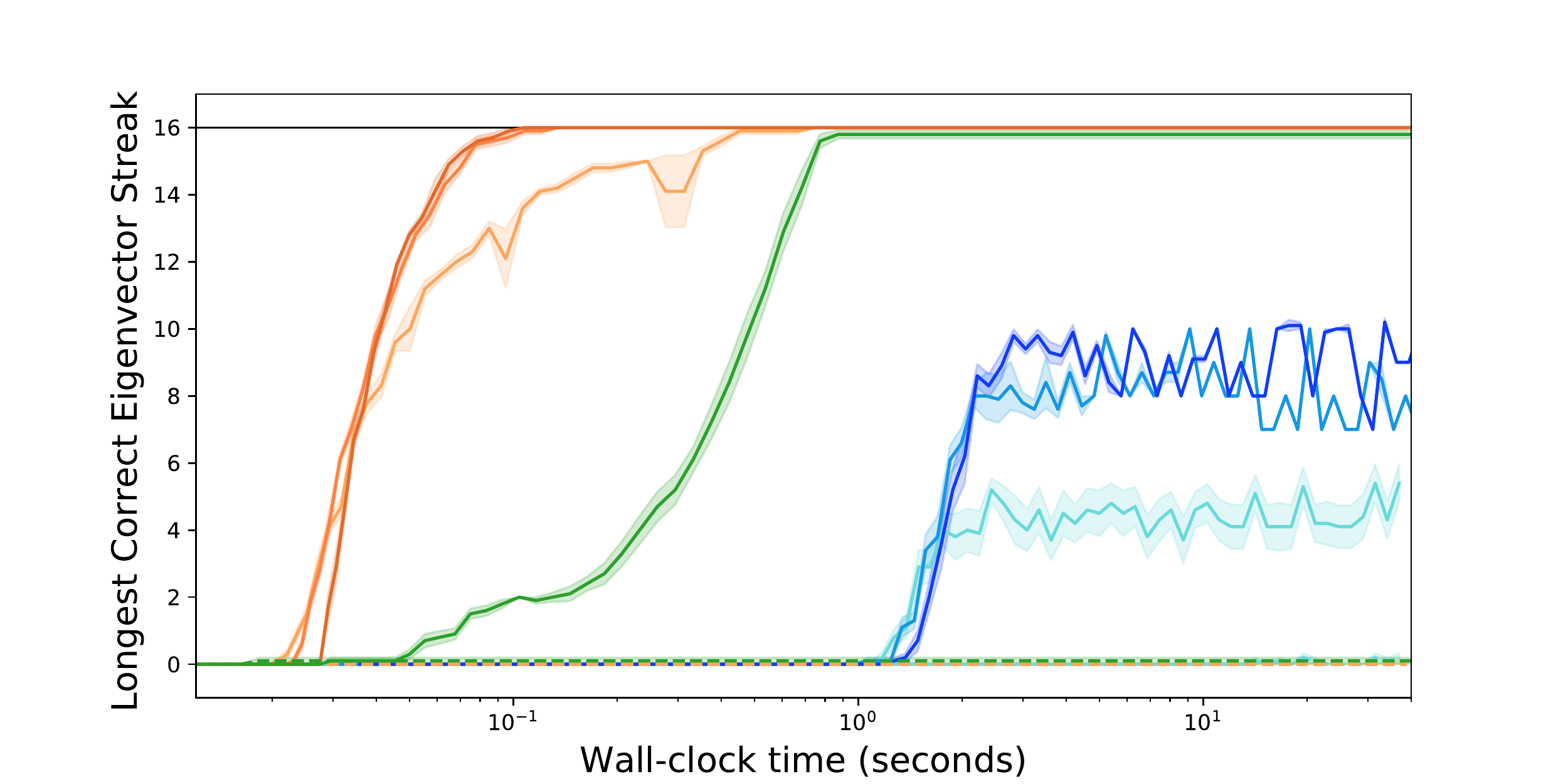}
      \caption{$V=\pi/512$}
    \end{subfigure}
    \begin{subfigure}[t]{0.5\textwidth}
      \includegraphics[width=\textwidth,trim=30 0 0 0]{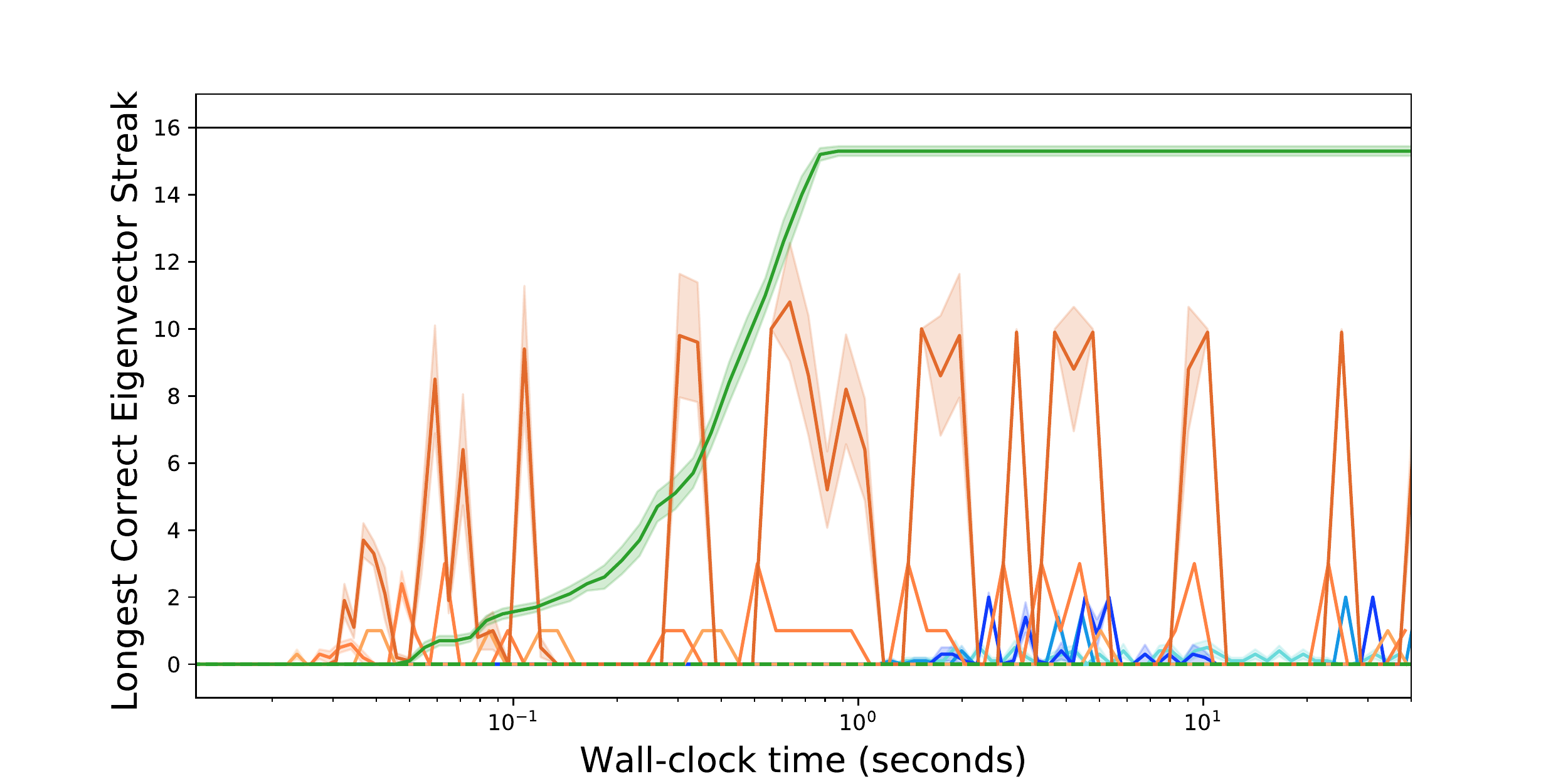}
      \caption{$V=\pi/1024$}
    \end{subfigure}
    \caption{Synthetic data with linear spectrum}
\end{figure*}

\begin{figure*}
  \begin{subfigure}[t]{0.5\textwidth}
    \centering
    \includegraphics[width=4.5cm,trim=30 0 0 0]{./plots/legend.pdf}
  \end{subfigure}
    \begin{subfigure}[t]{0.5\textwidth}
      \includegraphics[width=\textwidth,trim=30 0 0 0]{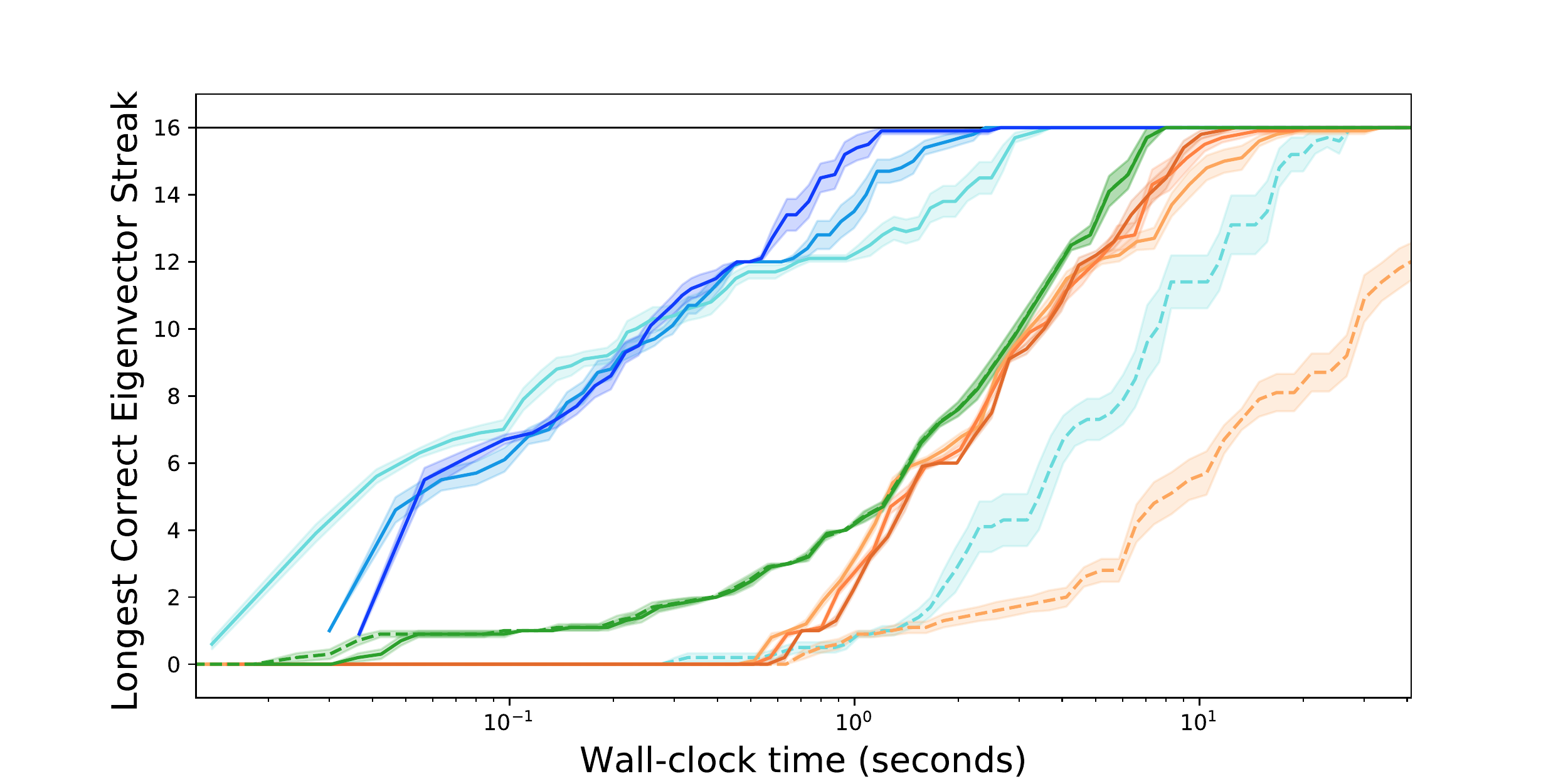}
      \caption{$V=\pi/16$}
    \end{subfigure}
    \begin{subfigure}[t]{0.5\textwidth}
      \includegraphics[width=\textwidth,trim=30 0 0 0]{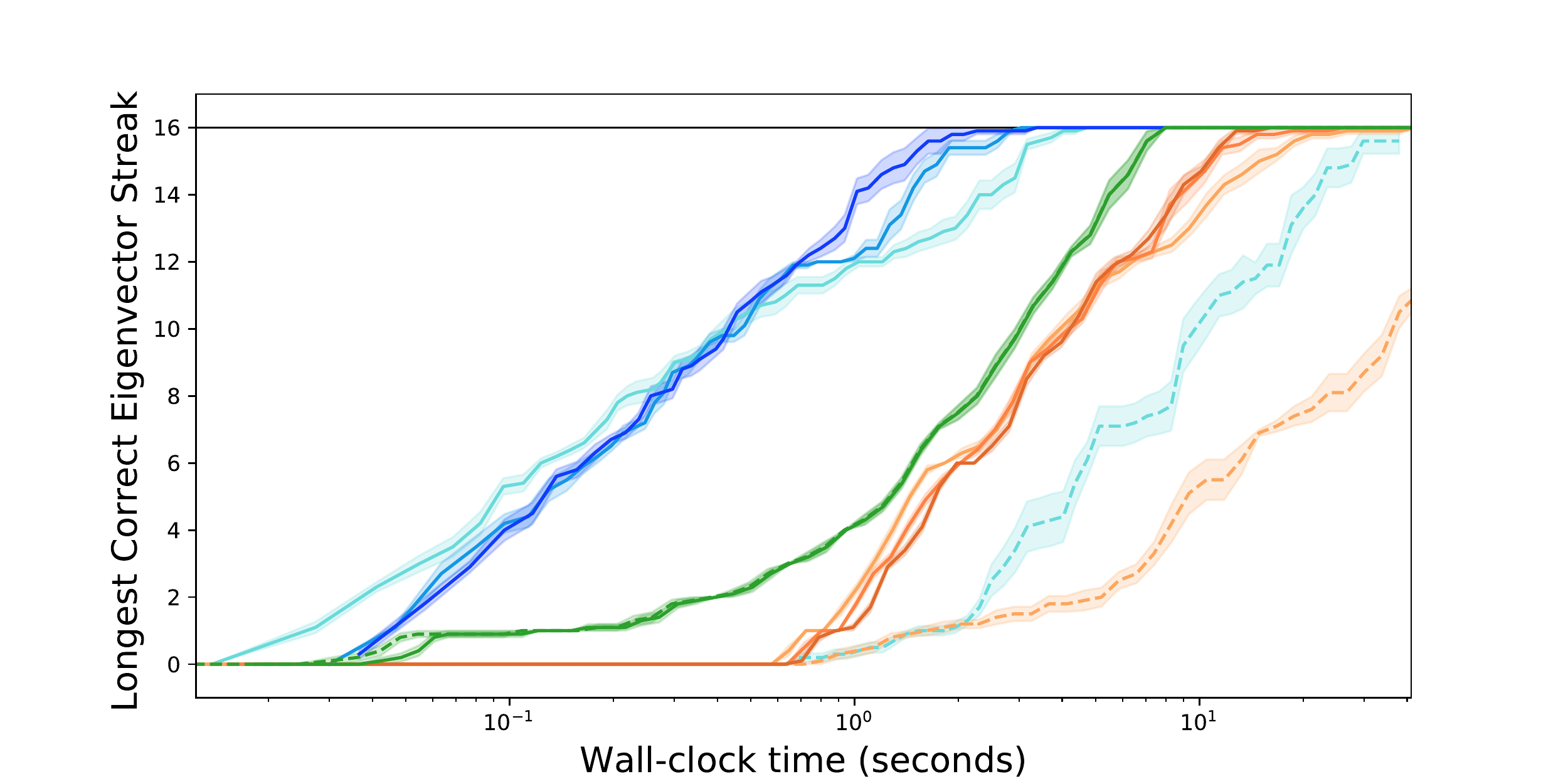}
      \caption{$V=\pi/32$}
    \end{subfigure}
    \begin{subfigure}[t]{0.5\textwidth}
      \includegraphics[width=\textwidth,trim=30 0 0 0]{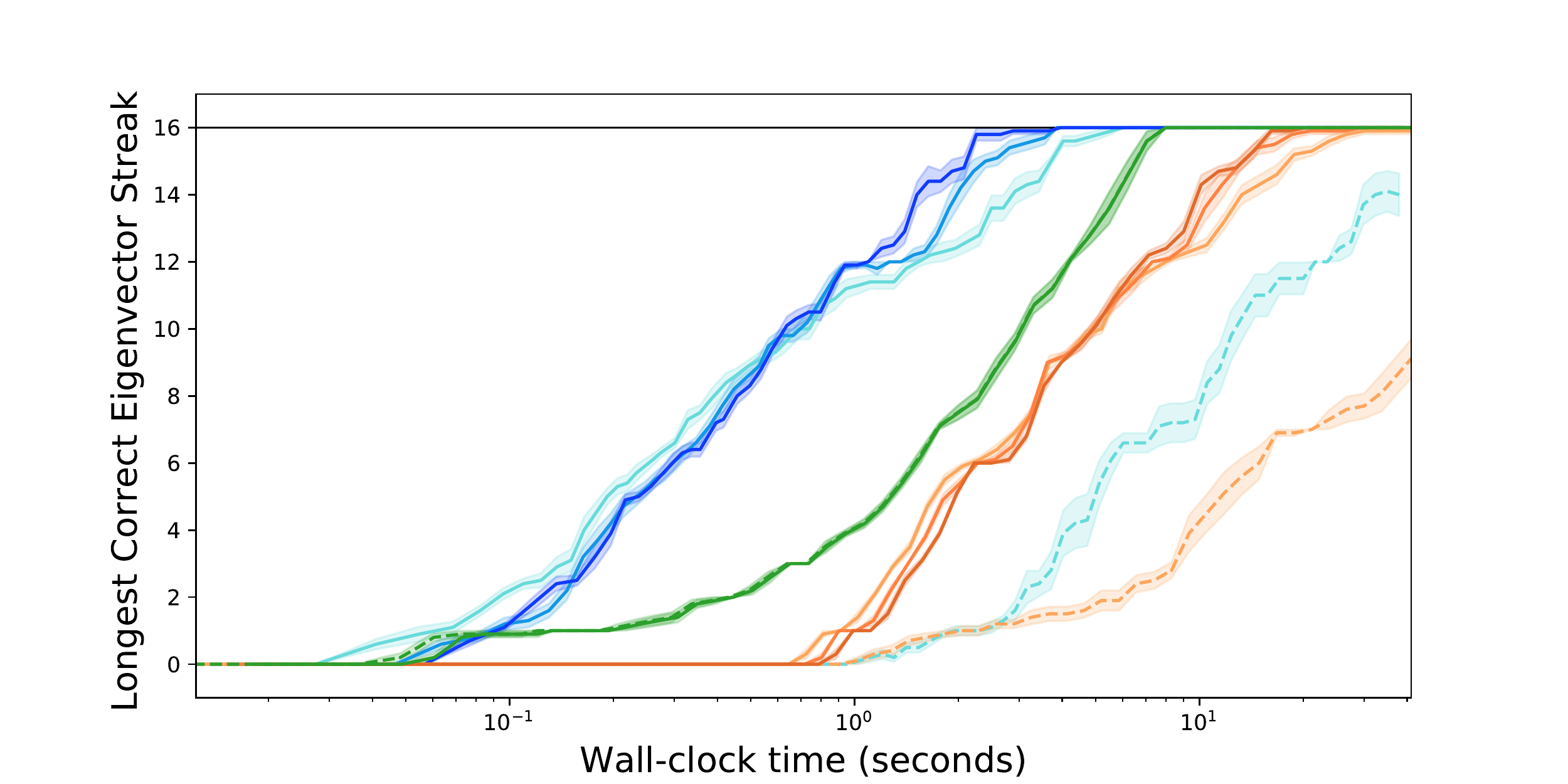}
      \caption{$V=\pi/64$}
    \end{subfigure}
    \begin{subfigure}[t]{0.5\textwidth}
      \includegraphics[width=\textwidth,trim=30 0 0 0]{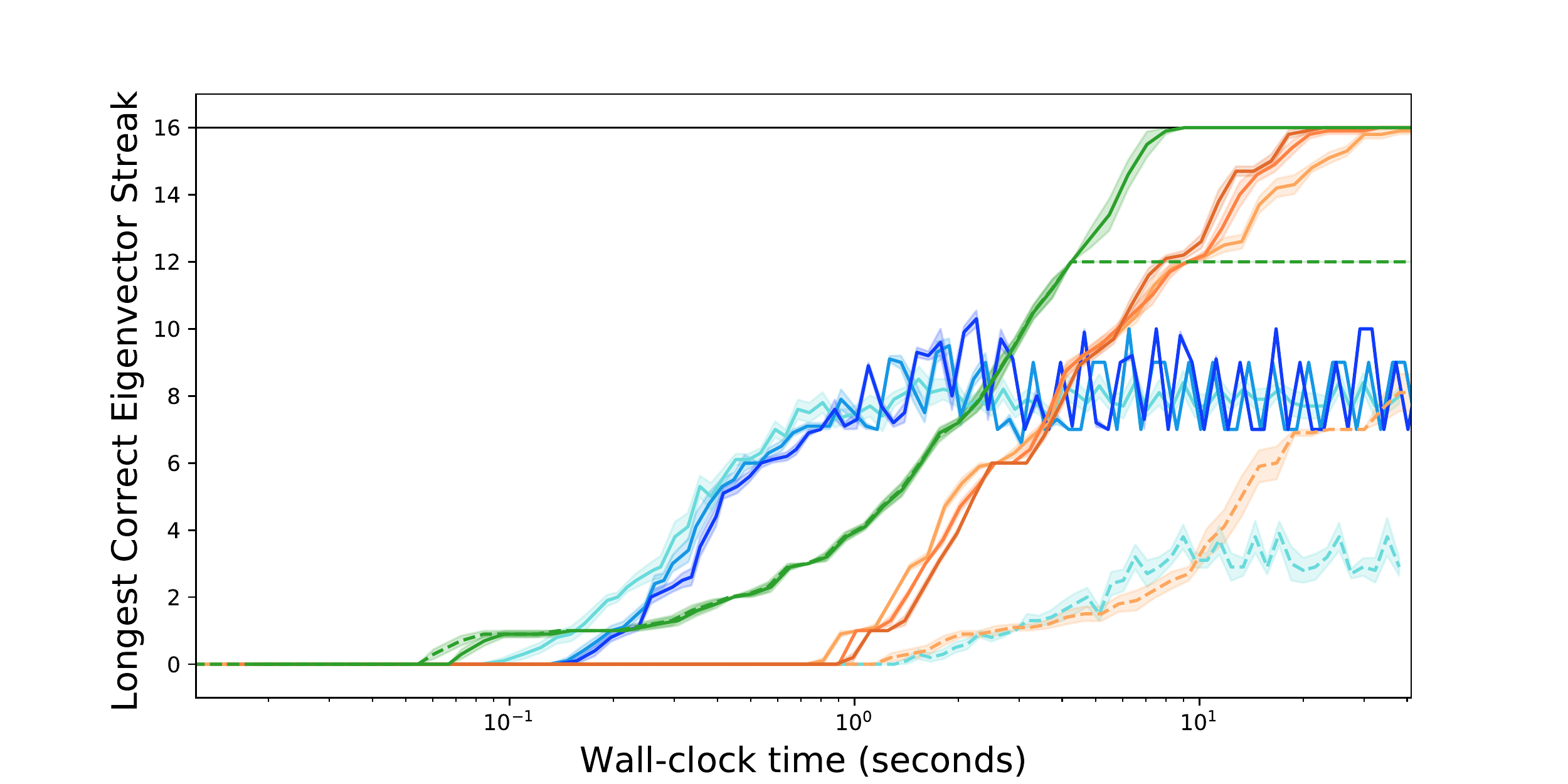}
      \caption{$V=\pi/128$}
    \end{subfigure}
    \begin{subfigure}[t]{0.5\textwidth}
      \includegraphics[width=\textwidth,trim=30 0 0 0]{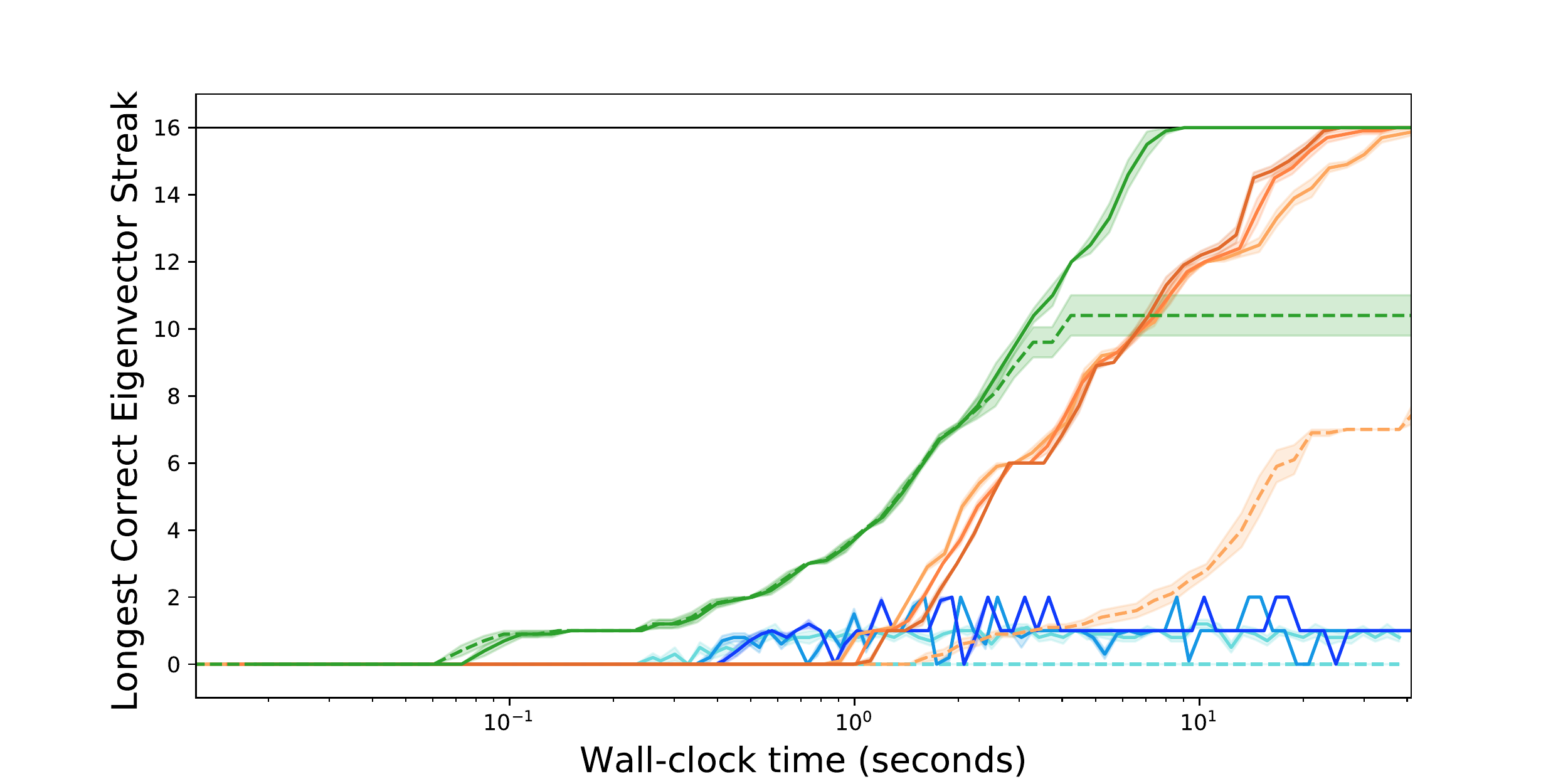}
     \caption{$V=\pi/256$}
    \end{subfigure}
    \begin{subfigure}[t]{0.5\textwidth}
      \includegraphics[width=\textwidth,trim=30 0 0 0]{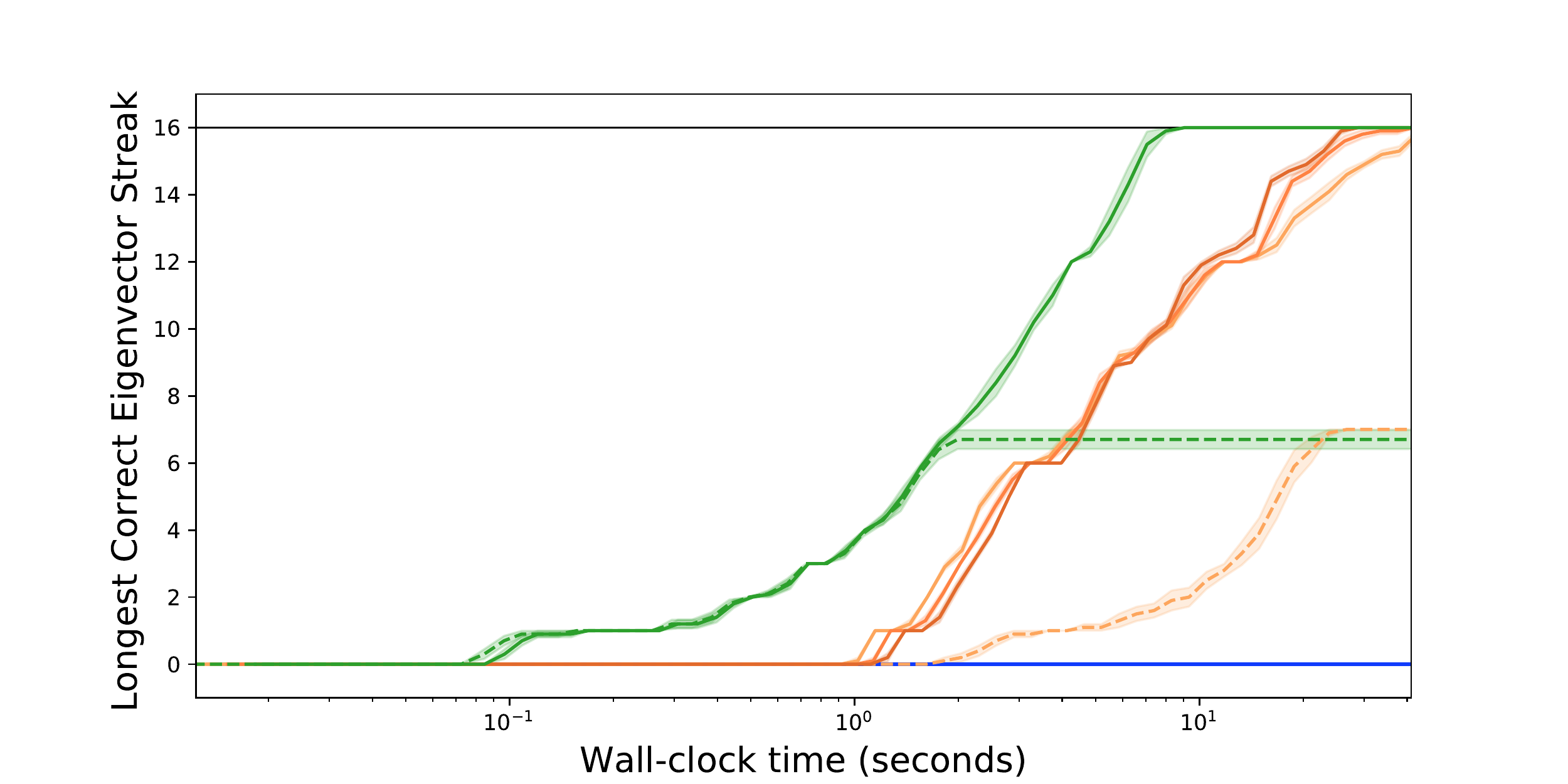}
      \caption{$V=\pi/512$}
    \end{subfigure}
    \begin{subfigure}[t]{0.5\textwidth}
      \includegraphics[width=\textwidth,trim=30 0 0 0]{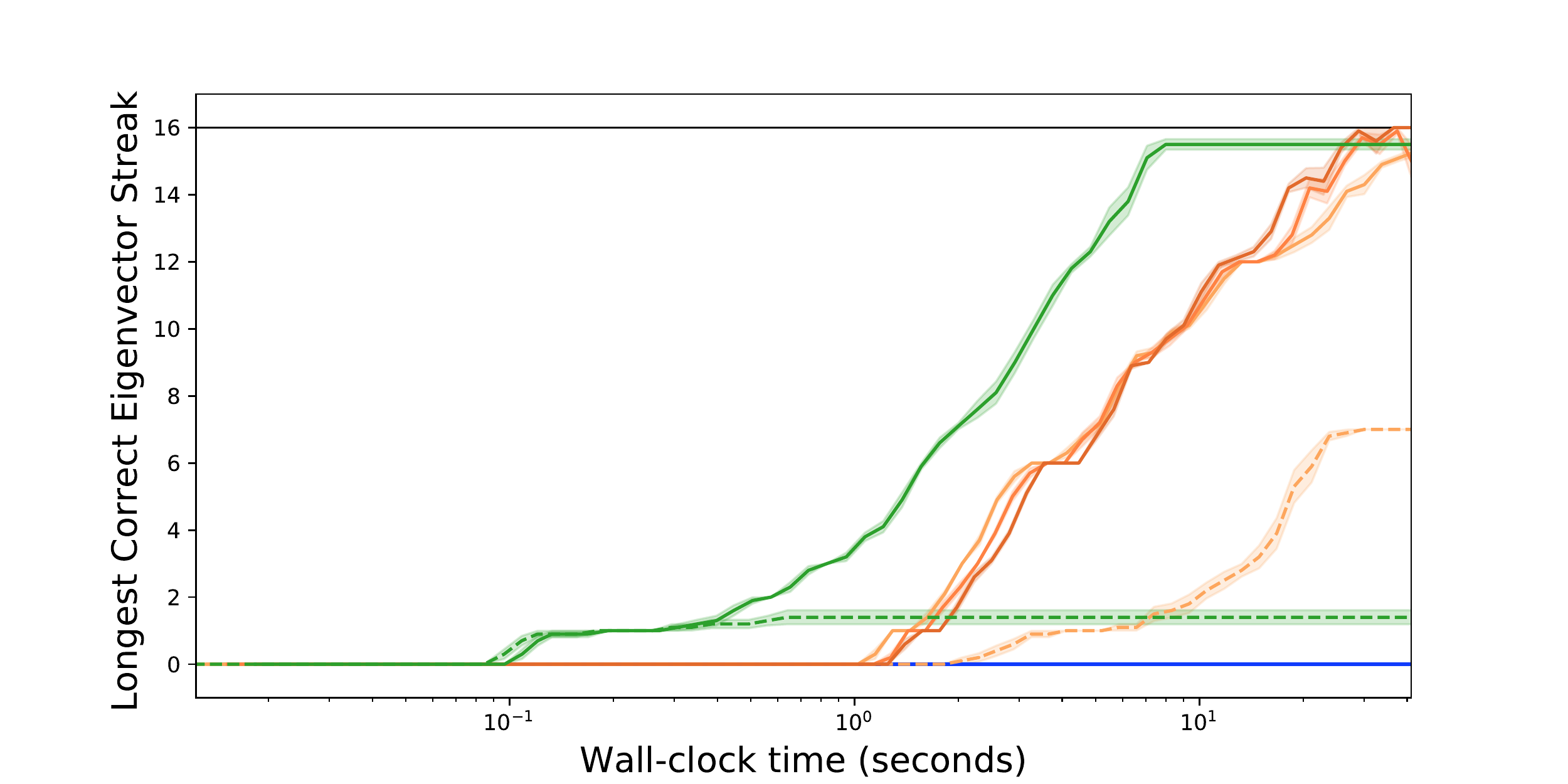}
      \caption{$V=\pi/1024$}
    \end{subfigure}
    \caption{MNIST}
\end{figure*}

\begin{figure*}
  \begin{subfigure}[t]{0.5\textwidth}
    \centering
    \includegraphics[width=4.5cm,trim=30 0 0 0]{./plots/legend.pdf}
  \end{subfigure}
    \begin{subfigure}[t]{0.5\textwidth}
      \includegraphics[width=\textwidth,trim=30 0 0 0]{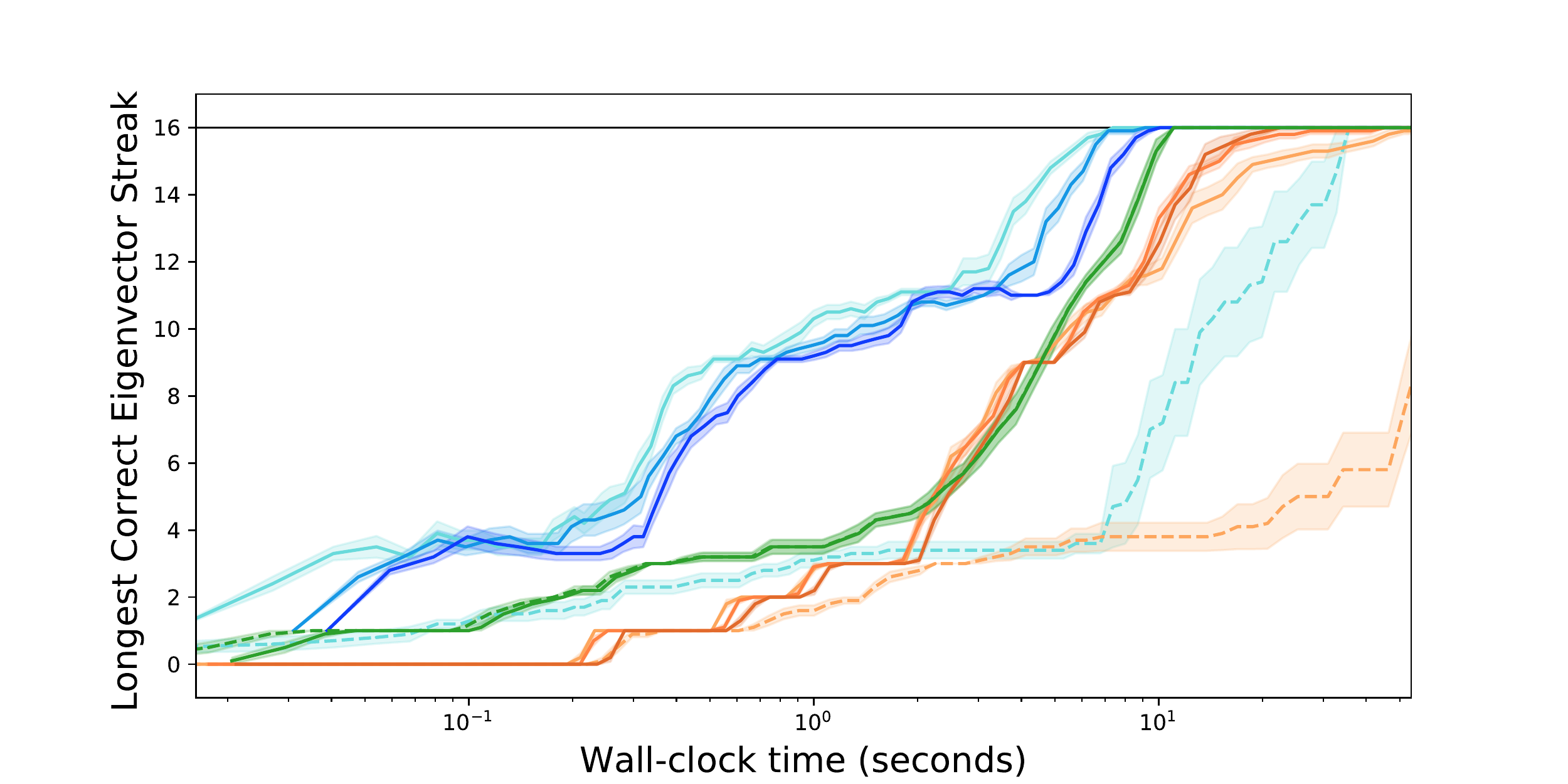}
      \caption{$V=\pi/16$}
    \end{subfigure}
    \begin{subfigure}[t]{0.5\textwidth}
      \includegraphics[width=\textwidth,trim=30 0 0 0]{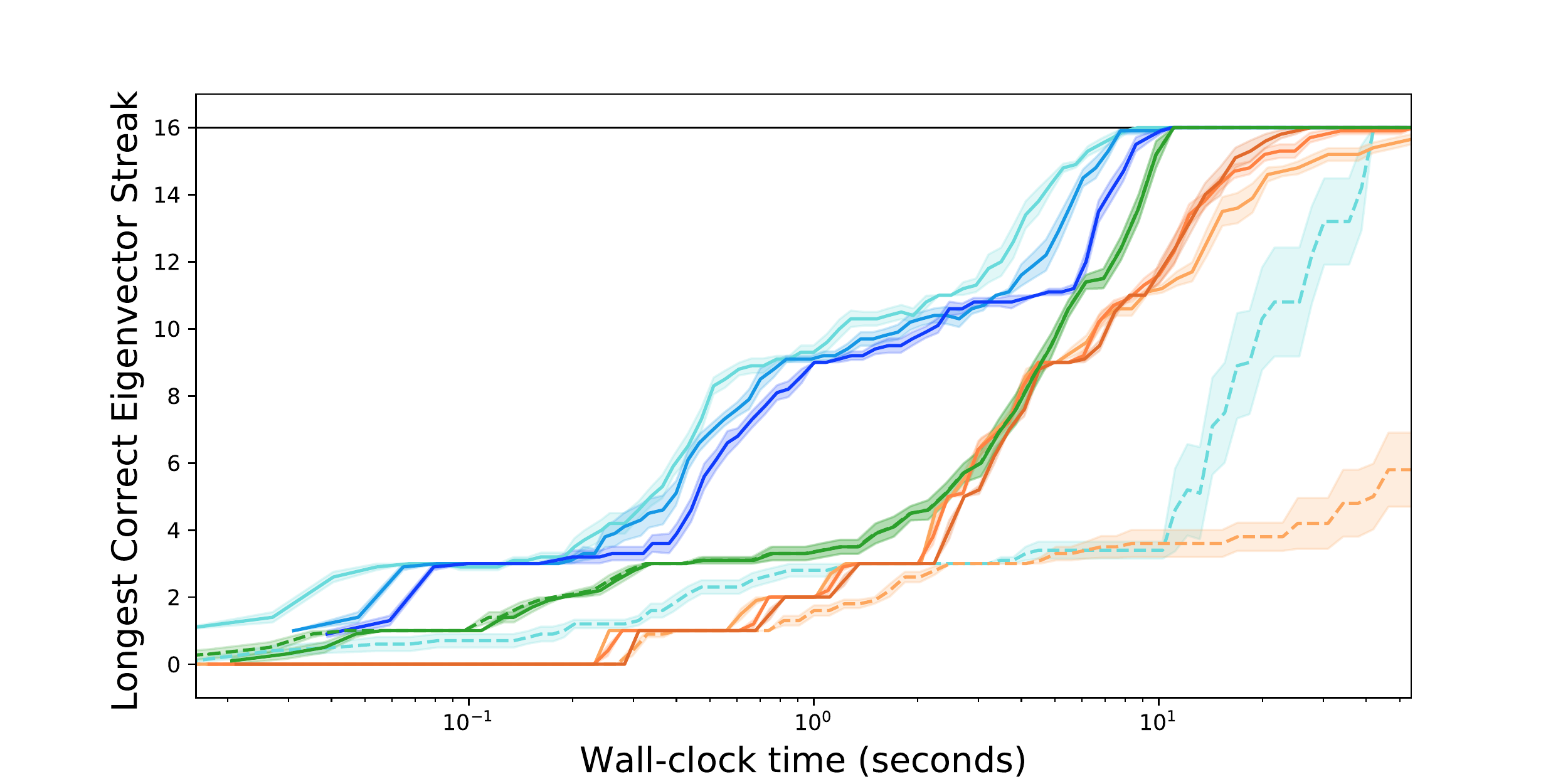}
      \caption{$V=\pi/32$}
    \end{subfigure}
    \begin{subfigure}[t]{0.5\textwidth}
      \includegraphics[width=\textwidth,trim=30 0 0 0]{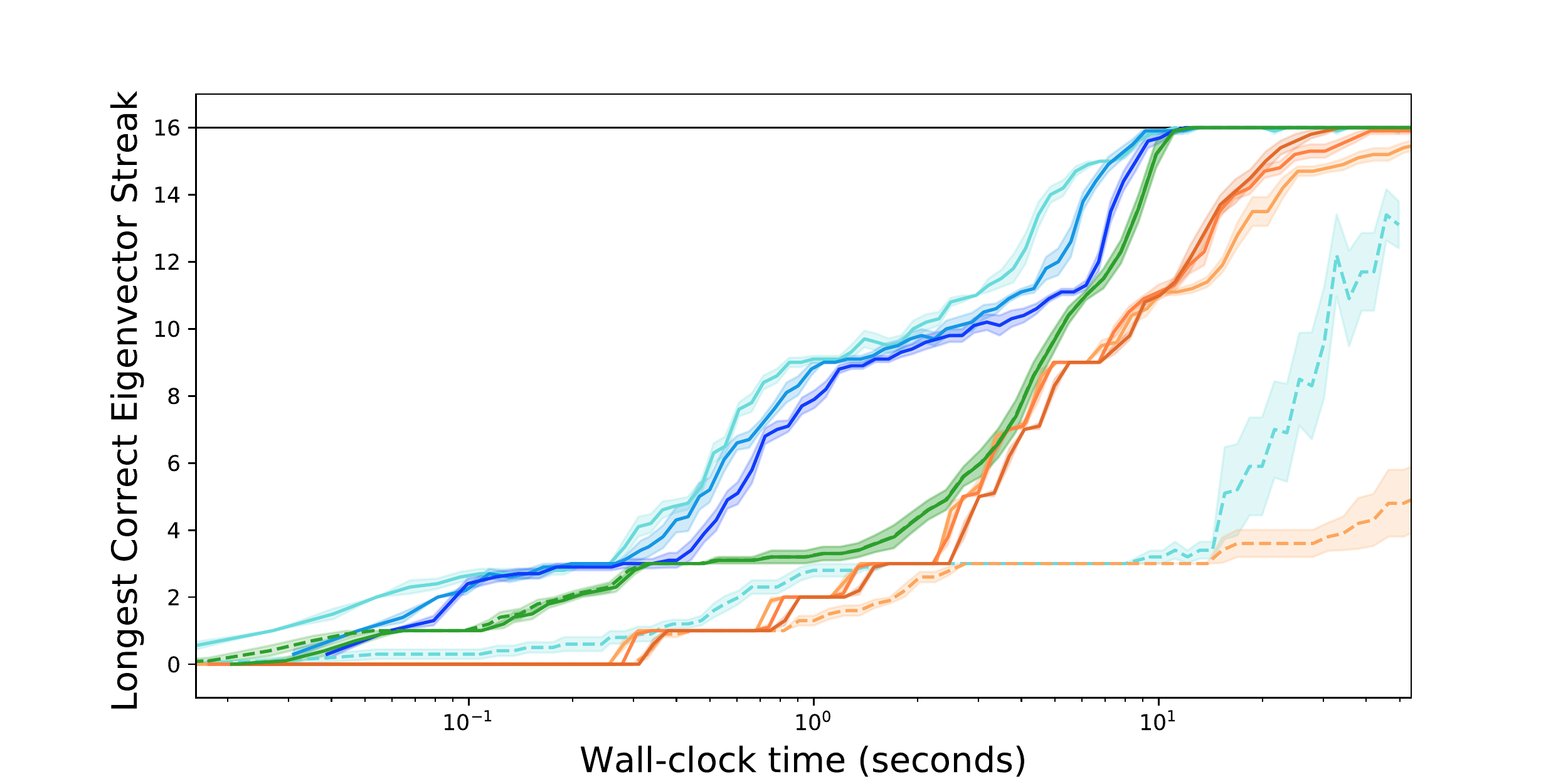}
      \caption{$V=\pi/64$}
    \end{subfigure}
    \begin{subfigure}[t]{0.5\textwidth}
      \includegraphics[width=\textwidth,trim=30 0 0 0]{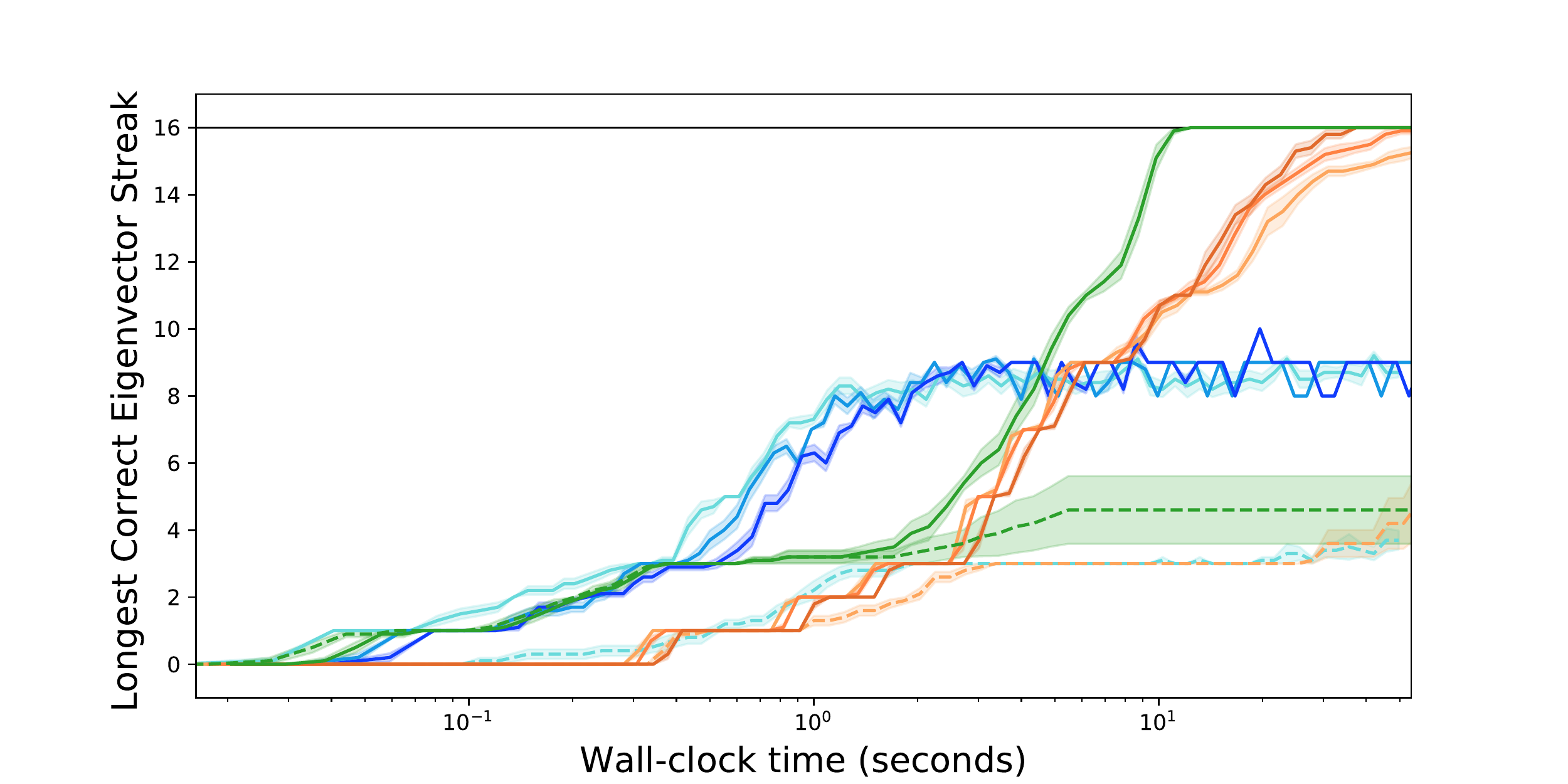}
      \caption{$V=\pi/128$}
    \end{subfigure}
    \begin{subfigure}[t]{0.5\textwidth}
      \includegraphics[width=\textwidth,trim=30 0 0 0]{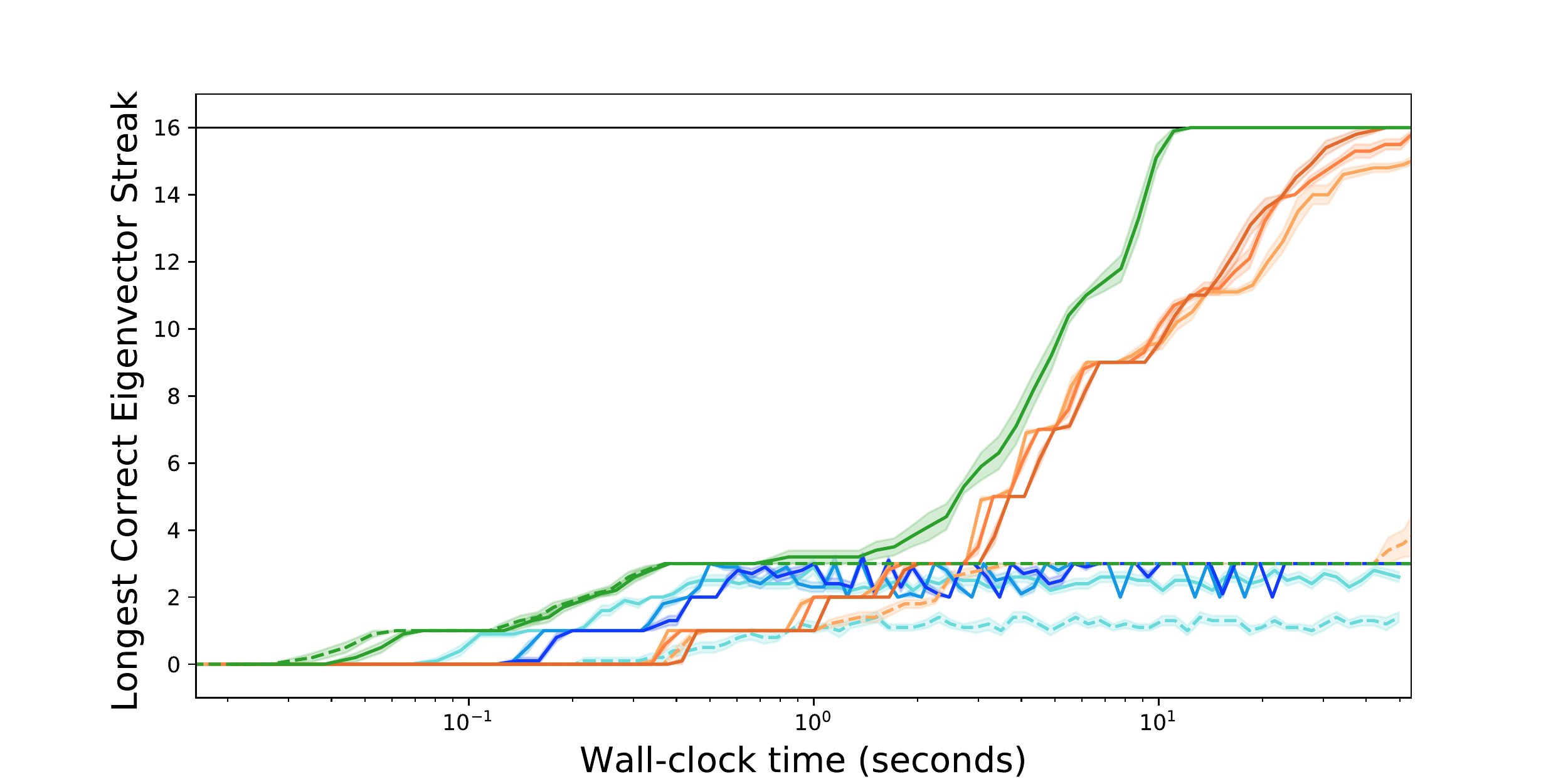}
     \caption{$V=\pi/256$}
    \end{subfigure}
    \begin{subfigure}[t]{0.5\textwidth}
      \includegraphics[width=\textwidth,trim=30 0 0 0]{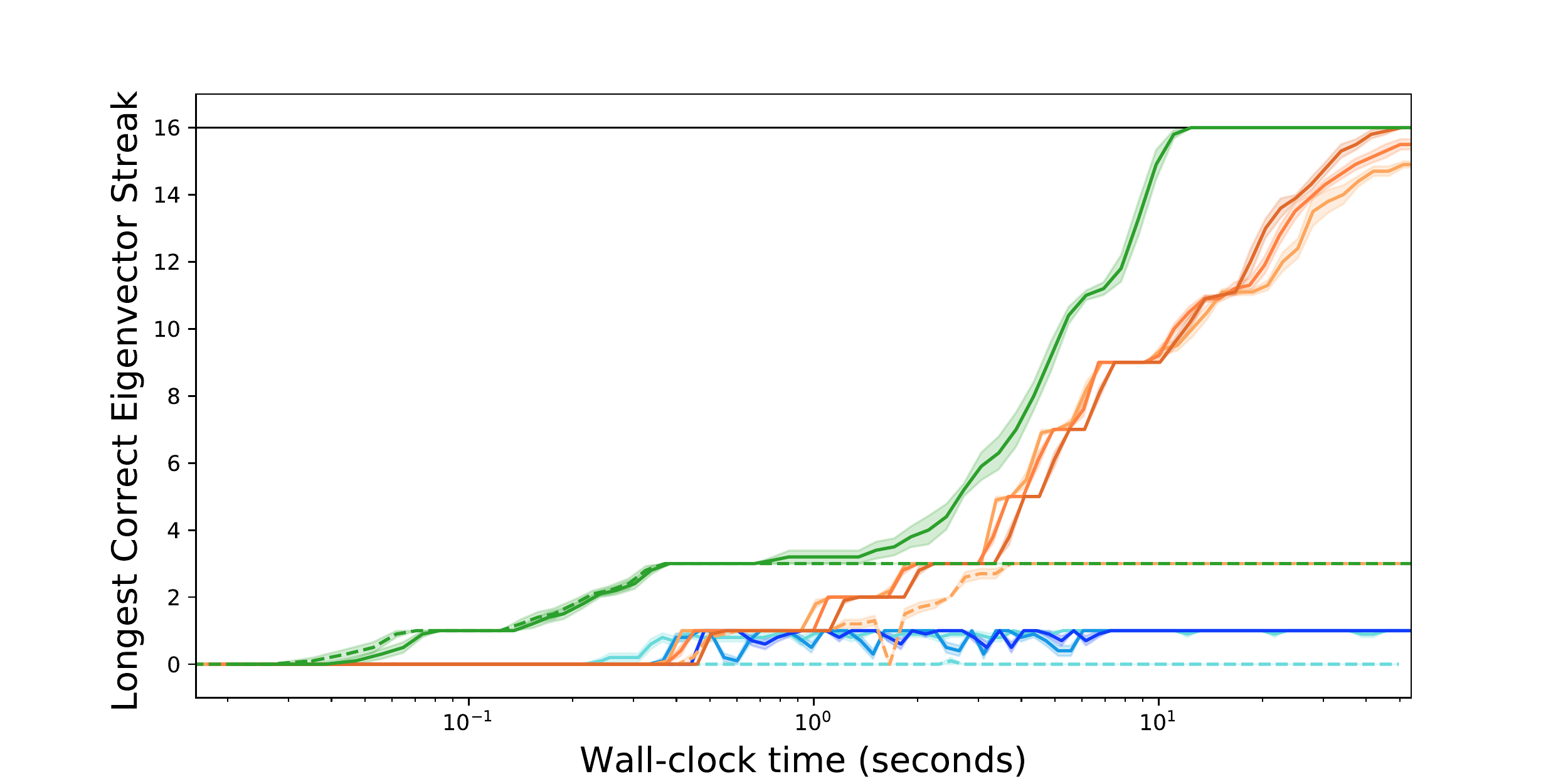}
      \caption{$V=\pi/512$}
    \end{subfigure}
    \begin{subfigure}[t]{0.5\textwidth}
      \includegraphics[width=\textwidth,trim=30 0 0 0]{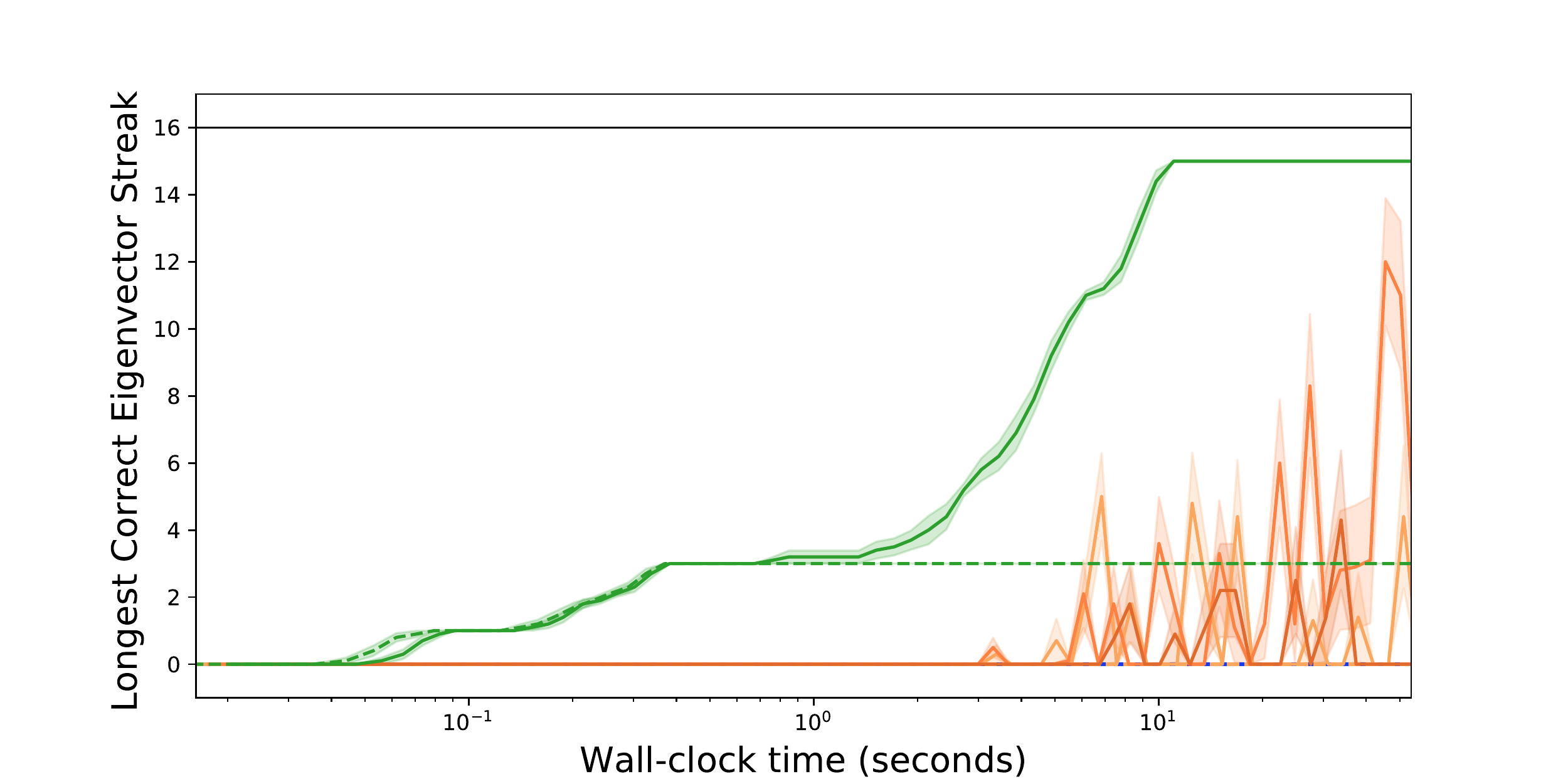}
      \caption{$V=\pi/1024$}
    \end{subfigure}
    \caption{CIFAR10}
\end{figure*}

\begin{figure*}
  \begin{subfigure}[t]{0.5\textwidth}
    \centering
    \includegraphics[width=4.5cm,trim=30 0 0 0]{./plots/legend.pdf}
  \end{subfigure}
    \begin{subfigure}[t]{0.5\textwidth}
      \includegraphics[width=\textwidth,trim=30 0 0 0]{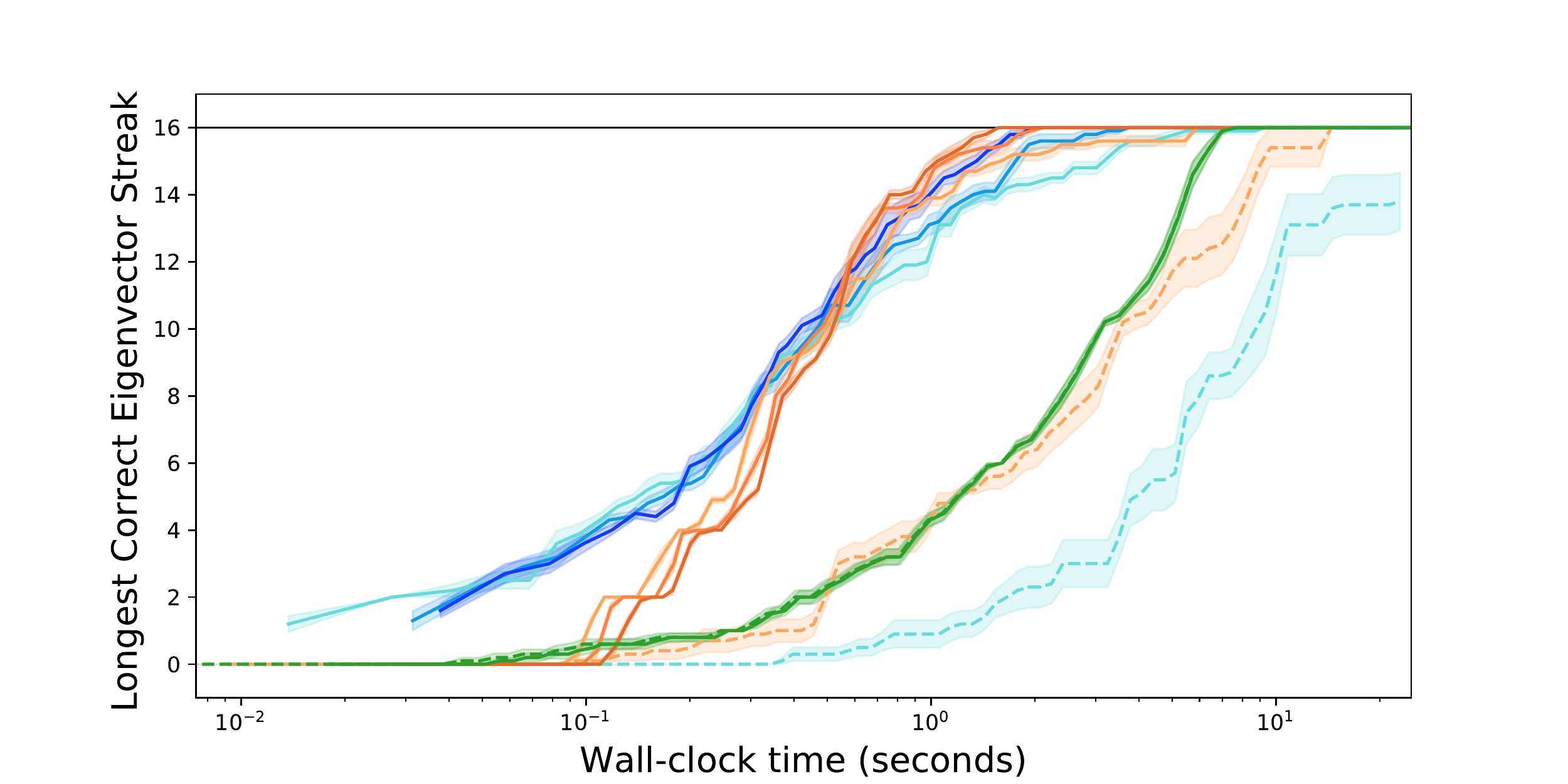}
      \caption{$V=\pi/16$}
    \end{subfigure}
    \begin{subfigure}[t]{0.5\textwidth}
      \includegraphics[width=\textwidth,trim=30 0 0 0]{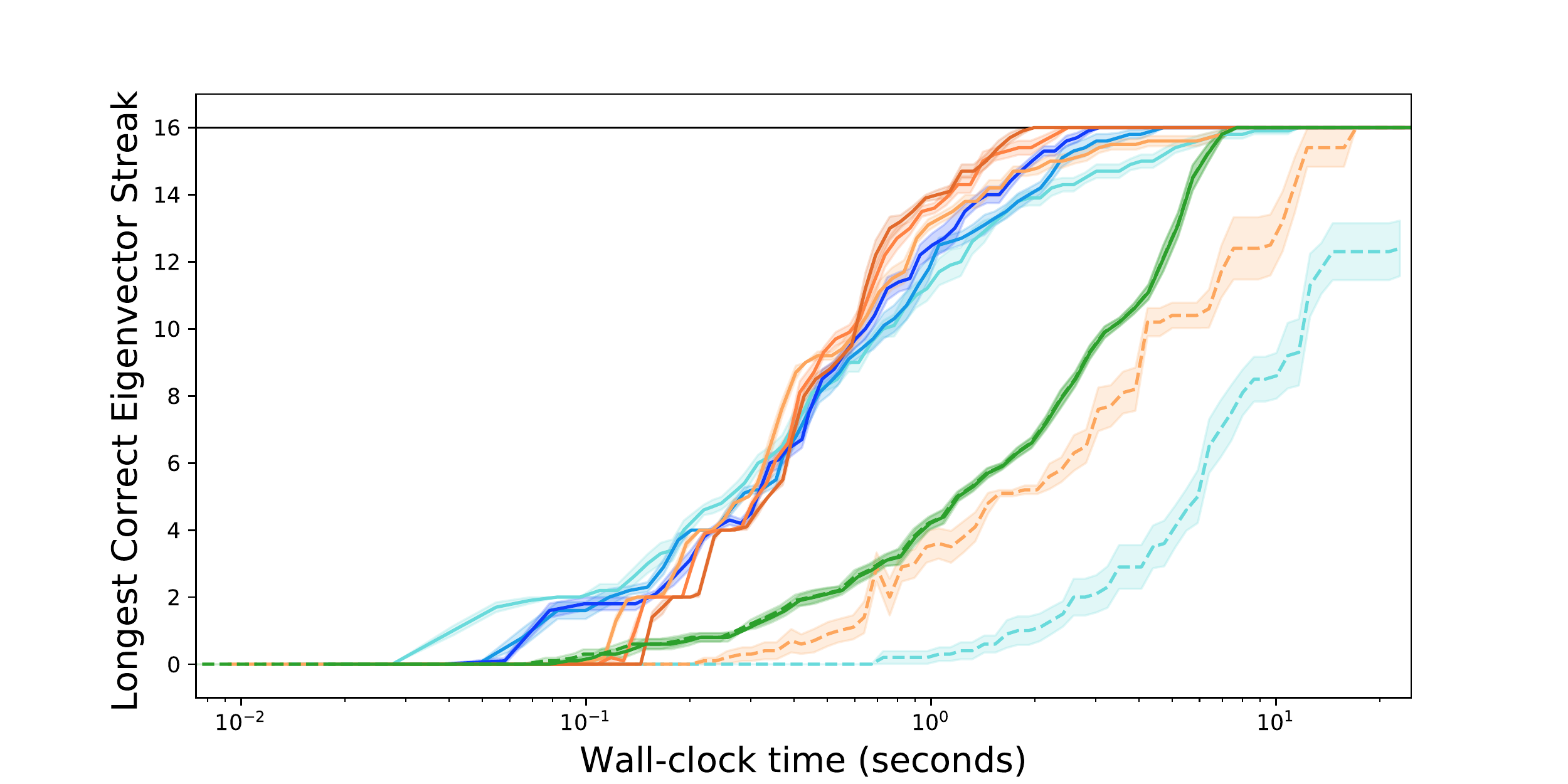}
      \caption{$V=\pi/32$}
    \end{subfigure}
    \begin{subfigure}[t]{0.5\textwidth}
      \includegraphics[width=\textwidth,trim=30 0 0 0]{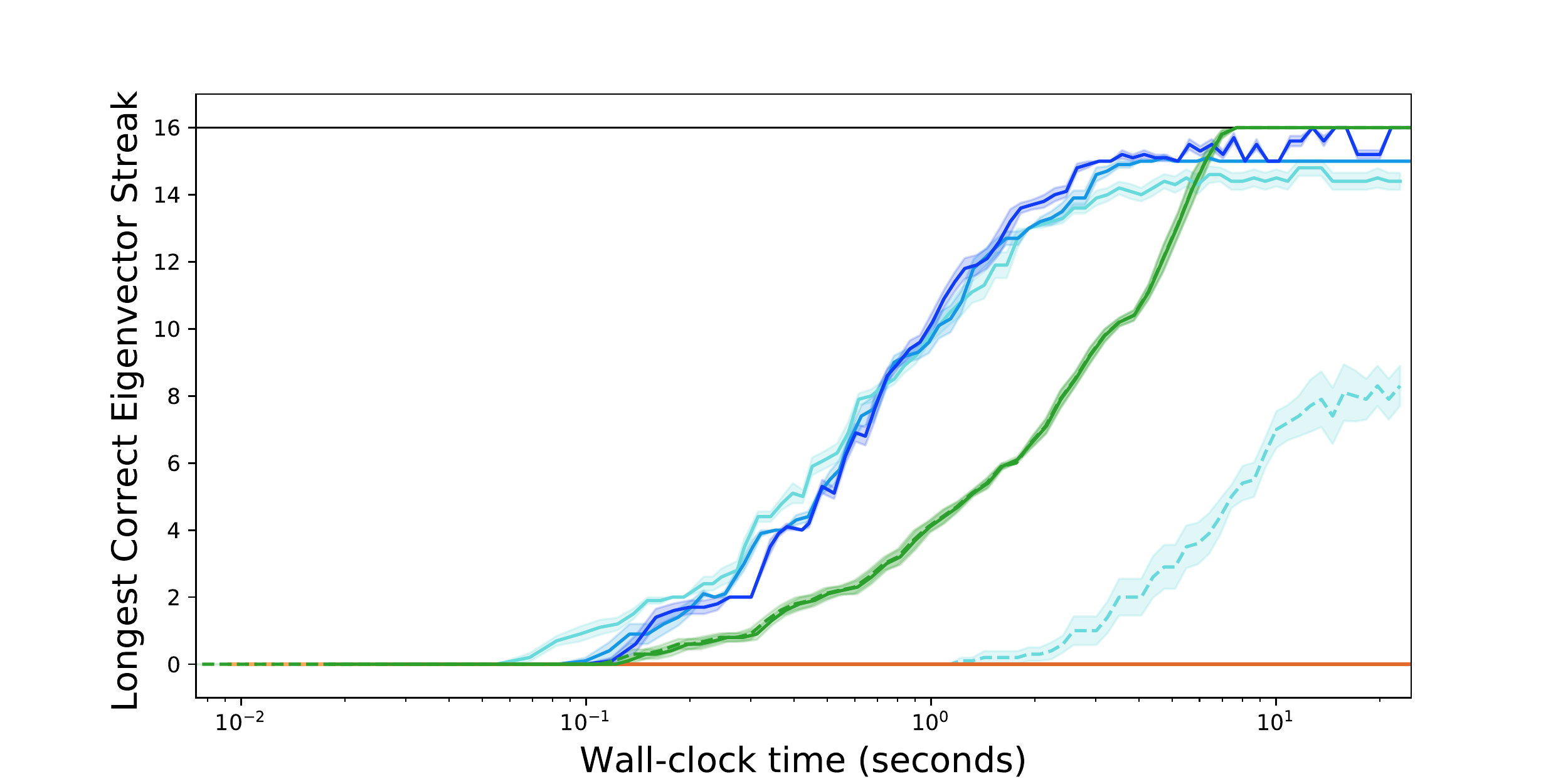}
      \caption{$V=\pi/64$}
    \end{subfigure}
    \begin{subfigure}[t]{0.5\textwidth}
      \includegraphics[width=\textwidth,trim=30 0 0 0]{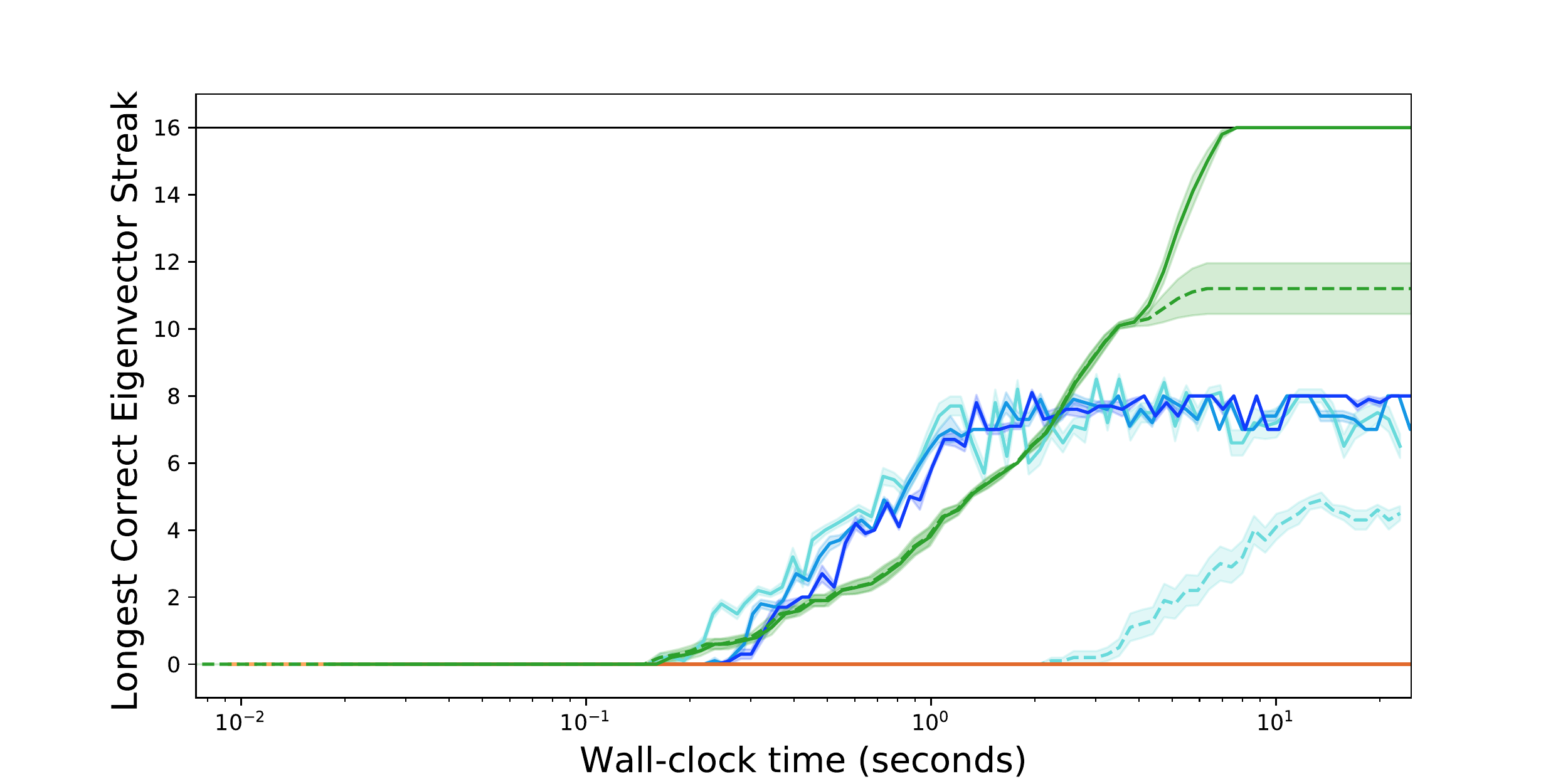}
      \caption{$V=\pi/128$}
    \end{subfigure}
    \begin{subfigure}[t]{0.5\textwidth}
      \includegraphics[width=\textwidth,trim=30 0 0 0]{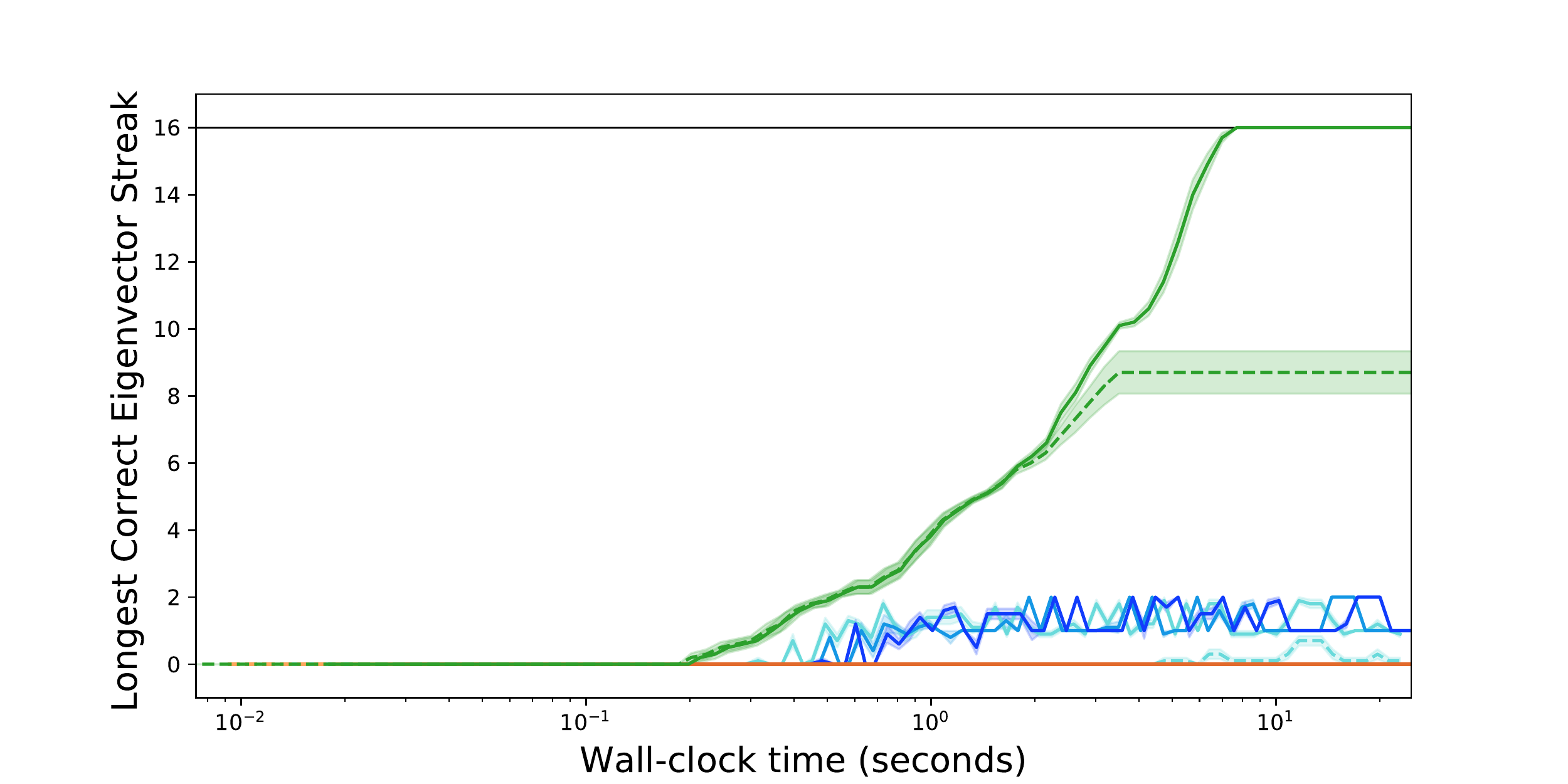}
     \caption{$V=\pi/256$}
    \end{subfigure}
    \begin{subfigure}[t]{0.5\textwidth}
      \includegraphics[width=\textwidth,trim=30 0 0 0]{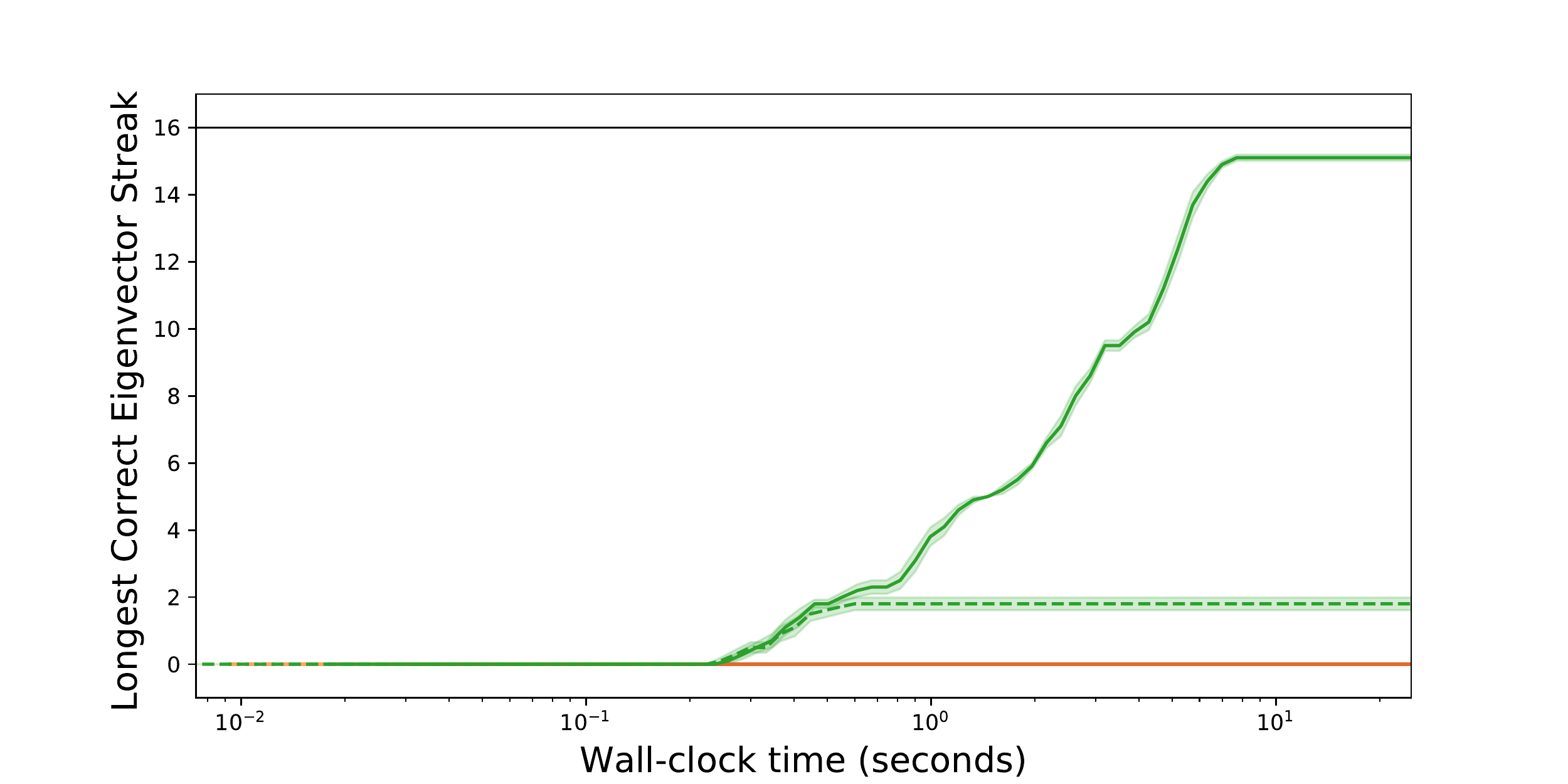}
      \caption{$V=\pi/512$}
    \end{subfigure}
    \begin{subfigure}[t]{0.5\textwidth}
      \includegraphics[width=\textwidth,trim=30 0 0 0]{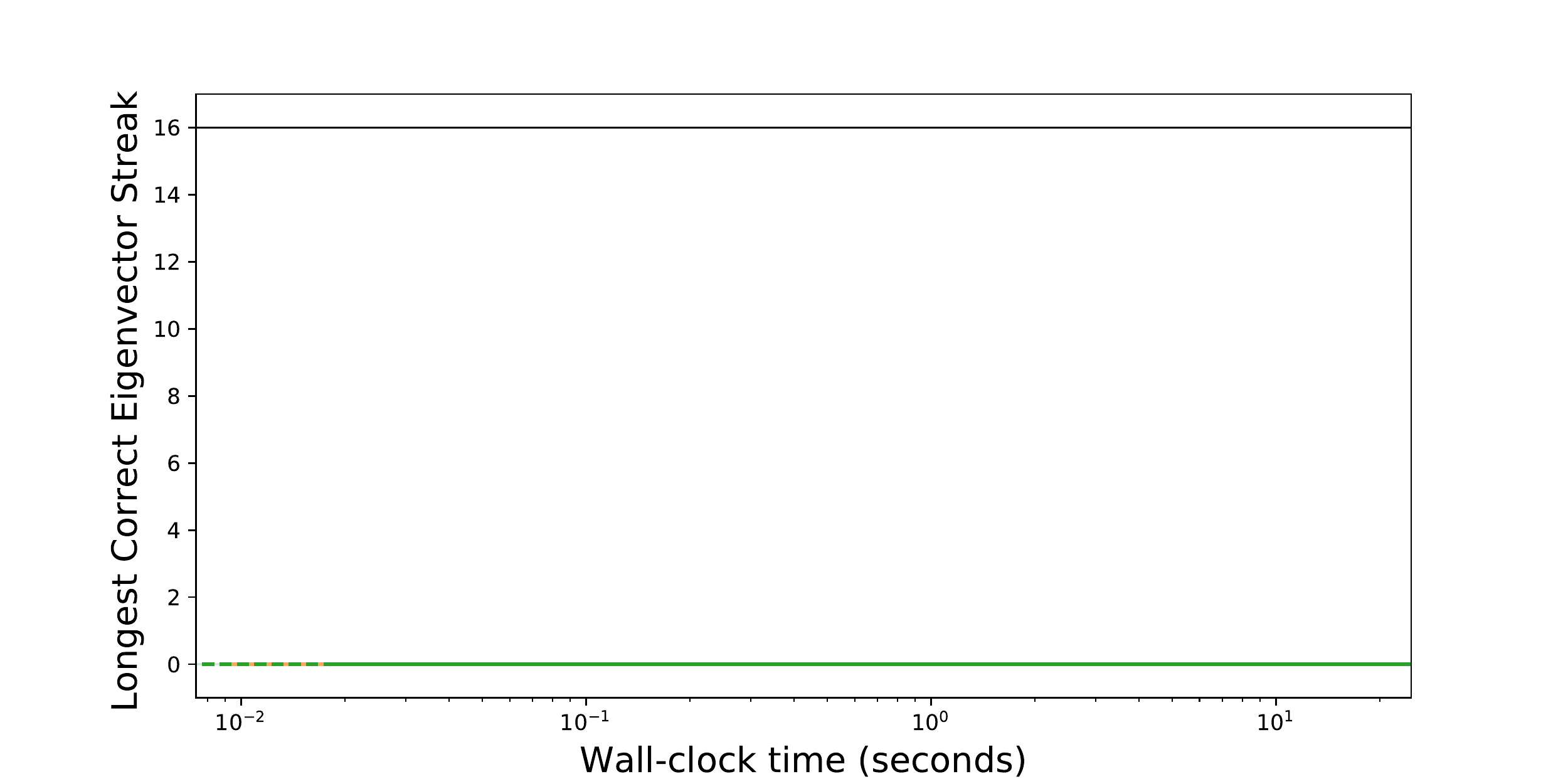}
      \caption{$V=\pi/1024$}
    \end{subfigure}
    \caption{NIPS bag of words}
\end{figure*}

\end{document}